\documentclass[runningheads]{llncs}
\usepackage{graphicx}

\usepackage{tikz}
\usepackage{comment}
\usepackage{amsmath,amssymb} 
\usepackage{color}
\usepackage{booktabs}
\usepackage{soul}
\usepackage{algorithm}
\usepackage{algorithmic}
\usepackage{listings}
\usepackage{subcaption}
\usepackage{setspace}
\usepackage{hyperref}

\def\*#1{\mathbf{#1}}

\usepackage[capitalize]{cleveref}
\crefname{section}{Sec.}{Secs.}
\Crefname{section}{Section}{Sections}
\Crefname{table}{Table}{Tables}
\crefname{table}{Tab.}{Tabs.}
\Crefname{equation}{Equation}{Equations}
\crefname{equation}{}{}

\usepackage[accsupp]{axessibility}  

\begin{document}
\pagestyle{headings}
\mainmatter
\def\ECCVSubNumber{4857}  

\title{Self-Supervised Classification Network} 

\titlerunning{Self-Supervised Classification Network}
%
\author{Elad Amrani\inst{1,2} \and
Leonid Karlinsky\inst{1} \and
Alex Bronstein\inst{2}}
\authorrunning{E. Amrani et al.}
%
\institute{IBM Research-AI \and Technion}

\maketitle

\begin{abstract}
We present Self-Classifier -- a novel self-supervised end-to-end classification learning approach. Self-Classifier learns labels and representations simultaneously in a single-stage end-to-end manner by optimizing for same-class prediction of two augmented views of the same sample. To guarantee non-degenerate solutions (i.e., solutions where all labels are assigned to the same class) we propose a mathematically motivated variant of the cross-entropy loss that has a uniform prior asserted on the predicted labels. In our theoretical analysis, we prove that degenerate solutions are not in the set of optimal solutions of our approach. Self-Classifier is simple to implement and scalable. Unlike other popular unsupervised classification and contrastive representation learning approaches, it does not require any form of pre-training, expectation-maximization, pseudo-labeling, external clustering, a second network, stop-gradient operation, or negative pairs. Despite its simplicity, our approach sets a new state of the art for unsupervised classification of ImageNet; and even achieves comparable to state-of-the-art results for unsupervised representation learning. Code is available at \url{https://github.com/elad-amrani/self-classifier}. 
\end{abstract}

\begin{figure}[t]
\begin{center}
  \includegraphics[width=0.48\linewidth]{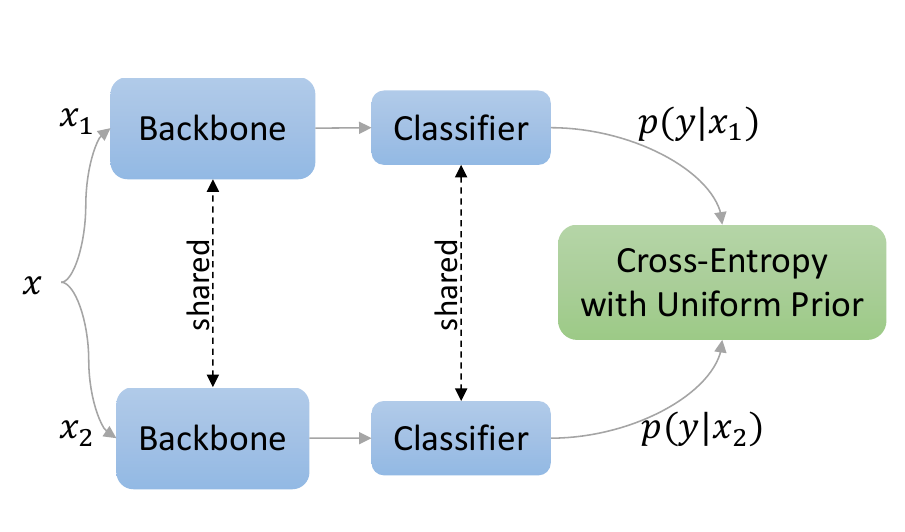}
\end{center}
  \caption{\textbf{\textit{Self-Classifier} architecture}. Two augmented views of the same image are processed by a shared network comprised of a backbone (e.g. CNN) and a classifier (e.g. projection MLP + linear classification head). The cross-entropy of the two views is minimized to promote same class prediction while avoiding degenerate solutions by asserting a uniform prior on class predictions. The resulting model learns representations and discovers the underlying classes in a single-stage end-to-end unsupervised manner. }
\label{fig:architecture}
\end{figure}

\section{Introduction}
Self-supervised visual representation learning has gained increasing interest over the past few years \cite{wu2018unsupervised,dosovitskiy2014discriminative,DBLP:journals/corr/abs-2002-05709,chen2020big,he2020momentum,chen2020improved,caron2020unsupervised,misra2020self}. The main idea is to define and solve a pretext task such that semantically meaningful representations can be learned without any human-annotated labels. The learned representations are later transferred to downstream tasks, e.g., by fine-tuning on a smaller dataset. Current state-of-the-art self-supervised models are based on contrastive learning (\cref{section:self_supervised_learning}). These models maximize the similarity between two different augmentations of the same image while simultaneously minimizing the similarity between different images, subject to different conditions.
Although they attain impressive overall performance, for some downstream tasks, such as unsupervised classification (\cref{section:unsupervised_image_classification}), the objective of the various proposed pretext tasks might not be sufficiently well aligned. For example, instance discrimination methods, such as \cite{he2020momentum,chen2020improved} used for pre-training in the current state-of-the-art unsupervised classification method \cite{van2020scan}, decrease similarity between all instances, even between those that belong to the same (unknown during training) class, thus potentially working against the set task. In contrast, in this paper we propose a classification-based pretext task whose objective is directly aligned with the end goal in this case.
Knowing only the number of classes $C$ we learn an unsupervised classifier (\textit{Self-Classifier}) such that two different augmentations of the same image are classified similarly. In practice, such a task is prone to degenerate solutions, where all samples are assigned to the same class. To avoid them, we assert a uniform prior on the standard cross-entropy loss function, such that a solution with an equipartition of the data is an optimal solution. In fact, we show that the set of optimal solutions no longer includes degenerate ones. 

Our approach can also be viewed as a form of deep unsupervised clustering (Section \ref{section:deep_unsupervised_clustering}) \cite{xie2016unsupervised,yang2016joint,chang2017deep,caron2018deep,haeusser2018associative,ji2019invariant,van2020scan,YM.2020Self-labelling} combined with contrastive learning. Similarly to deep clustering methods, we learn the parameters of a neural network and cluster (class) assignments simultaneously. Recently, clustering has been combined with contrastive learning in \cite{YM.2020Self-labelling,caron2020unsupervised} with great success, yet in both studies clustering was employed as a separate step used for pseudo-labeling. In contrast, in this work we learn representations and cluster labels in a single-stage end-to-end manner, using only minibatch SGD.

The key contributions of this paper are:
\begin{enumerate}
    \item A simple yet effective self-supervised single-stage end-to-end classification and representation learning approach. Unlike previous unsupervised classification works, our approach does not require any form of pre-training, expectation-maximization algorithm, pseudo-labeling, or external clustering. Unlike previous unsupervised representation learning works, our approach does not require a memory bank, a second network (momentum), external clustering, stop-gradient operation, or negative pairs.
    \item Although simple, our approach sets a new state of the art for unsupervised classification on ImageNet with 41.1\% top-1 accuracy, achieves results comparable to state of the art for unsupervised representation learning, and attains a significant ($\sim 2\%$ AP) improvement in transfer to COCO det/seg compared to other self. sup. methods.
    \item We are the first to provide quantitative analysis of self-supervised classification predictions alignment to a set of different class hierarchies (defined on ImageNet and its subpopulations), and show significant (up to $3.4\%$ AMI) improvement over previous state of the art in this new metric. 
\end{enumerate}


\section{Related Work}
\subsection{Self-Supervised Learning}
\label{section:self_supervised_learning}
Self-supervised learning methods learn compact semantic data representations by defining and solving a pretext task. In such tasks, naturally existing supervision signals are utilized for training. Many pretext tasks were proposed in recent years in the domain of computer vision, including colorization \cite{zhang2016colorful}, jigsaw puzzle \cite{noroozi2016unsupervised}, image inpainting \cite{pathak2016context}, context prediction \cite{doersch2015unsupervised}, rotation prediction \cite{DBLP:conf/iclr/GidarisSK18}, and contrastive learning \cite{wu2018unsupervised,dosovitskiy2014discriminative,DBLP:journals/corr/abs-2002-05709,chen2020big,he2020momentum,chen2020improved,caron2020unsupervised,misra2020self} just to mention a few.

Contrastive learning has shown great promise and has become a \emph{de facto} standard for self-supervised learning. Two of the earliest studies of contrastive learning are Exemplar CNN \cite{dosovitskiy2014discriminative}, and Non-Parametric Instance Discrimination (NPID) \cite{wu2018unsupervised}. Exemplar CNN \cite{dosovitskiy2014discriminative}, learns to discriminate between instances using a convolutional neural network classifier, where each class represents a single instance and its augmentations. While highly simple and effective, it does not scale to arbitrarily large amounts of unlabeled data since it requires a classification layer (softmax) the size of the dataset. NPID  \cite{wu2018unsupervised} tackles this problem by approximating the full softmax distribution with noise-contrastive estimation (NCE) \cite{gutmann2010noise} and utilizing a memory bank to store the recent representation of each instance to avoid computing the representations of the entire dataset at each time step of the learning process. Such approximation is effective since, unlike Exemplar CNN, it allows training with large amounts of unlabeled data. However, the proposed memory bank by NPID introduces a new problem - lack of consistency across representations stored in the memory bank, i.e., the representations of different samples in the memory bank are computed at multiple different time steps. Nonetheless, Exemplar CNN and NPID have inspired a line of studies of contrastive learning \cite{he2020momentum,chen2020improved,DBLP:journals/corr/abs-2002-05709,chen2020big,li2021prototypical,caron2020unsupervised}. 

One such recent study is SwAV \cite{caron2020unsupervised} which resembles the present work the most. SwAV takes advantage of contrastive methods without requiring to compute pairwise comparisons. More specifically, it simultaneously clusters the data while enforcing consistency between cluster assignments produced for different augmentations (or “views”) of the same image, instead of comparing features directly. To avoid a trivial solution where all samples collapse into a single cluster, SwAV alternates between representation learning using back propagation, and a separate clustering step using the Sinkhorn-Knopp algorithm. In contrast to SwAV, in this work we propose a model that allows learning both representations and cluster assignments in a single-stage end-to-end manner.  

\subsection{Deep Unsupervised Clustering}
\label{section:deep_unsupervised_clustering}
Deep unsupervised clustering methods simultaneously learn the parameters of a neural network and the cluster assignments of the resulting features using unlabeled data \cite{xie2016unsupervised,yang2016joint,chang2017deep,caron2018deep,haeusser2018associative,ji2019invariant,van2020scan,YM.2020Self-labelling}. Such a task is understandably vulnerable to degenerate solutions, where all samples are assigned to a single cluster. Many different solutions that were proposed to avoid the trivial outcome are based on one or few of the following: a) pre-training mechanism; b) Expectation-Maximization (EM) algorithm (i.e., alternating between representation learning and cluster assignment); c) pseudo-labeling; and d) external clustering algorithm such as $k$-means. 

Two of the earliest studies of deep clustering are DEC \cite{xie2016unsupervised} and JULE \cite{yang2016joint}. DEC \cite{xie2016unsupervised} initializes the parameters of its network using a deep autoencoder, and its cluster centroids using standard $k$-means clustering in the feature space. It then uses a form of EM algorithm, where it iterates between pseudo-labeling and learning from its own high confidence predictions. JULE \cite{xie2016unsupervised}, similarly to DEC, alternates between pseudo-labeling and learning from its own predictions. However, unlike DEC, JULE avoids a pre-training step and instead utilizes the prior on the input signal given by a randomly initialized ConvNet together with agglomerative clustering.

More recent approaches are SeLa \cite{YM.2020Self-labelling} and IIC \cite{ji2019invariant}. SeLa \cite{YM.2020Self-labelling} uses a form of EM algorithm, where it iterates between minimization of the cross entropy loss and pseudo-labeling by solving efficiently an instance of the {\it optimal transport problem} using the Sinkhorn-Knopp algorithm. IIC \cite{ji2019invariant} is a single-stage end-to-end deep clustering model conceptually similar to the approach presented in this paper. IIC maximizes the mutual information between predictions of two augmented views of the same sample. The two entropy terms constituting mutual information -- the entropy of a sample and its negative conditional entropy given the other sample compete with each other, with the entropy being maximal when the labels are uniformly distributed over the clusters, and the negative conditional entropy being maximal for sharp one-hot instance assignments. 

In this work, we follow a similar rationale for single-stage end-to-end classification without the use of any pseudo-labeling. Unlike IIC, our proposed loss is equivalent to the cross-entropy classification loss under a uniform label prior that guarantees non-degenerate, uniformly distributed optimal solution as explained in \cref{section:self-classifier}. Although many deep clustering approaches were proposed over the years, only two of them (SCAN \cite{van2020scan} and SeLa \cite{YM.2020Self-labelling}) have demonstrated scalability to large-scale datasets such as ImageNet. In fact, the task of unsupervised classification of large-scale datasets remains an open challenge.

\section{Self-Classifier}
\label{section:self-classifier}
Let $x_1, x_2$ denote two different augmented views of the same image sample $x$. Our goal is to learn a classifier $y \triangleq f(x_i) \in [C]$, where $C$ is the given number of classes,
such that two augmented views of the same sample are classified similarly, while avoiding degenerate solutions. 
A naive approach to this would be minimizing the following cross-entropy loss:
\begin{equation}
    \tilde{\ell}(x_1, x_2) = -\sum_{y \in [C]}{p(y|x_2)\log p(y|x_1)},
    \label{eq:ell_naive}
\end{equation}
where $p(y|x)$ is a row softmax with temperature $\tau_{row}$ \cite{wu2018unsupervised} of the matrix of logits $\mathcal{S}$ produced by our model (backbone $+$ classifier) for all classes (columns) and batch samples (rows). However, without additional regularization, an attempt to minimize \cref{eq:ell_naive} will quickly converge to a degenerate solution in which the network predicts a constant $y$ regardless of the $x$. In order to remedy this, we propose to invoke Bayes and total probability laws, obtaining:
\begin{equation}
    p(y|x_2) = \frac{p(y)p(x_2|y)}{p(x_2)} = \frac{p(y)p(x_2|y)}{\sum_{\tilde{y} \in [C]}{p(x_2|\tilde{y}) p(\tilde{y})}},
    \label{eq:full_bayes}
\end{equation}
\begin{equation}
    p(y|x_1) = \frac{p(y)p(y|x_1)}{p(y)} = \frac{p(y)p(y|x_1)}{\sum_{\tilde{x_1} \in B_1}{p(y|\tilde{x}_1) p(\tilde{x}_1)}},
    \label{eq:total_probability}
\end{equation}
where $B$ is a batch of $N$ samples ($B_1$ are the first augmentations of samples of $B$), and $p(x|y)$ is a column softmax of the aforementioned matrix of logits $\mathcal{S}$ with the temperature $\tau_{col}$. Now, assuming that $p(x_1)$ is uniform (under the reasonable assumption that the training samples are equi-probable), and, since we would like all classes to be used, assuming (an intuitive) uniform prior for $p(y)$, we obtain:
\begin{equation}
    \small
    \ell(x_1, x_2) = -\sum_{y \in [C]}{\frac{p(x_2|y)}{\sum_{\tilde{y}}{p(x_2|\tilde{y})}}\log \bigg(\frac{N}{C}\frac{p(y|x_1)}{\sum_{\tilde{x}_1}{p(y|\tilde{x}_1)}}\bigg)},
    \label{eq:ce_ours}
\end{equation}
where $p(y)$ and $p(\tilde{y})$ cancel out in \cref{eq:full_bayes}, and $p(y)/p(\tilde{x}_1)$ becomes $N/C$ in \cref{eq:total_probability}. In practice, we use a symmetric variant of this loss (that we empirically noticed to be better):
\begin{equation}
    \mathcal{L} = \frac{1}{2}\bigg(\ell(x_1, x_2) + \ell(x_2, x_1)\bigg).
    \label{eq:L_sym}
\end{equation}
Note that the naive cross entropy in \cref{eq:ell_naive} is in fact mathematically equivalent to our proposed loss function in \cref{eq:ce_ours}, under the assumption that $p(y)$ and $p(x)$ are uniform. Finally, despite being very simple (only few lines of PyTorch-like pseudocode in \cref{algo:pseudo_code}) our method sets a new state of the art in self-supervised classification (\cref{section:unsupervised_image_classification}).

\begin{figure}[ht]
\centering
\begin{minipage}{.87\linewidth}
    \lstset{language=Python, numbers=left, numberstyle=\tiny, stepnumber=0, numbersep=1pt, tabsize=1, basicstyle=\fontsize{9}{11}\ttfamily}
    \begin{algorithm}[H]
        \setstretch{0.1}
        \caption{{\it Self-Classifier} PyTorch-like Pseudocode}
        \begin{lstlisting}
# N: number of samples in batch
# C: number of classes
# t_r / t_c: row / column softmax temperatures
# aug(): random augmentations
# softmaxX(): softmax over dimension X
# normX(): L1 normalization over dimension X
for x in loader:
  s1, s2 = model(aug(x)), model(aug(x))

  log_y_x1 = log(N/C * norm0(softmax1(s1/t_r)))
  log_y_x2 = log(N/C * norm0(softmax1(s2/t_r)))

  y_x1 = norm1(softmax0(s1/t_c))
  y_x2 = norm1(softmax0(s2/t_c))

  l1 = - sum(y_x2 * log_y_x1) / N
  l2 = - sum(y_x1 * log_y_x2) / N
  L = (l1 + l2) / 2

  L.backward()
  optimizer.step()    
        \end{lstlisting}
        \label{algo:pseudo_code}
    \end{algorithm}
\end{minipage}
\end{figure}

\section{Theoretical Analysis}
\label{section:theoretical}
In this section, we show mathematically how \textit{Self-Classifier} avoids trivial solutions by design, i.e., a collapsing solution is not in the set of optimal solutions of our proposed loss function \cref{eq:ce_ours}. Proofs are provided in Supplementary.

\begin{theorem}[Non-Zero Posterior Probability]
\label{theorem:non_zero_prob}
Let $B$ be a batch of $N$ samples with two views per sample, $(x_1, x_2) \in B$. Let $p(y)$ and $p(x)$ be the class and sample distributions, respectively, where $y \in [C]$. Let \cref{eq:L_sym} be the loss function. Then, each class $y \in [C]$ will have at least one sample $y \in [C]$ with non-zero posterior probability $p(x|y) > 0$ assigned into it, and each sample $x \in [N]$ will have at least one class  $y \in [C]$ with $p(x|y) > 0$.
\end{theorem}

\begin{theorem}[Optimal Solution With Uniform Prior]
\label{theorem:custom_optimal_sol}
Let $B$ be a batch of $N$ samples with two views per sample, $(x_1, x_2) \in B$. Let $p(y)$ and $p(x)$ be the class and sample distribution, respectively, where $y \in [C]$. Then, the uniform probabilities $p(y) = \frac{1}{C}$, $p(x) = \frac{1}{N}$ constitute a global minimizer of the loss \cref{eq:ce_ours}.
\end{theorem}

\section{Implementation Details}
\subsection{Architecture}
\label{section:implementation_details_arch}
In all our experiments, we used ResNet-50 \cite{he2016deep} backbone (as customary for all compared SSL works) initialized randomly. Following previous work, for our projection heads we used an MLP with $2$ layers (of sizes $4096$ and $128$) with BN, leaky-ReLU activations, and $\ell_2$ normalization after the last layer.
On top of the projection head MLP we had $4$ classification heads into $1K, 2K, 4K$ and $8K$ classes respectively. Each classification head was a simple linear layer without additive bias term. Row-softmax temperature $\tau_{row}$ was set to $0.1$, while column-softmax temperature $\tau_{col}$ -- to $0.05$. Unless mentioned otherwise, evaluation for unsupervised classification (\cref{section:unsupervised_image_classification}) was done strictly using the $1K$-classes classification head. For linear evaluation (\cref{section:linear_classification}) the MLP was dropped and replaced with a single linear layer of $1K$ classes.  

\subsection{Image Augmentations}
We followed the data augmentations of BYOL \cite{grill2020bootstrap} (color jittering, Gaussian blur and solarization), multi-crop \cite{caron2020unsupervised} (two global views of $224\times224$ and six local views of $96\times96$) and nearest neighbor augmentation \cite{Dwibedi_2021_ICCV} (queue for nearest neighbor augmentation was set to $256K$). We refer to \cref{table:multi_crop_and_nn} in \cref{section:ablation} for performance results without multi-crop and nearest neighbor. 

\subsection{Optimization}
\label{section:optimization}
\paragraph{Unsupervised pre-training/classification.} Most of our training hyper-parameters are directly taken from SwAV \cite{caron2020unsupervised}. We used a LARS optimizer \cite{you2017large} with a learning rate of $4.8$ and weight decay of $10^{-6}$. The learning rate was linearly ramped up (starting from $0.3$) over the first $10$ epochs, and then decreased using a cosine scheduler for $790$ epochs with a final value of $0.0048$ (for a total of $800$ epochs). We used a batch size of $4096$ distributed across $64$ NVIDIA V100 GPUs.

\paragraph{Linear evaluation.} Similarly to \cite{chen2021exploring} we used a LARS optimizer \cite{you2017large} with a learning rate of $0.8$ and no weight decay. The learning rate was decreased using a cosine scheduler for $100$ epochs. We used a batch size of $4096$ distributed across $16$ NVIDIA V100 GPUs. We have also tried the SGD optimizer in \cite{he2020momentum} with a batch size of $256$, which gives similar results.

\begin{table}[bt]
  \centering
  \setlength\tabcolsep{4pt}
    \caption{\textbf{ImageNet unsupervised image classification using ResNet-50}. NMI: Normalized Mutual Information, AMI: Adjusted Normalized Mutual Information, ARI: Adjusted Rand-Index, ACC: Clustering accuracy. ${\dagger}$: produced by fitting a $k$-means classifier on the learned representations of the training set (models from official repositories were used), and then running inference on the validation set (results for SimCLRv2 and InfoMin are taken from \cite{zheltonozhskii2020self}). SimSiam provide only 100-epoch model in their official repository. ${*}$: best result taken from the paper’s official repository. Top-3 best methods per-metric are underlined. Best in bold}
  \label{table:unsupervised_image_classification}
  \begin{tabular}{lccccc}
    \toprule
    Method & Epochs & NMI & AMI & ARI & ACC \\
    \midrule
    \multicolumn{6}{l}{\textit{representation learning methods}}\\
    SimCLRv2$^{\dagger}$ \cite{chen2020big} & 1000 & 61.5 & 34.9 & 11.0 & 22.4 \\
    SimSiam$^{\dagger}$ \cite{chen2021exploring} & 100 & 62.2 & 34.9 & 11.6 & 24.9 \\
    SwAV$^{\dagger}$ \cite{caron2020unsupervised} & 800 & 64.1 & 38.8 & 13.4 & 28.1 \\
    MoCoV2$^{\dagger}$ \cite{chen2020improved} & 800 & 66.6 & 45.3 & 12.0 & 30.6 \\
    DINO$^{\dagger}$ \cite{caron2021emerging} & 800 & 66.2 & 42.3 & 15.6 & 30.7 \\
    OBoW$^{\dagger}$ \cite{gidaris2021obow} & 200 & 66.5 & 42.0 & 16.9 & 31.1 \\
    InfoMin$^{\dagger}$ \cite{tian2020makes} & 800 & 68.8 & 48.3 & 14.7 & 33.2  \\
    BarlowT$^{\dagger}$ \cite{DBLP:conf/icml/ZbontarJMLD21} & 1000 & 67.1 & 43.6 & 17.6 & 34.2 \\
    \midrule
    \midrule
    \multicolumn{6}{l}{\textit{clustering based methods}}\\
    SeLa${^*}$ \cite{YM.2020Self-labelling} & 280 & 65.7 & 42.0 & 16.2 & 30.5 \\
    SCAN \cite{van2020scan} & 800+125 & 72.0 & 51.2 & 27.5 & \underline{39.9} \\
    \midrule
    \textbf{Self-Classifier} & 100 & 71.2 & 49.2 & 26.1 & 37.3 \\
    \textbf{Self-Classifier} & 200 & \underline{72.5} & \underline{51.6} & \underline{28.1} & 39.4 \\
    \textbf{Self-Classifier} & 400 & \underline{72.9} & \underline{52.3} & \underline{28.8} & \underline{40.2} \\
    \textbf{Self-Classifier} & 800 & \underline{\textbf{73.3}} & \underline{\textbf{53.1}} & \underline{\textbf{29.5}} & \underline{\textbf{41.1}} \\
    \bottomrule
  \end{tabular}
\end{table}

\begin{table}[bt]
  \centering
  \setlength\tabcolsep{4pt}
    \caption{\textbf{ImageNet-superclasses unsupervised image classification accuracy using ResNet-50}. We define new datasets that contain broad classes which each subsume several of the original ImageNet classes. See Supplementary for details of each superclass. ${\dagger}$: produced by fitting a $k$-means classifier on the learned representations of the training set (models from official repositories were used), and then running inference on the validation set. Results for SCAN and SeLa were produced using ImageNet-pretrained models provided in their respective official repositories}
  \label{table:unsupervised_image_classification_superclasses}
  \begin{tabular}{lcccccc}
    \toprule
     & \multicolumn{6}{c}{Number of ImageNet Superclasses} \\
    Method & 10 & 29 & 128 & 466 & 591 & 1000 \\
    \midrule
    \multicolumn{7}{l}{\textit{representation learning methods}}\\
    SwAV$^{\dagger}$ \cite{caron2020unsupervised} & 79.1 & 69.4 & 58.0 & 46.3 & 34.5 & 28.1 \\
    MoCoV2$^{\dagger}$ \cite{chen2020improved} & 80.0 & 72.8 & 63.8 & 51.4 & 36.8 & 30.6 \\
    DINO$^{\dagger}$ \cite{caron2021emerging} & 79.7 & 71.3 & 60.7 & 49.2 & 37.8 & 30.7 \\
    OBoW$^{\dagger}$ \cite{gidaris2021obow} & 83.9 & 76.5 & 67.4 & 53.5 & 35.7 & 31.1 \\
    BarlowT$^{\dagger}$ \cite{DBLP:conf/icml/ZbontarJMLD21} & 80.2 & 72.1 & 62.7 & 52.7 & 40.9 & 34.2 \\
    \midrule
    \midrule
    \multicolumn{7}{l}{\textit{clustering based methods}}\\
    SeLa \cite{YM.2020Self-labelling} & 55.2 & 44.9 & 40.6 & 36.6 & 37.8 & 30.5 \\
    SCAN \cite{van2020scan} & 85.3 & 79.3 & 71.2 & 59.6 & 44.7 & 39.9 \\
    \midrule
    \textbf{Self-Classifier} & \textbf{85.7} & \textbf{79.7} & \textbf{71.8} & \textbf{60.0} & \textbf{46.7} & \textbf{41.1} \\
    \bottomrule
  \end{tabular}
\end{table}

\begin{table*}[bt]
  \centering
  \tiny
  \caption{\textbf{ImageNet-subsets (BREEDS) unsupervised image classification using ResNet-50}. The four BREEDS datasets are: Entity13, Entity30, Living17 and Nonliving26. NMI: Normalized Mutual Information, AMI: Adjusted Normalized Mutual Information, ARI: Adjusted Rand-Index, ACC: Clustering accuracy. ${\dagger}$: produced by fitting a $k$-means classifier on the learned representations of the training set (models from official repositories were used), and then running inference on the validation set. Results for SCAN and SeLa were produced using ImageNet-pretrained models provided in their respective official repositories}
  \label{table:unsupervised_image_classification_breeds}
  \begin{tabular}{lcccc|cccc|cccc|cccc}
    \toprule
     & \multicolumn{4}{c}{\underline{Entity13}} & \multicolumn{4}{c}{\underline{Entity30}} & \multicolumn{4}{c}{\underline{Living17}} & \multicolumn{4}{c}{\underline{Nonliving26}}\\
    Method & NMI & AMI & ARI & ACC & NMI & AMI & ARI & ACC & NMI & AMI & ARI & ACC & NMI & AMI & ARI & ACC \\
    \midrule
    \multicolumn{17}{l}{\textit{representation learning methods}}\\
    SwAV$^{\dagger}$ \cite{caron2020unsupervised} & 64.8 & 39.9 & 15.2 & 75.6 & 64.6 & 39.4 & 15.1 & 70.5 & 61.0 & 40.3 & 15.7 & 85.2 & 62.0 & 41.1 & 19.2 & 63.1 \\
    MoCoV2$^{\dagger}$ \cite{chen2020improved} & 67.3 & 46.6 & 14.7 & 79.0 & 66.4 & 45.8 & 15.1 & 74.6 & 61.2 & 45.7 & 16.3 & 89.7 & 63.3 & 46.2 & 19.3 & 66.2 \\
    DINO$^{\dagger}$ \cite{caron2021emerging} & 67.2 & 43.7 & 18.0 & 78.2 & 66.8 & 43.2 & 18.1 & 73.7 & 63.8 & 45.1 & 19.6 & 88.2 & 63.8 & 43.9 & 21.8 & 66.7 \\
    OBoW$^{\dagger}$ \cite{gidaris2021obow} & 66.4 & 42.3 & 17.5 & 82.2 & 64.9 & 40.7 & 16.4 & 77.6 & 53.8 & 34.0 & 12.0 & 91.1 & 64.8 & 45.4 & 22.9 & 67.9 \\
    BarlowT$^{\dagger}$ \cite{DBLP:conf/icml/ZbontarJMLD21} & 68.2 & 45.5 & 20.5 & 77.7 & 67.7 & 45.1 & 20.7 & 73.0 & 64.7 & 47.2 & 22.2 & 88.0 & 64.8 & 45.7 & 24.9 & 66.7 \\
    \midrule
    \midrule
    \multicolumn{17}{l}{\textit{clustering based methods}}\\
    SeLa \cite{YM.2020Self-labelling} & 67.6 & 44.8 & 19.4 & 50.7 & 68.2 & 45.7 & 21.2 & 52.6 & \textbf{71.8} & \textbf{53.9} & \textbf{29.7} & 80.8 & 68.9 & 46.6 & 24.6 & 67.1 \\
    SCAN \cite{van2020scan} & 72.4 & 52.3 & 29.2 & 83.7 & 71.3 & 50.8 & 27.8 & 80.0 & 65.2 & 49.4 & 25.3 & \textbf{92.5} & 70.0 & 53.6 & 33.4 & 74.4 \\
    \midrule
        \textbf{Self-Classifier} & \textbf{73.6} & \textbf{54.1} & \textbf{30.7} & \textbf{84.4} & \textbf{72.9} & \textbf{53.4} & \textbf{29.8} & \textbf{81.0} & 67.2 & 51.8 & 26.4 & 90.8 & \textbf{72.2} & \textbf{57.0} & \textbf{36.8} & \textbf{76.7} \\
    \bottomrule
  \end{tabular}
\end{table*}

\section{Results}
\label{section:experiments} 
\subsection{Unsupervised Image Classification}
\label{section:unsupervised_image_classification}
We evaluate our approach on the task of unsupervised image classification using the large-scale ImageNet dataset (\cref{table:unsupervised_image_classification,table:unsupervised_image_classification_superclasses,table:unsupervised_image_classification_breeds}). We report the standard clustering metrics: Normalized Mutual Information (NMI), Adjusted Normalized Mutual Information (AMI), Adjusted Rand-Index (ARI), and Clustering Accuracy (ACC).  

Our approach sets a new state-of-the-art performance for unsupervised image classification using ImageNet, on all four metrics (NMI, AMI, ARI and ACC), even when trained for a substantial lower number of epochs (\cref{table:unsupervised_image_classification}). We compare our approach to the latest large-scale deep clustering methods \cite{YM.2020Self-labelling,van2020scan} that have been explicitly evaluated on ImageNet. Additionally, we also compare our approach to the latest self-supervised representation learning methods (using ImageNet-pretrained models provided in their respective official repositories) after fitting a $k$-means classifier to the learned representations computed on the training set. For all methods we run inference on the validation set (unseen during training). 

The current state-of-the-art approach, SCAN \cite{van2020scan}, is a multi-stage algorithm that involves: 1) pre-training (800 epochs); 2) offline $k$-nearest neighbor mining; 3) clustering (100 epochs); and 4) self-labeling and fine-tuning (25 epochs). In contrast, \textit{Self-Classifier} is a single-stage simple-to-implement model (\cref{algo:pseudo_code}) that is trained only with minibatch SGD. At only 200 epochs \textit{Self-Classifier} already outperforms SCAN with 925 epochs.

SCAN provided an interesting qualitative analysis of alignment of its unsupervised class predictions to a certain (single) level of the default (WordNet) ImageNet semantic hierarchy. In contrast, here we propose a more diverse set of quantitative metrics to evaluate the performance of self-supervised classification methods on various levels of the default ImageNet hierarchy, as well as on several hierarchies of carefully curated ImageNet subpopulations (BREEDS \cite{santurkar2020breeds}). We believe that this new set of hierarchical alignment metrics expanding on the leaf-only metric used so far, will allow deeper investigation of how self-supervised classification approaches perceive the internal taxonomy of classes of unlabeled data they are applied to, exposing their strength and weaknesses in a new and interesting light. We use these new metrics to compare our proposed approach to previous unsupervised clustering work \cite{van2020scan,YM.2020Self-labelling}, as well as state-of-the-art representation learning work \cite{chen2020big,chen2021exploring,caron2020unsupervised,chen2020improved,caron2021emerging,gidaris2021obow,tian2020makes,DBLP:conf/icml/ZbontarJMLD21}.

In \cref{table:unsupervised_image_classification_superclasses} we report results for different numbers of ImageNet superclasses (10, 29, 128, 466 and 591) resulting from cutting the default (WordNet) ImageNet hierarchy on different levels. See Supplementary for details of each superclass. The results in this table, that are significantly higher then the result for leaf (1000) classes for any hierarchy level, indicate that examples misclassified on the leaf level tend to be assigned to other clusters from within the same superclass. Furthermore, we see that \textit{Self-Classifier} consistently outperforms previous work on all hierarchy levels.

In \cref{table:unsupervised_image_classification_breeds} we report the results on four ImageNet subpopulation datasets of BREEDS \cite{santurkar2020breeds}. These datasets are accompanied by class hierarchies re-calibrated by \cite{santurkar2020breeds} such that classes on same hierarchy level are of the same visual granularity. Each dataset contains a specific subpopulation of ImageNet, such as `Entities', `Living' things and `Non-living' things, allowing for a more fine-grained evaluation of hierarchical alignment of self-supervised classification predictions. Again, we see consistent improvement of \textit{Self-Classifier} over previous work and self-supervised representation baselines.

\subsection{Image Classification with Linear Models}
\label{section:linear_classification}
We evaluate the quality of our unsupervised features using the standard linear classification protocol. Following the self-supervised pre-training stage, we freeze the features and train on top of it a supervised linear classifier (a single fully-connected layer). This classifier operates on the global average pooling features of a ResNet. \cref{table:linear_eval_imagenet} summarizes the results and comparison to the state-of-the-art methods for various number of training budgets (100 to 800 epochs).

In addition to good results for unsupervised classification (\cref{section:unsupervised_image_classification}), \textit{Self-Classifier} additionally achieves results comparable to state of the art for linear classification evaluation using ImageNet. Specifically, as detailed in \cref{table:linear_eval_imagenet}, it is one of the top-3 result for 3 out of 4 of the training budgets reported, and top-1 in the 100 epochs category.

\begin{table}[bt]
  \centering
  \caption{\textbf{ImageNet linear classification using ResNet-50}. Top-1 accuracy vs. number of training epochs. Top-3 best methods per-category are underlined}
  \label{table:linear_eval_imagenet}
  \setlength\tabcolsep{8pt}
  \begin{tabular}{l|cccc}
    \toprule
     & \multicolumn{4}{c}{Number of Training Epochs} \\
    Method & 100 & 200 & 400 & 800 \\
    \midrule
    Supervised & 76.5 & -- & -- & -- \\
    \midrule
    SimCLR \cite{DBLP:journals/corr/abs-2002-05709} & 66.5 & 68.3 & 69.8 & 70.4 \\
    MoCoV2 \cite{chen2020improved} & 67.4 & 67.5 & 71.0 & 71.1 \\
    SimSiam \cite{chen2021exploring} & 68.1 & 70.0 & 70.8 & 71.3 \\
    SimCLRv2 \cite{chen2020big} & -- & -- & -- & 71.7 \\
    InfoMin \cite{tian2020makes} & -- & -- & -- & 73.0 \\
    BarlowT \cite{DBLP:conf/icml/ZbontarJMLD21} & -- & -- & -- & 73.2 \\
    OBoW \cite{gidaris2021obow} & -- & \underline{73.8} & -- & -- \\
    BYOL \cite{grill2020bootstrap} & 66.5 & 70.6 & 73.2 & 74.3 \\
    NNCLR \cite{Dwibedi_2021_ICCV} & \underline{69.4} & 70.7 & \underline{74.2} & \underline{74.9} \\
    DINO \cite{caron2021emerging} & -- & -- & -- & \underline{\textbf{75.3}} \\
    SwAV \cite{caron2020unsupervised} & \underline{72.1} & \underline{\textbf{73.9}} & \underline{\textbf{74.6}} & \underline{\textbf{75.3}} \\
    \midrule
    \textbf{Self-Classifier}  & \underline{\textbf{72.4}} & \underline{73.5} & \underline{74.2} & 74.1 \\
    \bottomrule
  \end{tabular}
\end{table}

\subsection{Transfer Learning}
We further evaluate the quality of our unsupervised features by transferring them to other tasks - object detection and instance segmentation. \cref{table:transfer_learning} reports results for VOC07+12 \cite{everingham2010pascal} and COCO \cite{lin2014microsoft} datasets. We fine-tune our pre-trained model end-to-end in the target datasets using the public codebase from MoCo \cite{he2020momentum}. We obtain significant ($\sim 2\%$) improvements in the more challenging COCO det/seg over all the self-supervised baselines.

\begin{table}[ht]
  \centering
  \tiny
  \caption{\textbf{Transfer learning: object detection and instance segmentation}. Results for other methods are taken from \cite{DBLP:conf/icml/ZbontarJMLD21}}
  \label{table:transfer_learning}
  \begin{tabular}{l|ccc|ccc|ccc}
    \toprule
    Method & \multicolumn{3}{c}{\underline{VOC07+12 det}} & \multicolumn{3}{c}{\underline{COCO det}} & \multicolumn{3}{c}{\underline{COCO seg}} \\
     & AP & AP$_{50}$ & AP$_{75}$ & AP & AP$_{50}$ & AP$_{75}$ & AP & AP$_{50}$ & AP$_{75}$ \\
    \midrule
    Supervised & 53.5 & 81.3 & 58.8 & 38.2 & 58.2 & 41.2 & 33.3 & 54.7 & 35.2 \\
    \midrule
    MoCo-v2\cite{chen2020improved} & \textbf{57.4} & 82.5 & \textbf{64.0} & 39.3 & 58.9 & 42.5 & 34.4 & 55.8 & 36.5 \\
    SwAV\cite{caron2020unsupervised} & 56.1 & \textbf{82.6} & 62.7 & 38.4 & 58.6 & 41.3 & 33.8 & 55.2 & 35.9 \\
    SimSiam\cite{chen2021exploring} & 57.0 & 82.4 & 63.7 & 39.2 & 59.3 & 42.1 & 34.4 & 56.0 & 36.7 \\
    BarlowT\cite{DBLP:conf/icml/ZbontarJMLD21} & 56.8 & \textbf{82.6} & 63.4 & 39.2 & 59.0 & 42.5 & 34.3 & 56.0 & 36.5 \\
    \midrule
    \midrule
    \textbf{Self-Classifier} & 56.6 & 82.4 & 62.6 & \textbf{41.5} & \textbf{61.3} & \textbf{45.0} & \textbf{36.1} & \textbf{58.1} & \textbf{38.7} \\
    \bottomrule
  \end{tabular}
\end{table}

\subsection{Qualitative Results}
In Supplementary, we visualize and analyse a subset of high/low accuracy classes predicted by \textit{Self-Classifier} on \textbf{unseen data} (ImageNet validation).

\section{Ablation Study}
\label{section:ablation}
In this section, we evaluate the impact of the design choices of \textit{Self-Classifier}. Namely, the loss function, the number of classes ($C$), number of classification heads, fixed vs learnable classifier, MLP architecture, Softmax temperatures (row and column), batch-size, some of the augmentations choices, and NN queue length. We evaluate the different models after 100 self-supervised epochs and report results on ImageNet validation set. We report both the K-NN (K=20) classifier accuracy (evaluating the learned representations) and the unsupervised clustering accuracy (evaluating unsupervised classification performance).

\textbf{Loss function}. 
For both illustrating the generality of our proposed loss function and making more direct comparison with the unsupervised classification state-of-the-art (SCAN \cite{van2020scan}), in \cref{table:ablation_loss} we report the results of running SCAN official code, while replacing their loss function (in the clustering step) with ours (Eq. \eqref{eq:L_sym}) and keeping everything else (e.g. classification heads and augmentations) same as in SCAN. As we can see, our proposed loss generalizes well and improves SCAN result (e.g. by $1.5\%$ ARI and $0.5\%$ ACC). Further results improvements are obtained using our full method (as shown in \cref{table:unsupervised_image_classification}).

\begin{table}[bt]
  \centering
  \setlength\tabcolsep{4pt}
    \caption{\textbf{Ablation: loss function generality}. 
    For column definitions see \cref{table:unsupervised_image_classification}. 
    SCAN + Eq. \eqref{eq:L_sym} is SCAN with clustering step loss replaced with ours.
    }
  \label{table:ablation_loss}
  \begin{tabular}{lcccc}
    \toprule
    Method & NMI & AMI & ARI & ACC \\
    \midrule
    SCAN \cite{van2020scan} & 72.0 & 51.2 & 27.5 & 39.9 \\
    SCAN + our loss (Eq. \eqref{eq:L_sym}) & \textbf{72.7} & \textbf{52.2} & \textbf{29.0} & \textbf{40.4} \\
    \bottomrule
  \end{tabular}
\end{table}

\textbf{Number of classes and classification heads.} \cref{table:ablation_num_cls} reports the results for various number of classes and classification heads. 
Very interestingly, and somewhat contrary to the intuition of previous unsupervised classification works \cite{van2020scan,YM.2020Self-labelling} who used the same number of classes for all heads, 
we found that using a different number of classes for each head while still keeping the total number of parameters constant (e.g. 15x1k vs. 1k+2k+4k+8k) improves results on both metrics. We believe that such a learning objective forces the model to learn a representation that is more invariant to the number of classes, thus improving its generalization performance. 

\textbf{Fixed/Learnable classifier.} As expected, we found that a learnable classifier performs better than a fixed one (\cref{table:fixes_vs_learnable}).

\begin{table}[!htb]
    \caption{\textbf{Ablation study}. 
    After 100 epochs, reporting performance for ImageNet as accuracy of $<$'k-NN'~$\vert$~'unsupervised clustering'$>$ in each experiment.
    }
    
    \begin{subtable}{\linewidth}
      \centering
        \caption{Classification heads. $^{(2k)}~^{(4k)}~^{(8k)}$: 2k, 4k and 8k over-clustering accuracy.}
        \tiny
        \label{table:ablation_num_cls}
        \begin{tabular}{c|ccccccc}
        \toprule
         & 1$\times$1k & 5$\times$1k & 10$\times$1k & 15$\times$1k & 1$\times$2k & 1$\times$4k & 1k+2k+4k+8k \\
        \midrule
        Acc. (\%) & 59.6$\vert$34.1 & 58.7$\vert$34.0 & 58.6$\vert$33.5 & 58.8$\vert$33.9 & 59.3$\vert$38.8$^{(2k)}$ & 57.0$\vert$42.9$^{(4k)}$ &  \textbf{61.7}$\vert$\textbf{37.3},40.6$^{(2k)}$,44.2$^{(4k)}$,48.0$^{(8k)}$\\
        \bottomrule
        \end{tabular}
    \end{subtable}
    
    \begin{subtable}{.38\linewidth}
      \centering
        \caption{Softmax Temperature}
        \label{table:ablation_sm_temp}
        \begin{tabular}{c|cc}
        \toprule
         & \multicolumn{2}{c}{$\tau_{row}$} \\
        $\tau_{column}$ & 0.07 & 0.1 \\
        \midrule
        0.03 & 59.9$\vert$36.9 & 59.2$\vert$36.9 \\
        0.05 & 58.9$\vert$29.2 & \textbf{61.7}$\vert$\textbf{37.3} \\
        \bottomrule
        \end{tabular}
    \end{subtable}%
    \hfill 
    \begin{subtable}{.62\linewidth}
      \centering
        \caption{MLP architecture}
        \label{table:ablation_mlp}
        \begin{tabular}{c|ccc}
        \toprule
         & \multicolumn{2}{c}{MLP output layer} \\
        MLP hidden layer(s) & 128 & 256 \\
        \midrule
        1x4096 & \textbf{61.7}$\vert$\textbf{37.3} & 61.3$\vert$33.5 \\
        2x4096 & 60.9$\vert$36.4 & 60.4$\vert$33.6 \\
        2x8192 & 60.0$\vert$36.9 & 59.6$\vert$36.7 \\
        \bottomrule
        \end{tabular}
    \end{subtable}

    \begin{subtable}{.38\linewidth}
      \centering
        \caption{Fixed / Learnable classifier}
        \label{table:fixes_vs_learnable}
        \begin{tabular}{c|cc}
        \toprule
         & Fixed & Learnable \\
        \midrule
        Acc. (\%) & 57.6$\vert$32.2 & \textbf{61.7}$\vert$\textbf{37.3} \\
        \bottomrule
        \end{tabular}
    \end{subtable}%
    \hfill 
    \begin{subtable}{.62\linewidth}
      \centering
        \caption{Nearest neighbor queue length}
        \label{table:ablation_nn_q_len}
        \begin{tabular}{c|cccc}
        \toprule
        Queue len. & 128K & 256K & 512K & 1M \\
        \midrule
        Acc. (\%) & 59.2$\vert$36.8 & \textbf{61.7}$\vert$\textbf{37.3} & 60.3$\vert$36.9 & 56.8$\vert$35.5 \\
        \bottomrule
        \end{tabular}
    \end{subtable} 
    
        \begin{subtable}{\linewidth}
      \centering
        \caption{Batch size}
        \label{table:ablation_batch_size}
        \begin{tabular}{c|ccccc}
        \toprule
        Batch Size & 256 & 512 & 1024 & 2048 & 4096 \\
        \midrule
        Acc. (\%) & 49.0$\vert$20.9 & 52.2$\vert$23.1 & 54.5$\vert$26.8 & 57.0$\vert$35.1 & \textbf{61.7}$\vert$\textbf{37.3} \\
        \bottomrule
        \end{tabular}
    \end{subtable}
    
\end{table}

\textbf{MLP architecture.} \cref{table:ablation_mlp} reports the results for various sizes of hidden/output layers. Surprisingly, we found that decreasing the number of hidden layers and their size improves both metrics. As a result, our best model (4096/128 MLP) has $30\%$ less parameters than the model used in SCAN \cite{van2020scan} (that used 2048 sized input to its cls. heads). In addition, we verified there is no peak performance difference between ReLU and leaky-ReLU activation in the MLP.

\textbf{Softmax Temperature.} Table \ref{table:ablation_sm_temp} reports the results for a range of Row/Column softmax temperatures. We found that the ratio between the two temperatures is important for performance (specifically clustering accuracy). The model is robust to ratios (row over column) in the range of 2.0 - 3.5. 

\textbf{Batch Size.} Table \ref{table:ablation_batch_size} reports the results for a range of batch size values (256 to 4096). Similarly to previous self-supervised work (and specifically clustering-based), performance improves as we increase batch size.

\textbf{Multi-crop and nearest neighbor augmentations.} \cref{table:multi_crop_and_nn} reports the impact of removing multi-crop \cite{caron2020unsupervised} and nearest neighbor augmentations \cite{Dwibedi_2021_ICCV} on linear classification accuracy and compares to other state-of-the-art methods.

\begin{table}[bt]
\begin{center}
\setlength\tabcolsep{3pt}
\small
\caption{Performance without multi-crop and without nearest neighbor augmentations. ImageNet Top-1 linear classification accuracy after 100 epochs. \cite{DBLP:journals/corr/abs-2002-05709,grill2020bootstrap,caron2020unsupervised,chen2020improved} are taken from \cite{chen2021exploring}}
\label{table:multi_crop_and_nn}
\begin{tabular}{c|cccccc}
\toprule
 & SimCLR \cite{DBLP:journals/corr/abs-2002-05709} & BYOL \cite{grill2020bootstrap} & SwAV \cite{caron2020unsupervised} & MoCoV2 \cite{chen2020improved} & SimSiam \cite{chen2021exploring} & \textbf{Ours}  \\
\midrule
Acc. (\%) & 66.5 & 66.5 & 66.5 & 67.4 & \textbf{68.1} & \textbf{68.1}  \\
\bottomrule
\end{tabular}
\end{center}
\end{table}

\textbf{Nearest neighbor queue length.} The model is somewhat robust to a queue length in the range of 128K - 512K (\cref{table:ablation_nn_q_len}), while increasing it further decreases performance. Most likely due to stale embeddings (as noted by \cite{Dwibedi_2021_ICCV} as well).

\section{Comparative Analysis}
A common and critical element of all self-supervised learning methods is collapse prevention. In this section, we discuss the various approaches of state-of-the-art models for preventing collapse. The approaches can be categorized into two categories: 1) negative samples; and 2) stop-grad operation. Where in practice, stop-grad operation includes two more sub-categories: 2.a) external clustering; and 2.b) momentum encoder. In this paper, we propose a third and completely new approach for collapse prevention - a non-collapsing loss function, i.e., a loss function without degenerate optimal solutions.

\textbf{Negative samples.} SimCLR \cite{DBLP:journals/corr/abs-2002-05709} and Moco \cite{he2020momentum} prevent collapse by utilizing negative pairs to explicitly force dissimilarity.

\textbf{External clustering.} SwAV \cite{caron2020unsupervised}, SeLa \cite{YM.2020Self-labelling} and SCAN \cite{van2020scan} prevent collapse by utilizing external clustering algorithm such as K-Means (SCAN) or Sinkhorn-Knopp (SwAV/SeLa) for generating pseudo-labels.

\textbf{Momentum encoder.} MoCo \cite{he2020momentum}, BYOL \cite{grill2020bootstrap} and DINO \cite{caron2021emerging} prevent collapse by utilizing the momentum encoder proposed by MoCo. The momentum encoder generates a different yet fixed pseudo target in every iteration.

\textbf{Stop-grad operation.} SimSiam \cite{chen2021exploring} prevent collapse by applying a stop-grad operation on one of the views, which acts as a fixed pseudo label. In fact, except for SimCLR, all of the above methods can be simply differentiated by where exactly a stop-grad operation is used. SwAV/SeLa/SCAN apply a stop-grad operation on the clustering phase, while MoCo/BYOL/DINO apply a stop-grad operation on a second network that is used for generating assignments.

\textbf{Non-collapsing loss function.} In contrast, we show mathematically (\cref{section:theoretical}) and empirically (\cref{section:experiments}) that \textit{Self-Classifier} prevents collapse with a novel loss function \cref{eq:ce_ours} and without the use of external clustering, pseudo-labels, momentum encoder, stop-grad nor negative pairs. More specifically, a collapsing solution is simply not in the set of optimal solutions of our proposed loss, which makes it possible to train \textit{Self-Classifier} using just a single network and a simple SGD.

\section{Conclusions and Limitations}
We introduced \textit{Self-Classifier}, a new approach for unsupervised end-to-end classification and self-supervised representation learning. Our approach is mathematically justified and simple to implement. It sets a new state-of-the-art performance for unsupervised classification on ImageNet and achieves comparable to state of the art results for unsupervised representation learning. We provide a thorough investigation of our method in a series of ablation studies. Furthermore, we propose a new hierarchical alignment quantitative metric for self-supervised classification establishing baseline performance for a wide range of methods and showing advantages of our proposed approach in this new task.
\textit{Limitations} of this paper include: (i) our method relies on knowledge of the number of classes, but in some cases it might not be optimal as the true number of classes should really be dictated by the data itself. In this paper we relax this potential weakness by introducing the notion of multiple classification heads, but we believe further investigation would be an interesting future work direction; (ii) one of the most common sources of error we observed is merging of nearby classes (e.g. different breeds of cat), introducing additional regularization for reducing this artifact is also an interesting direction of future work.

%
%
\bibliographystyle{splncs04}
\bibliography{egbib}

\clearpage
\section{Visualization of High/Low Accuracy Classes Predicted by Self-Classifier}
\label{section:low_high_acc_visualization}

\begin{figure}[H]
\small
\centering
\begin{tabular}{ccccccc}
    \includegraphics[width=0.15\linewidth]{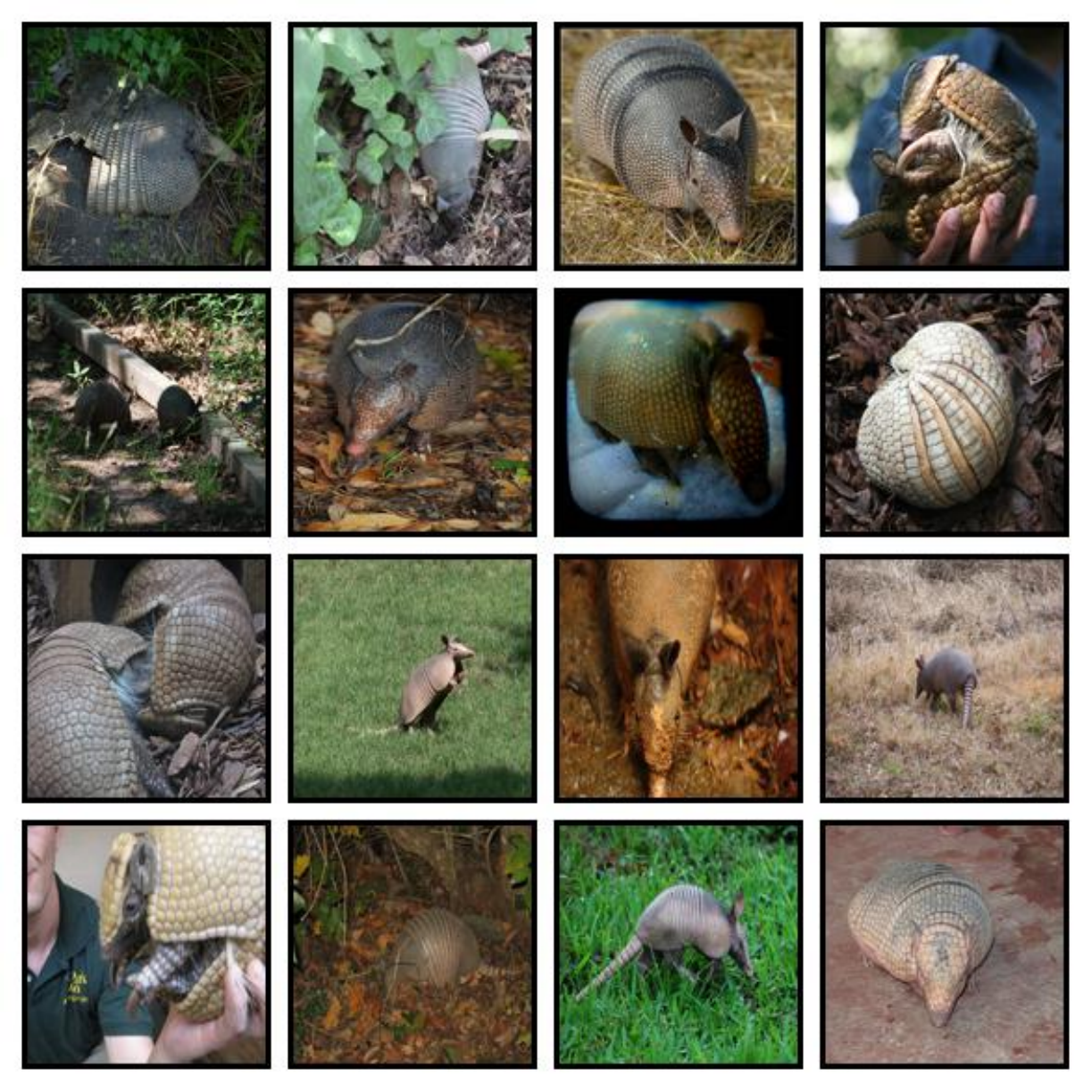} & \includegraphics[width=0.15\linewidth]{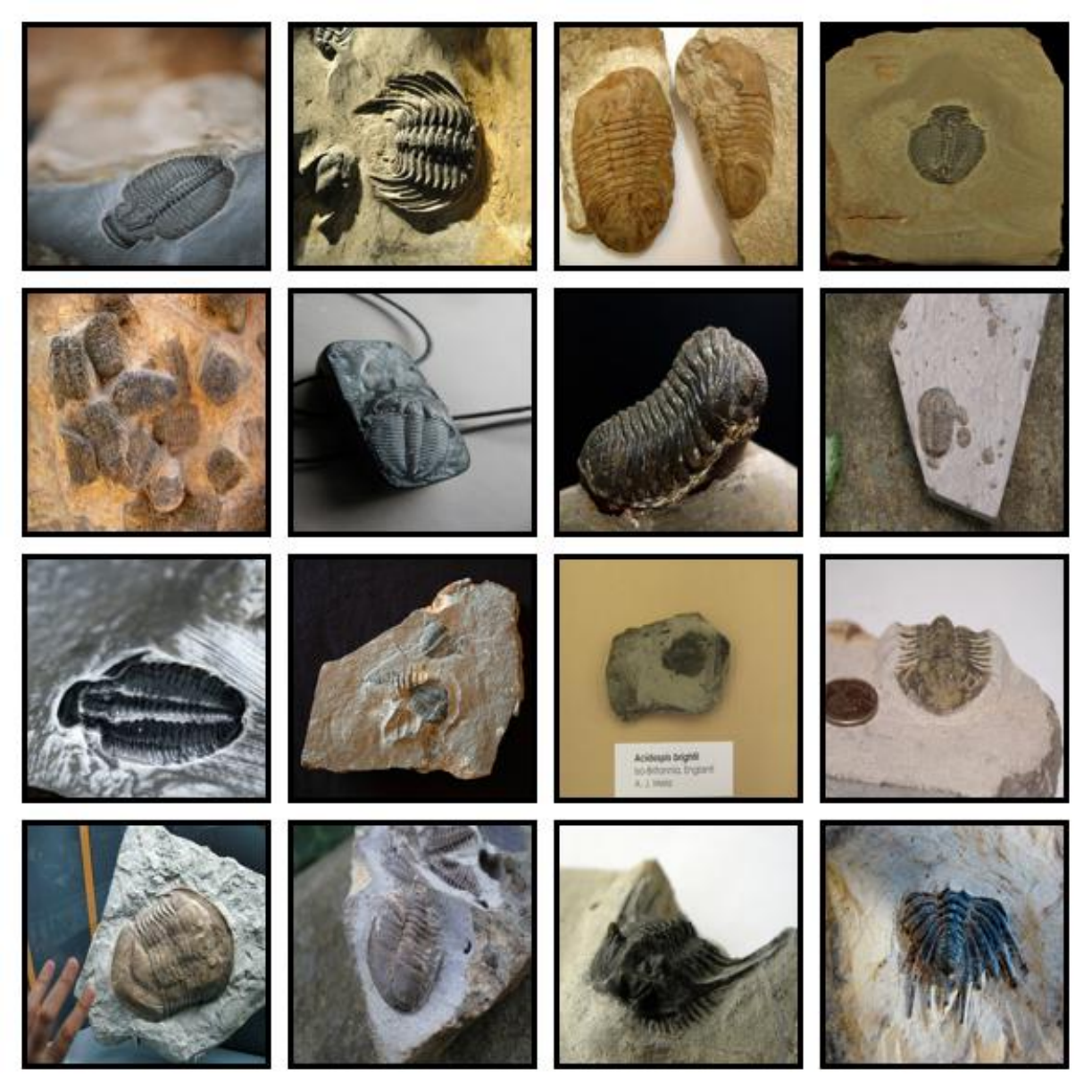} & \includegraphics[width=0.15\linewidth]{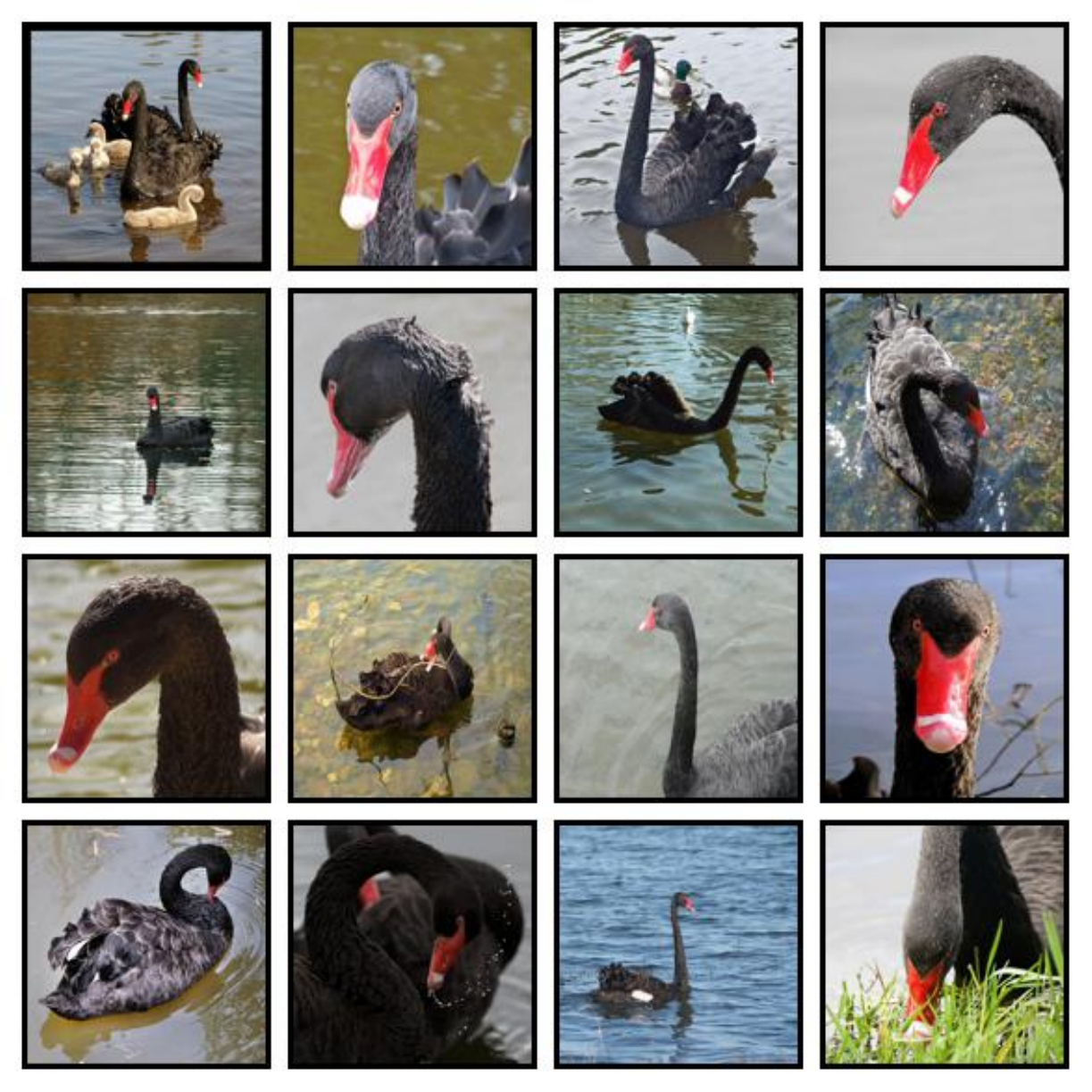} & \includegraphics[width=0.15\linewidth]{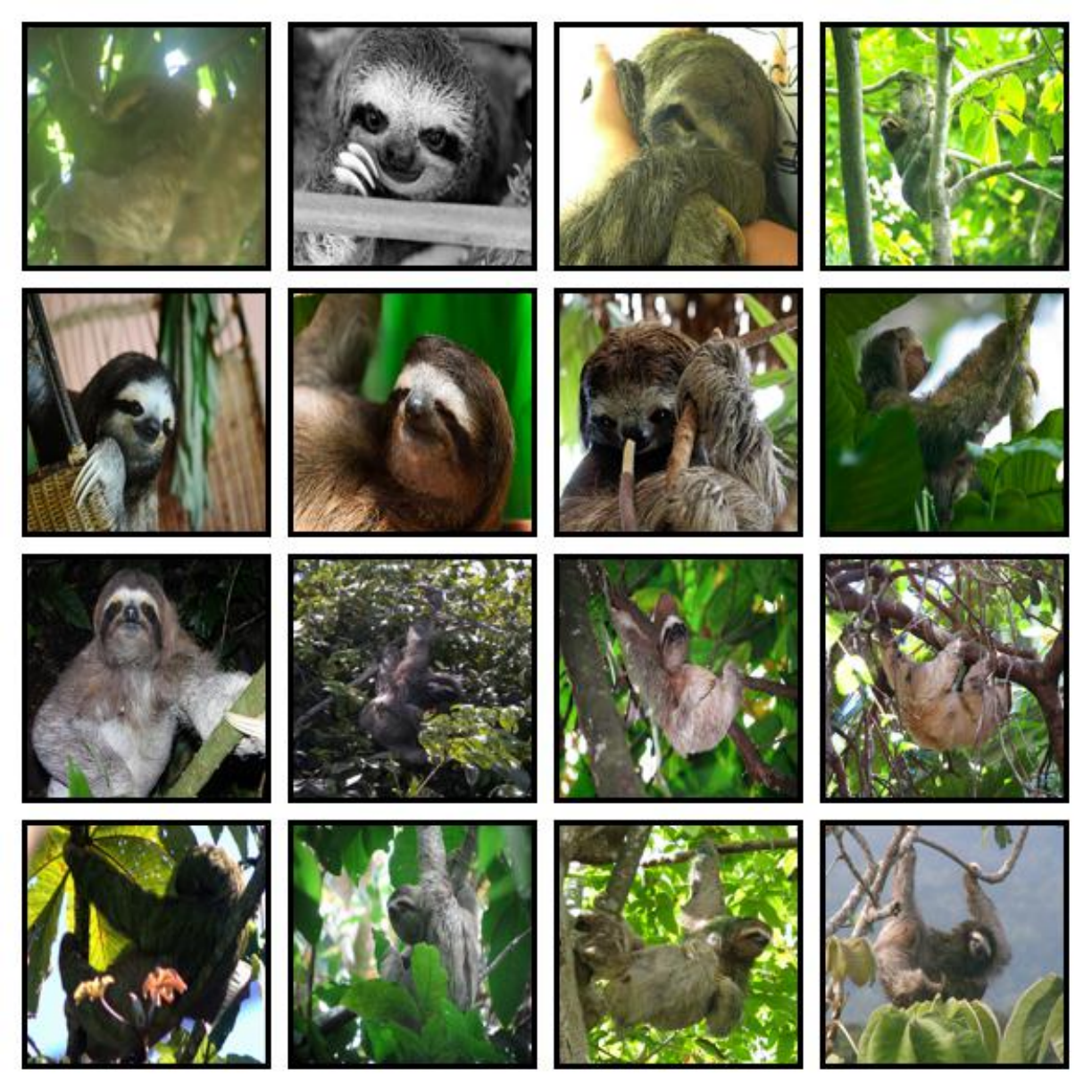} & \includegraphics[width=0.15\linewidth]{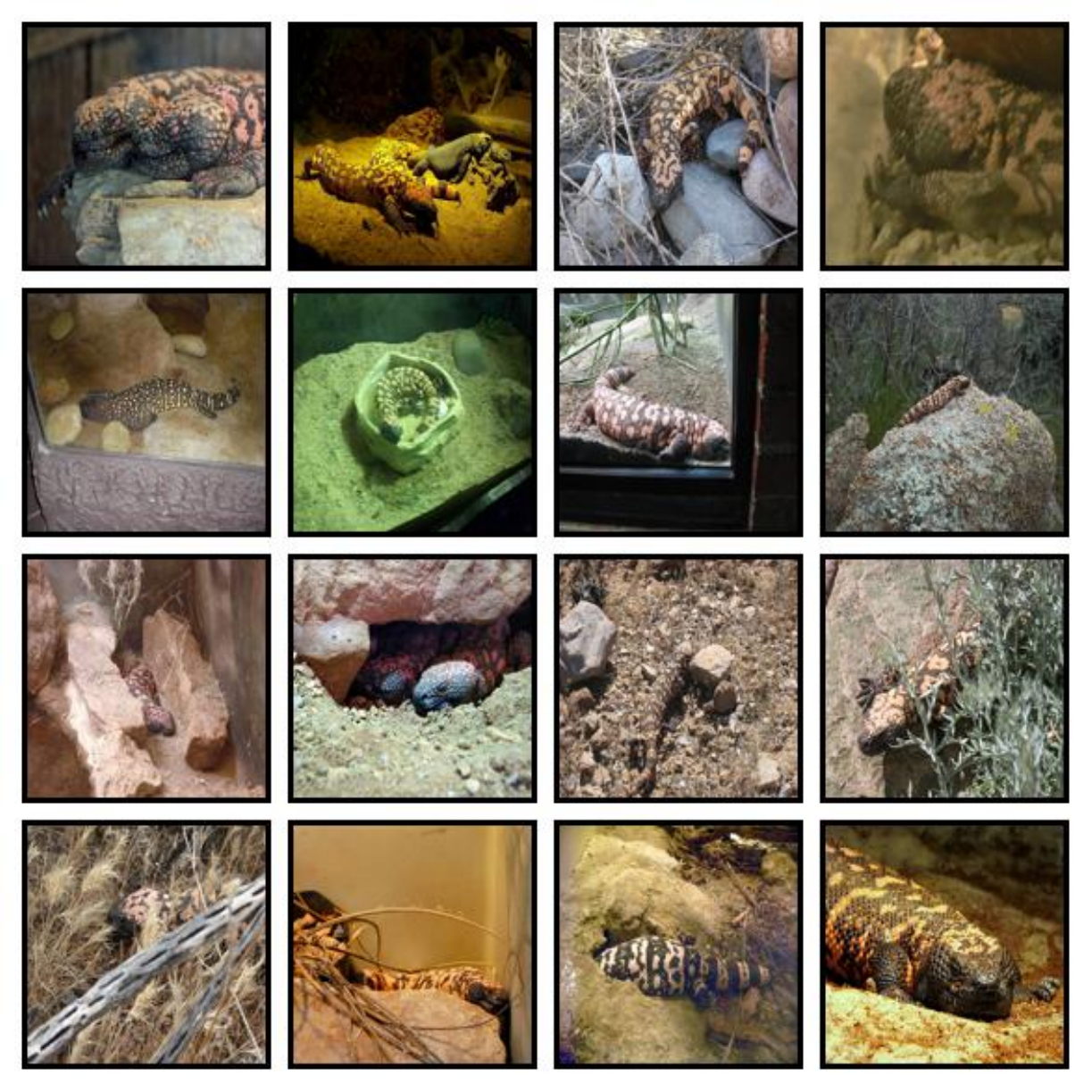} \\
  \includegraphics[width=0.15\linewidth]{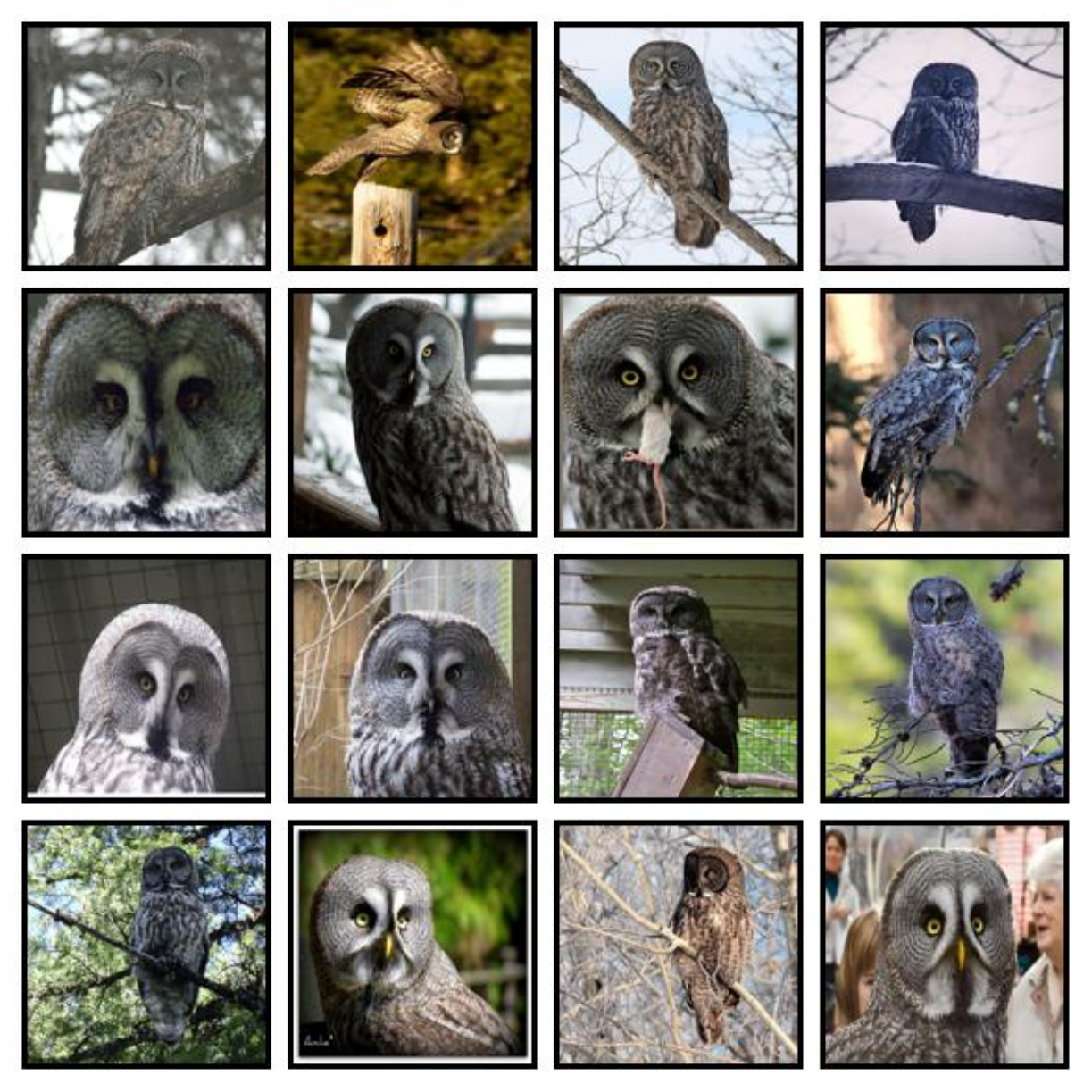} & \includegraphics[width=0.15\linewidth]{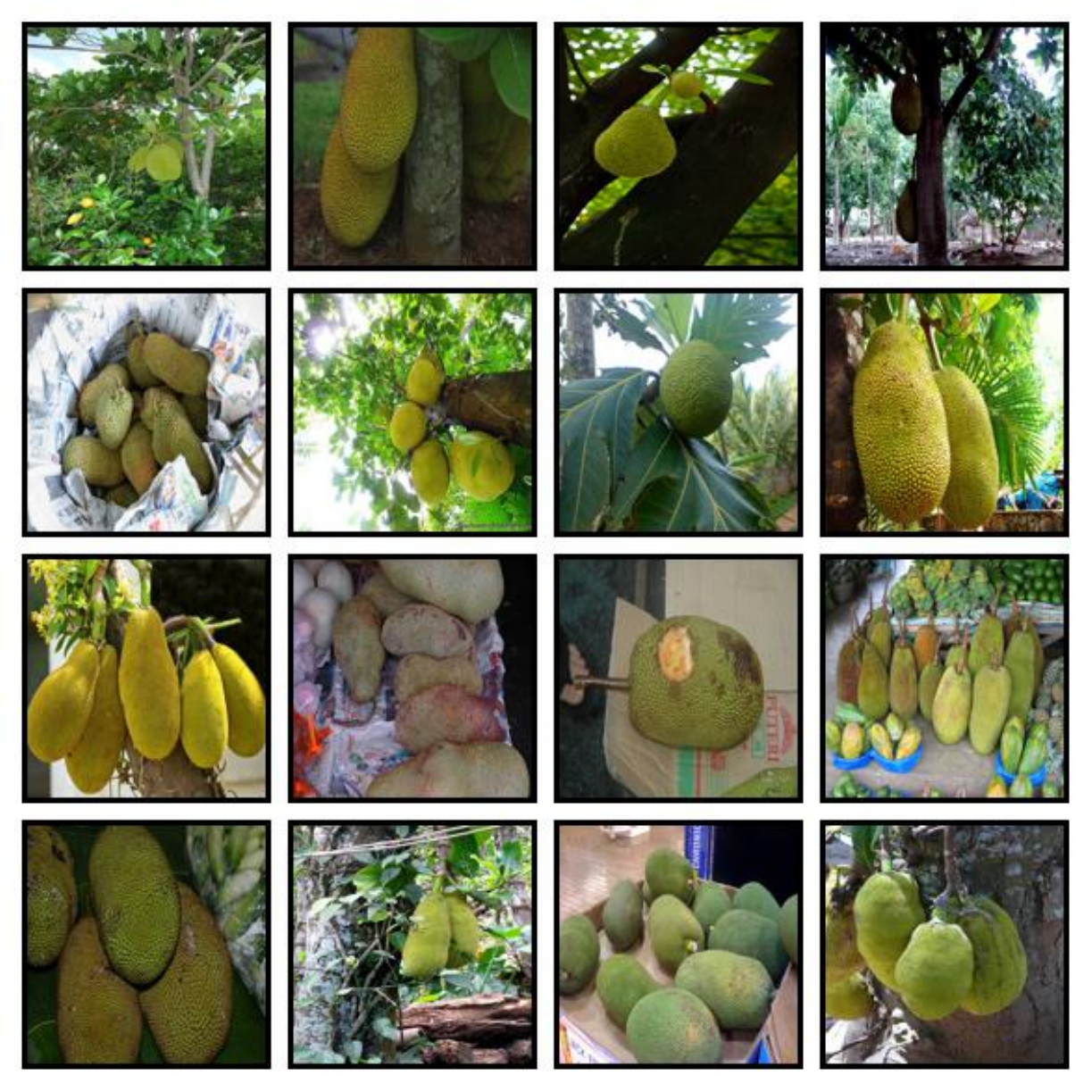} & \includegraphics[width=0.15\linewidth]{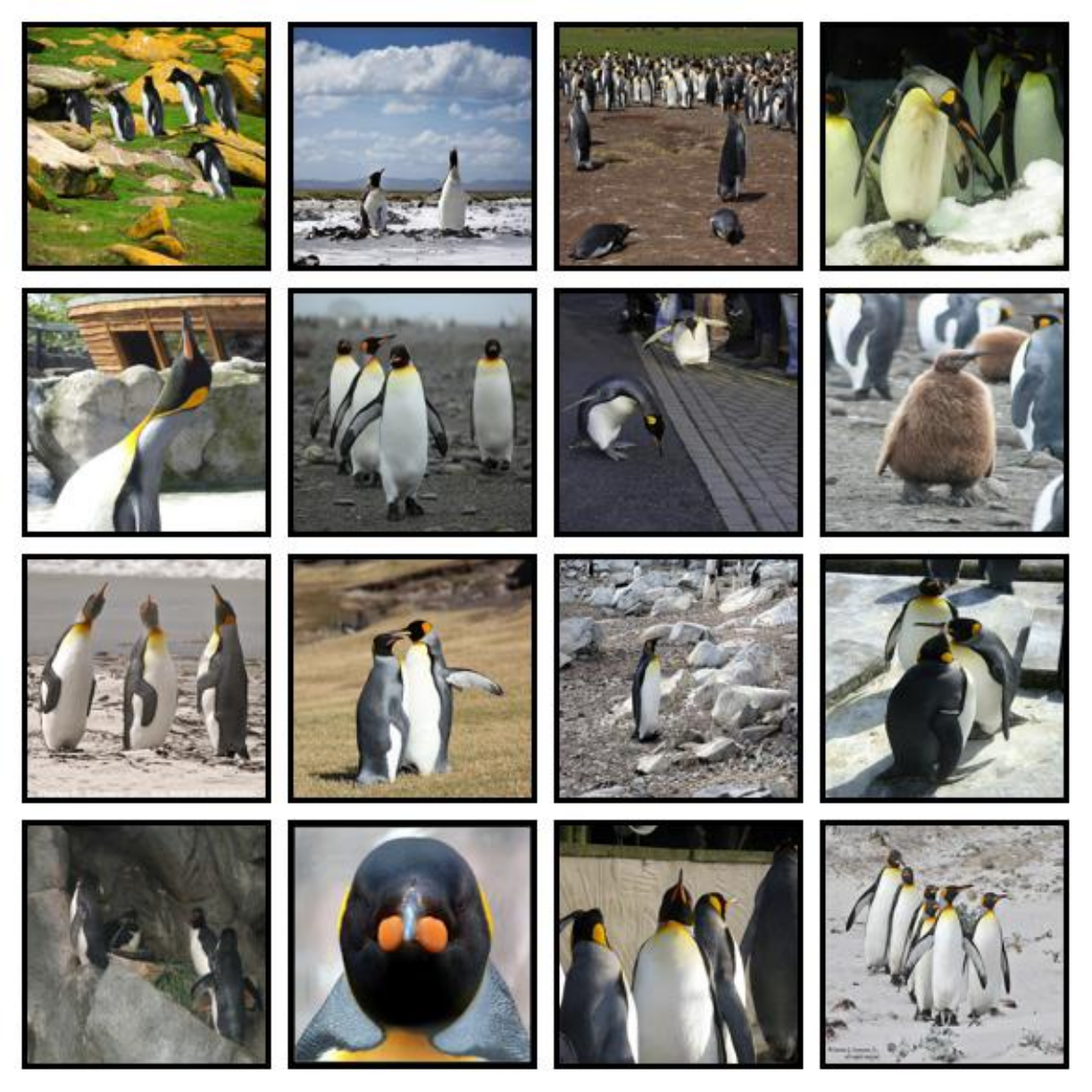} & \includegraphics[width=0.15\linewidth]{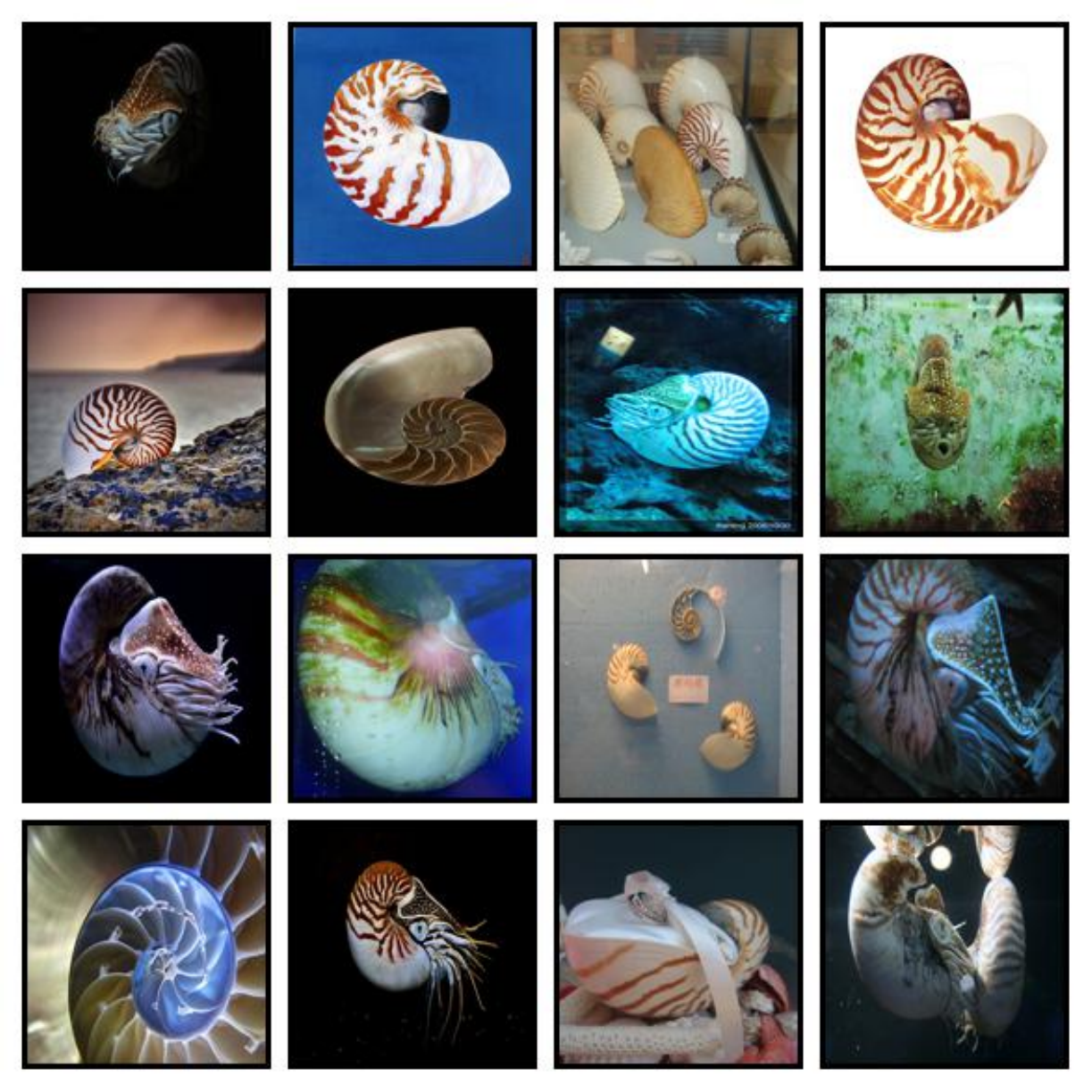} & \includegraphics[width=0.15\linewidth]{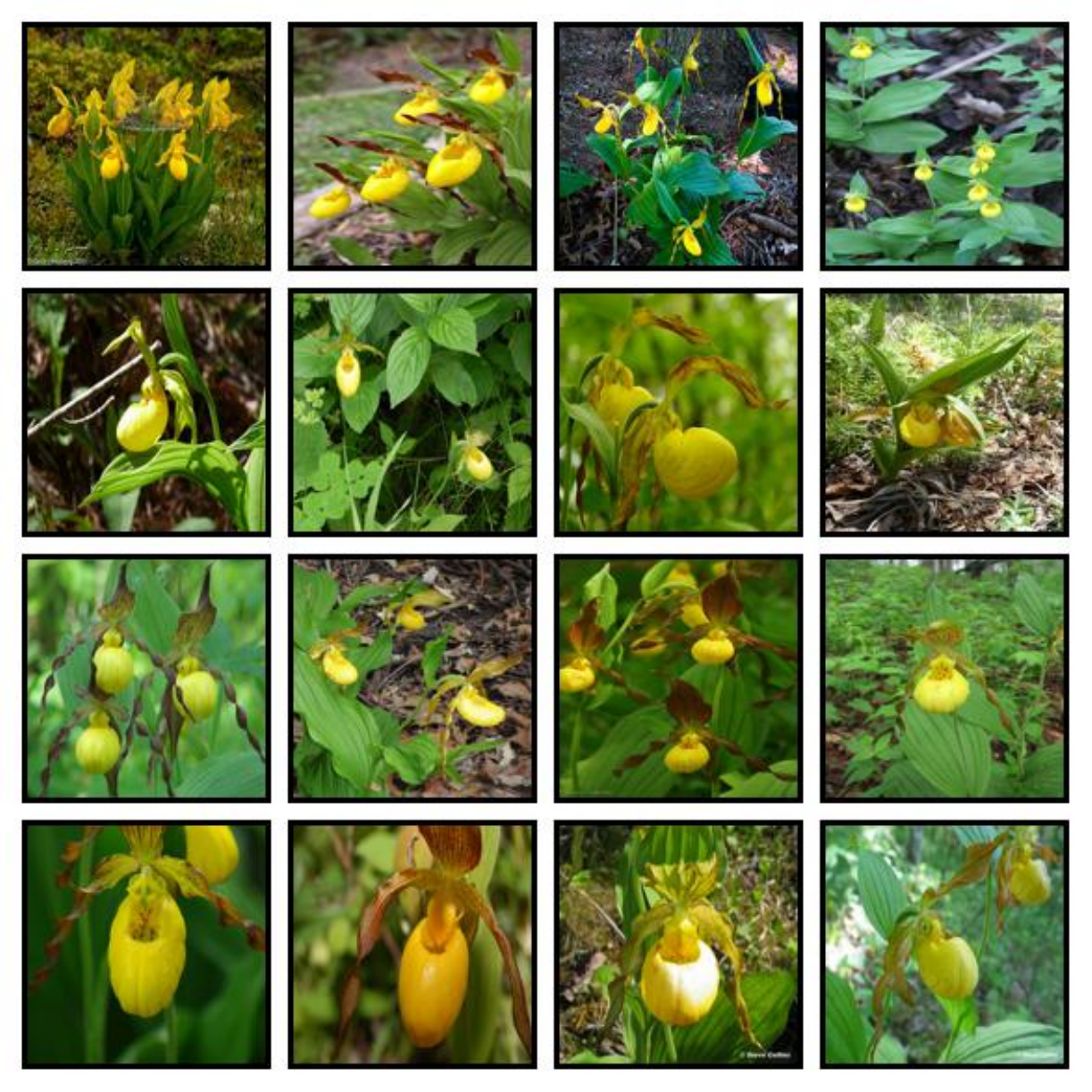} \\ \includegraphics[width=0.15\linewidth]{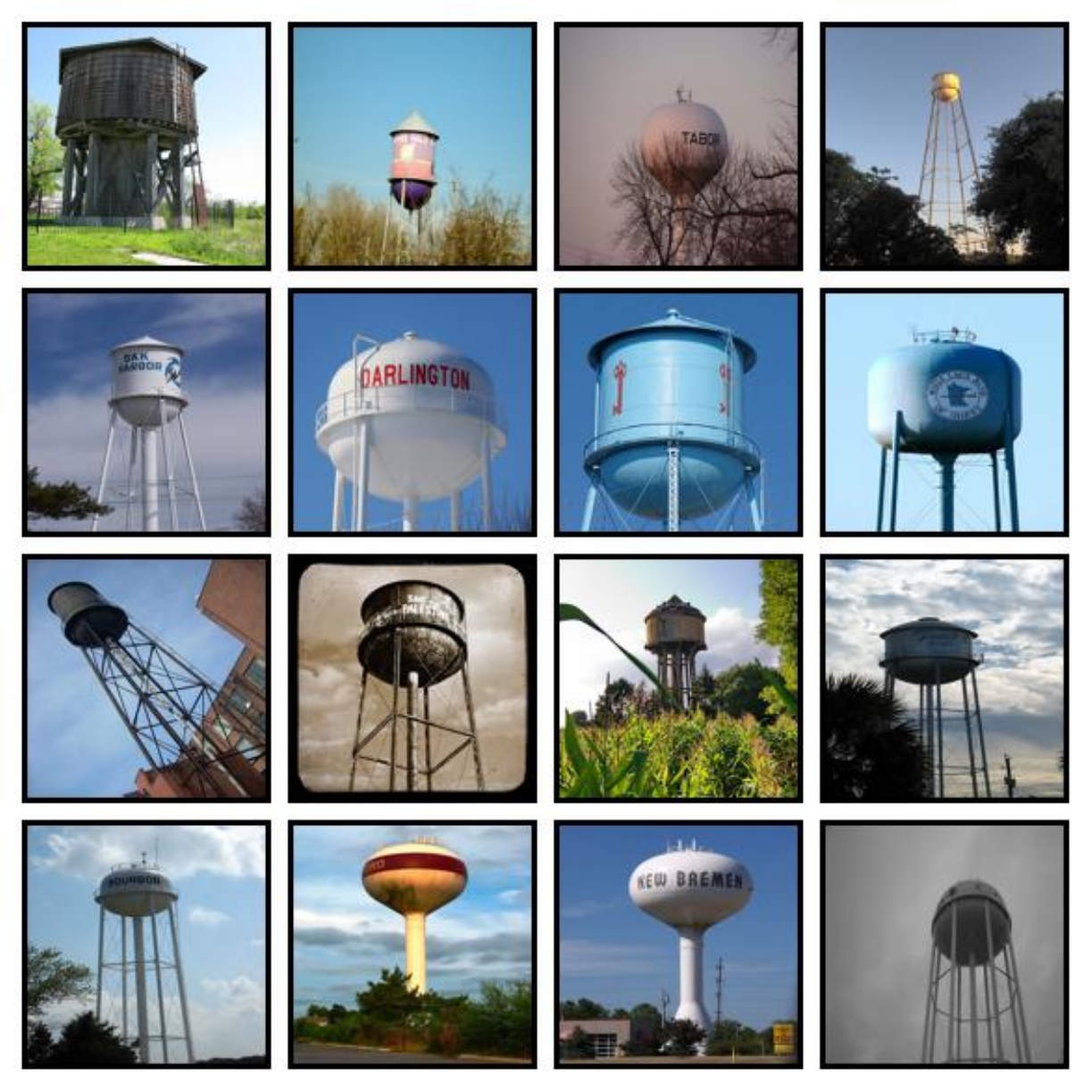} & \includegraphics[width=0.15\linewidth]{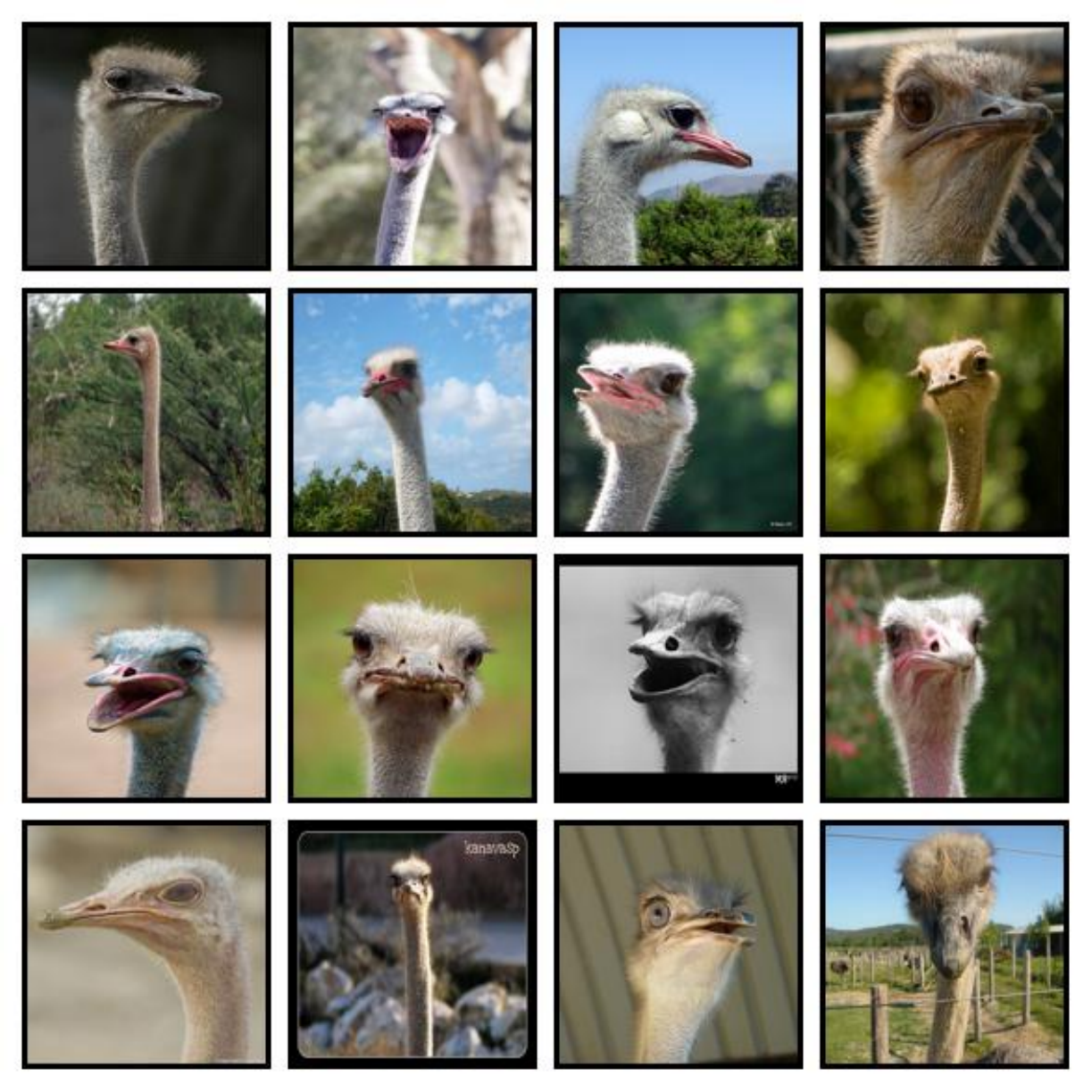} & \includegraphics[width=0.15\linewidth]{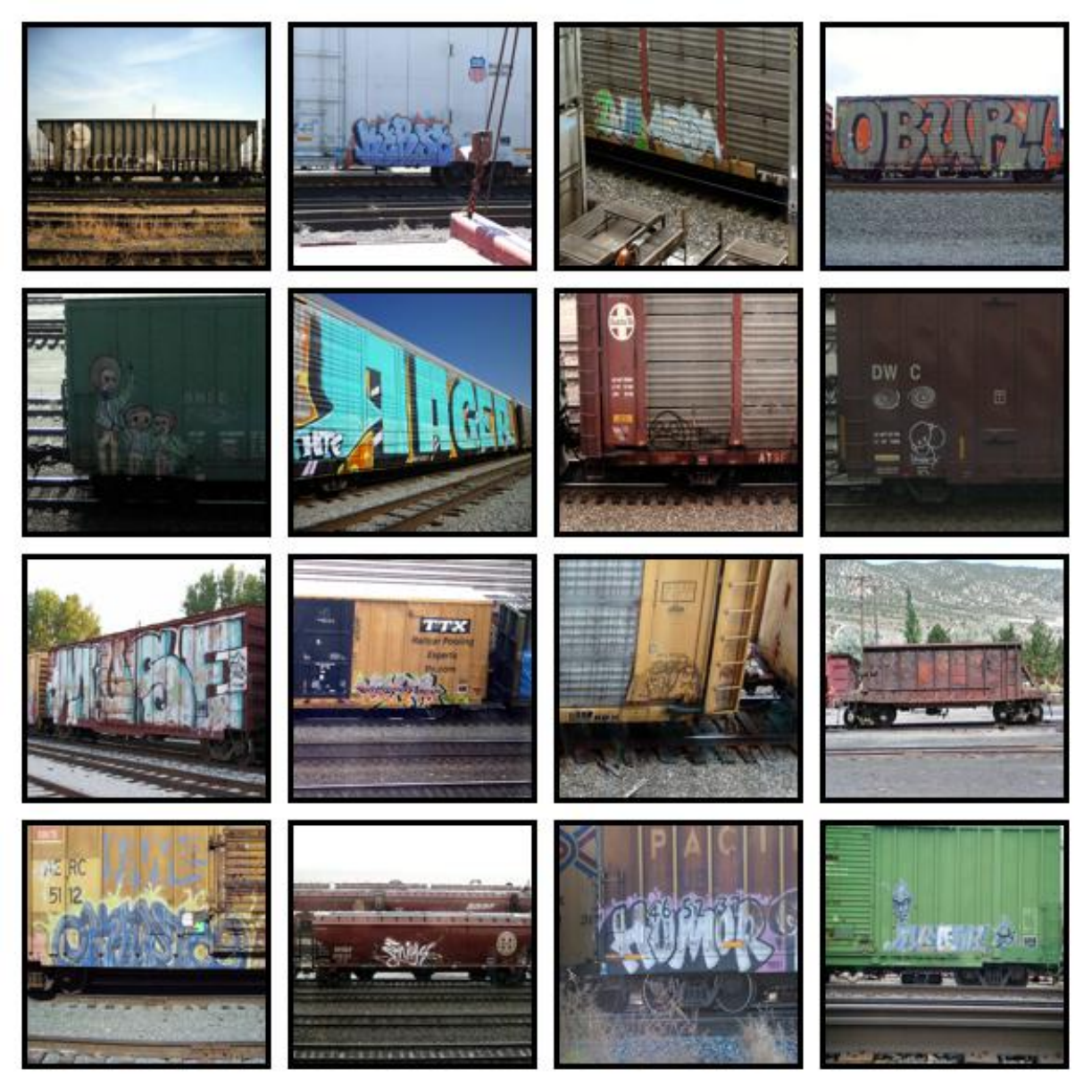} & \includegraphics[width=0.15\linewidth]{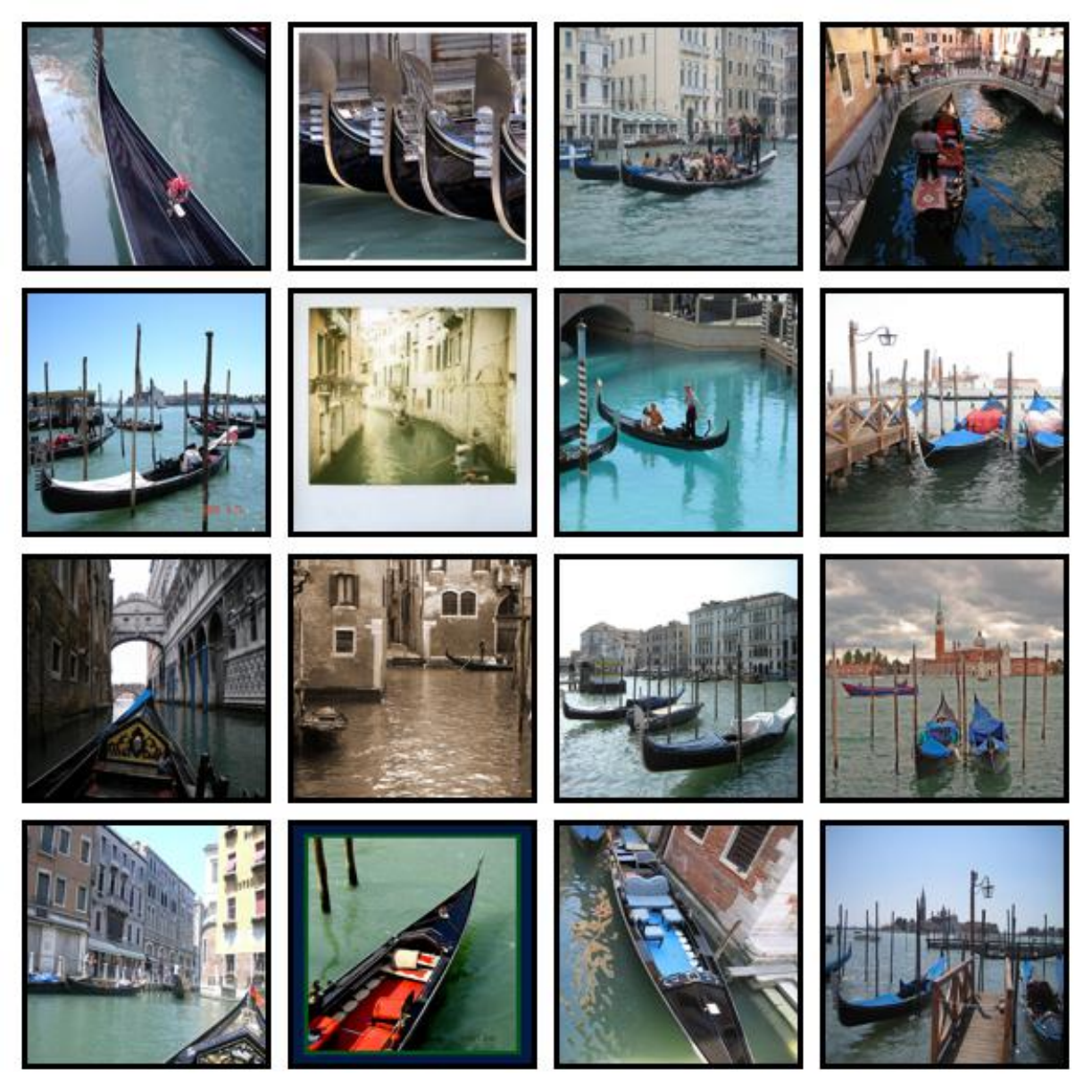} & \includegraphics[width=0.15\linewidth]{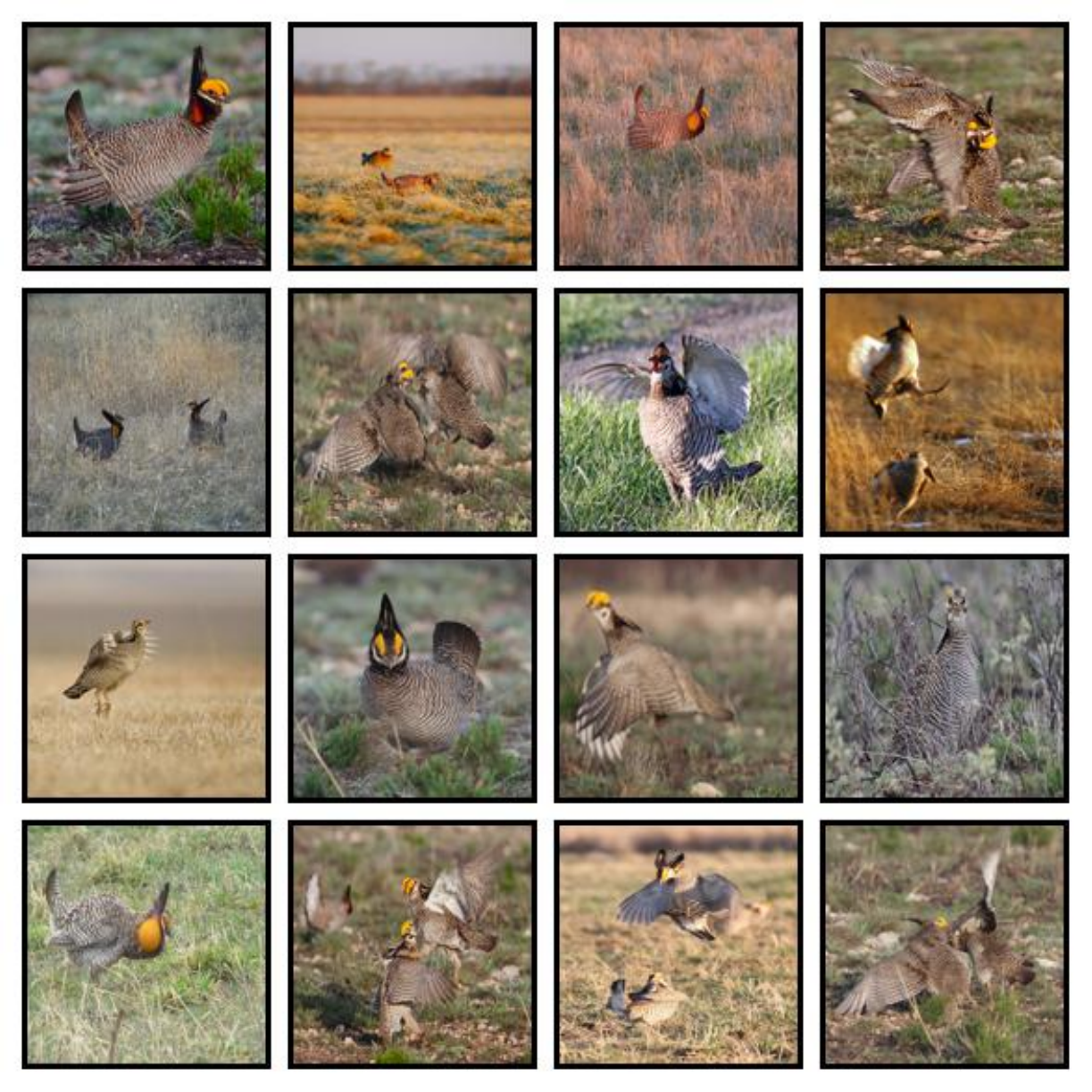} \\
  \includegraphics[width=0.15\linewidth]{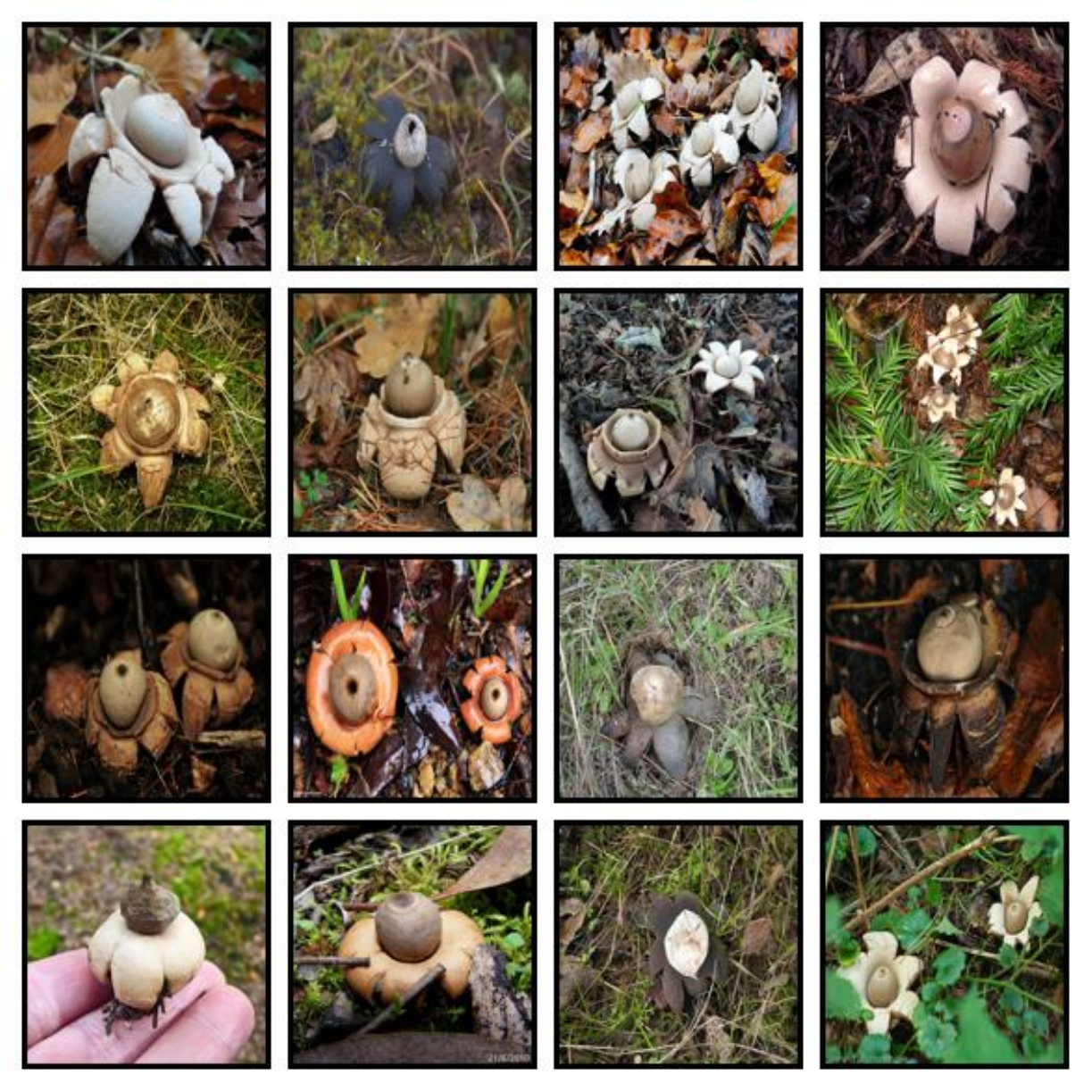} & \includegraphics[width=0.15\linewidth]{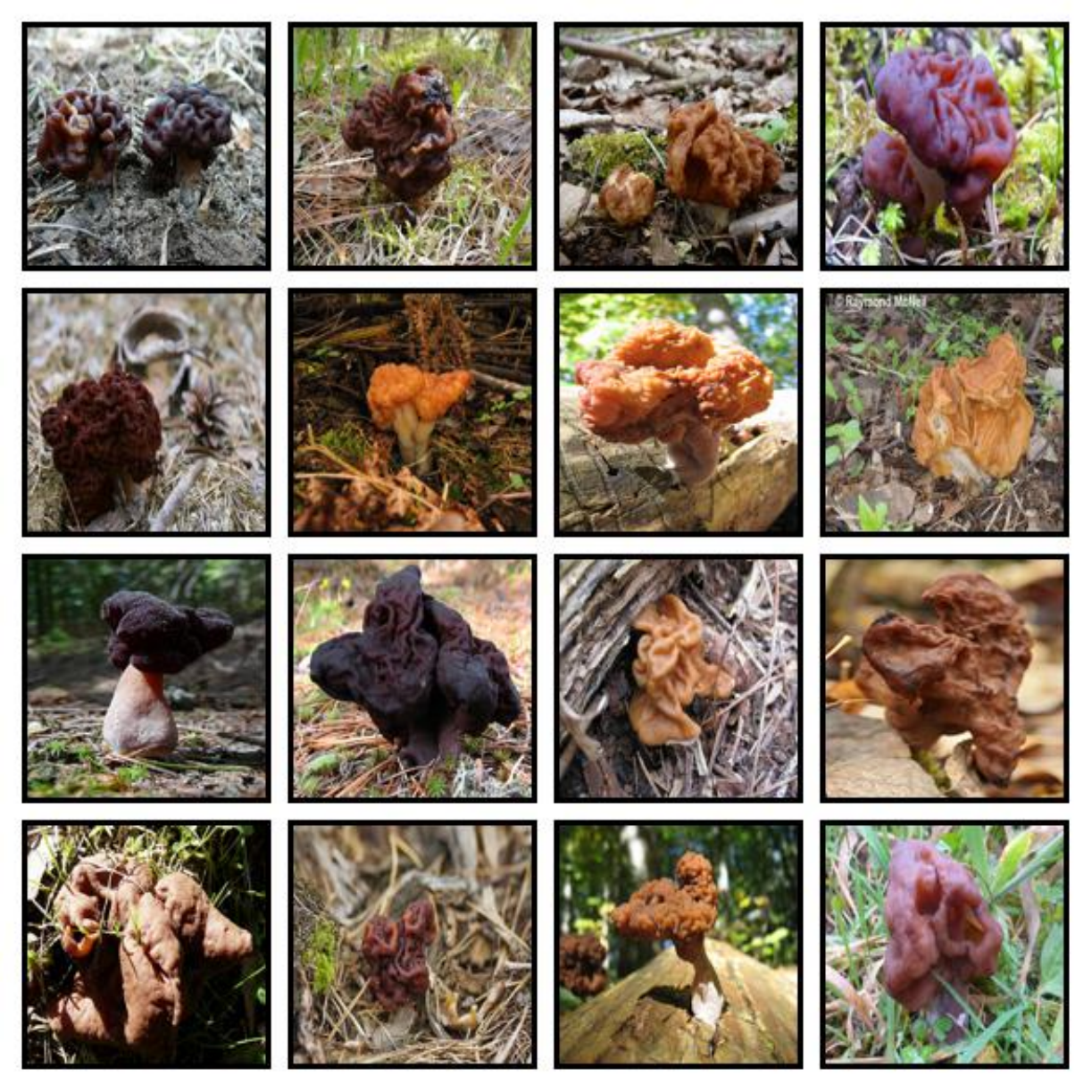} & \includegraphics[width=0.15\linewidth]{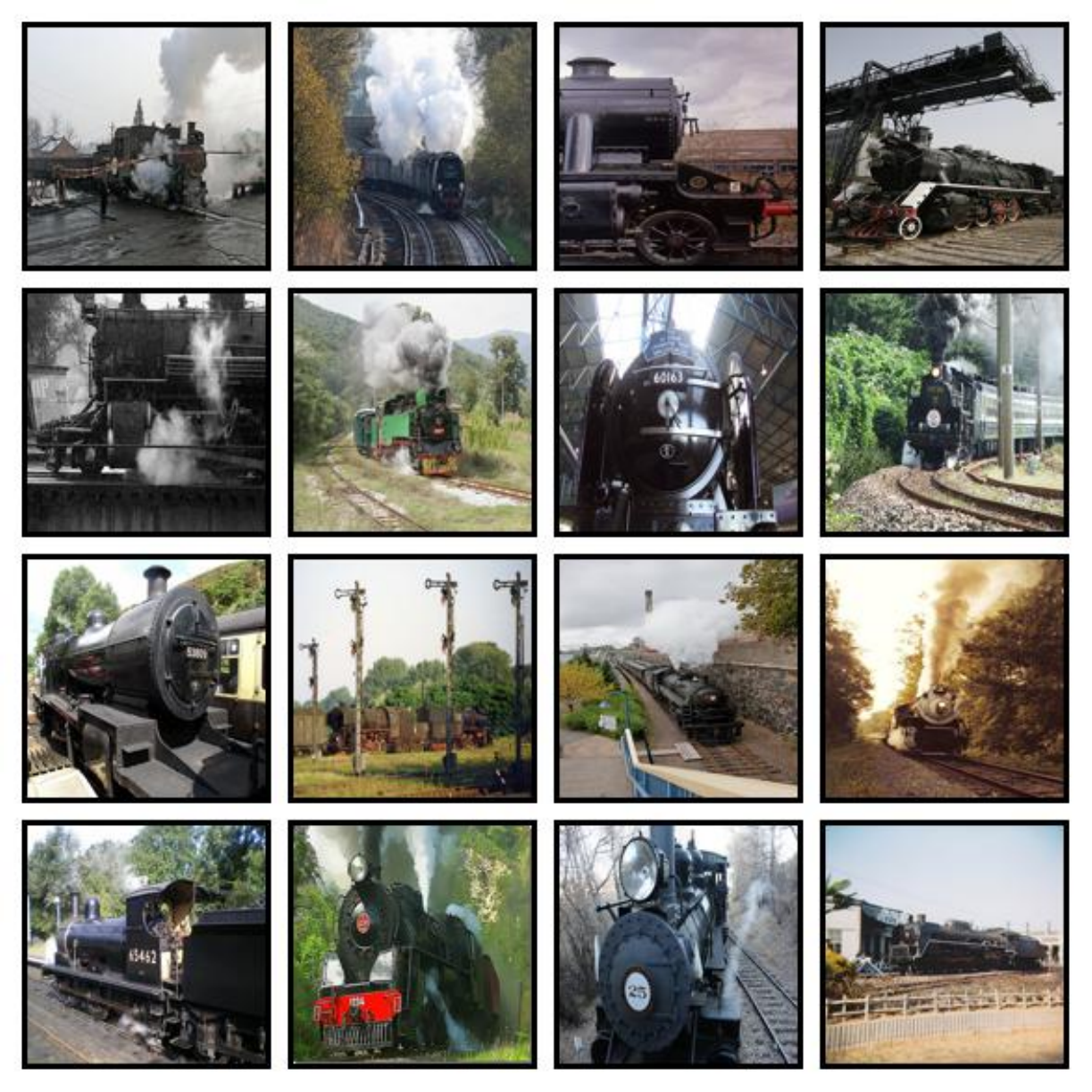} & \includegraphics[width=0.15\linewidth]{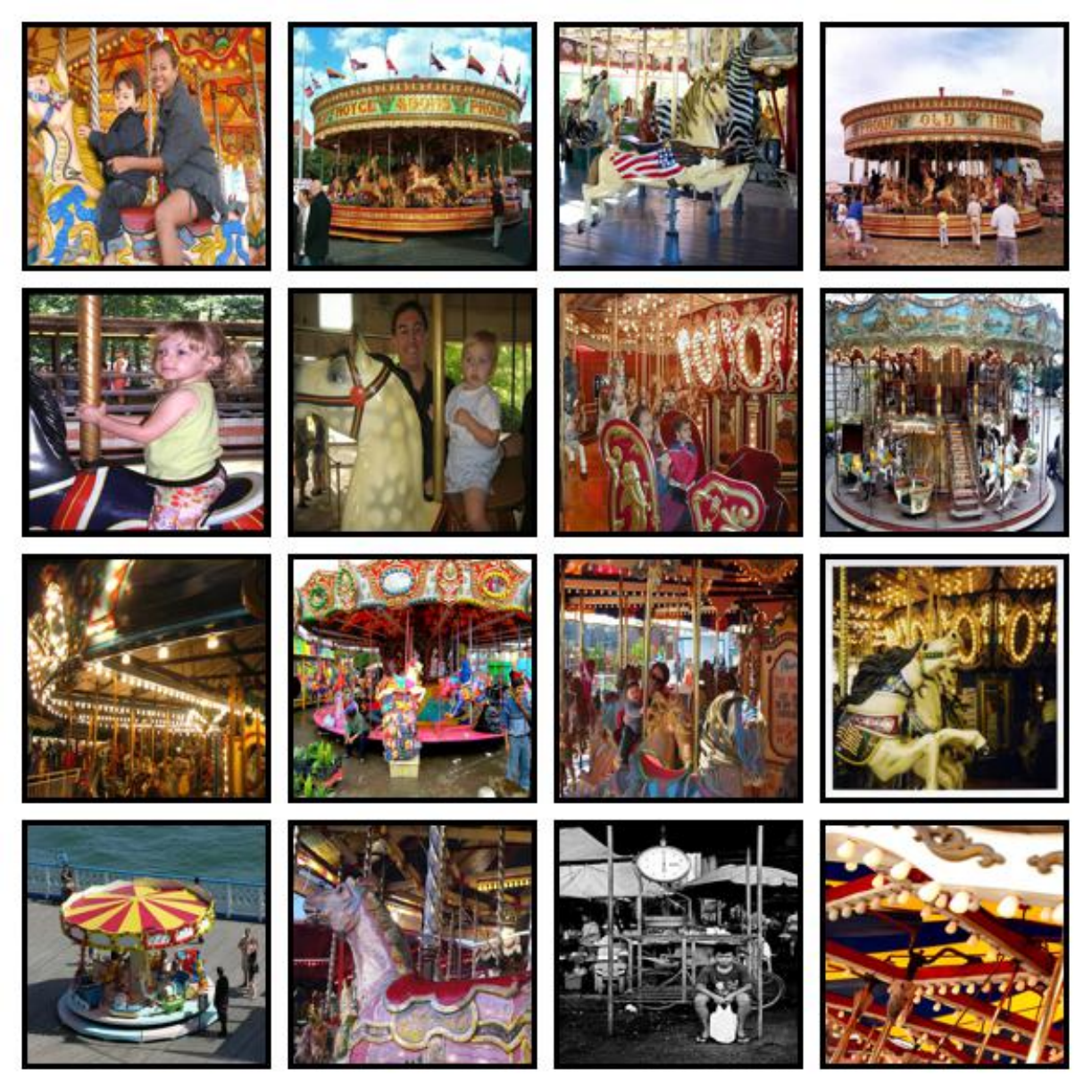} & \includegraphics[width=0.15\linewidth]{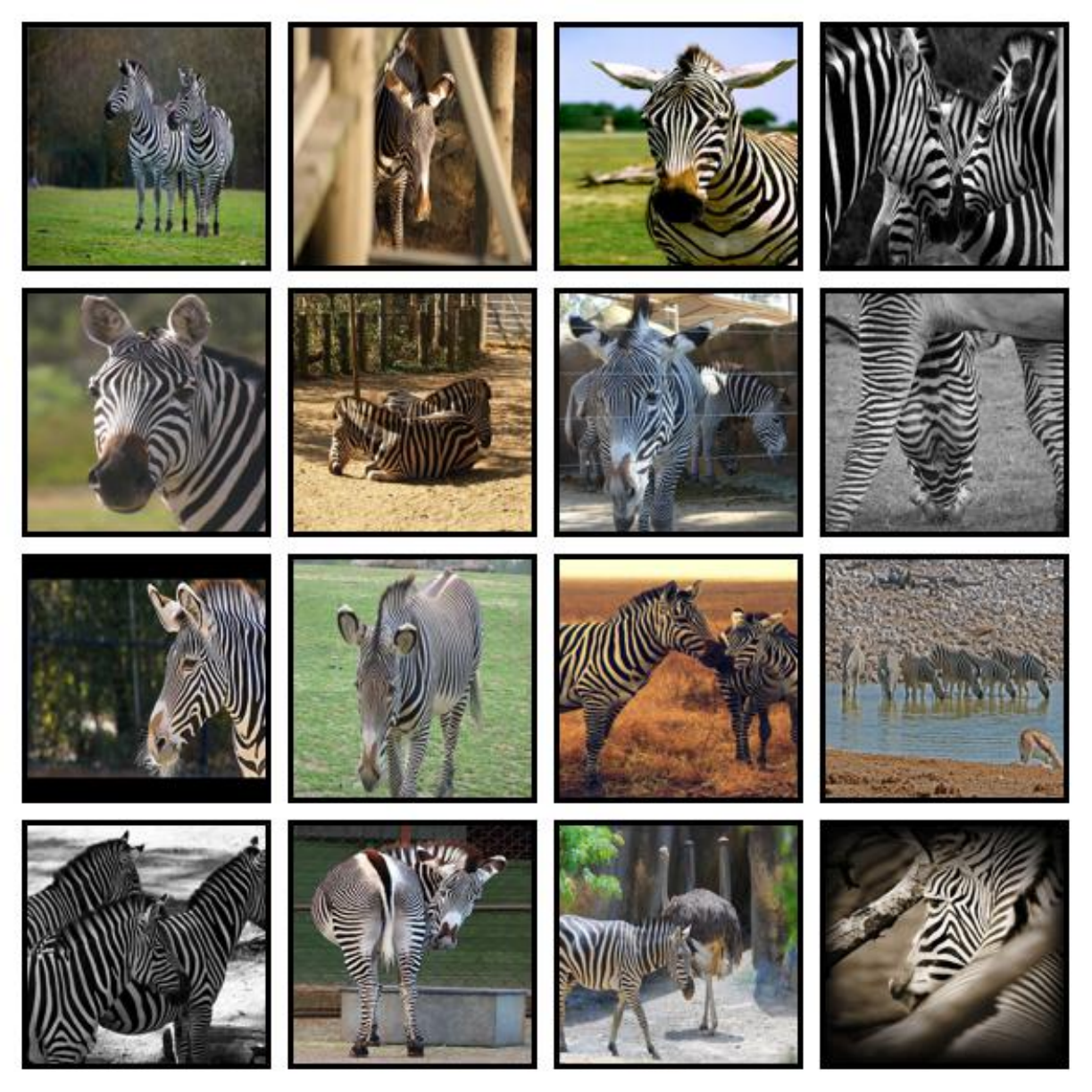} \\
  \includegraphics[width=0.15\linewidth]{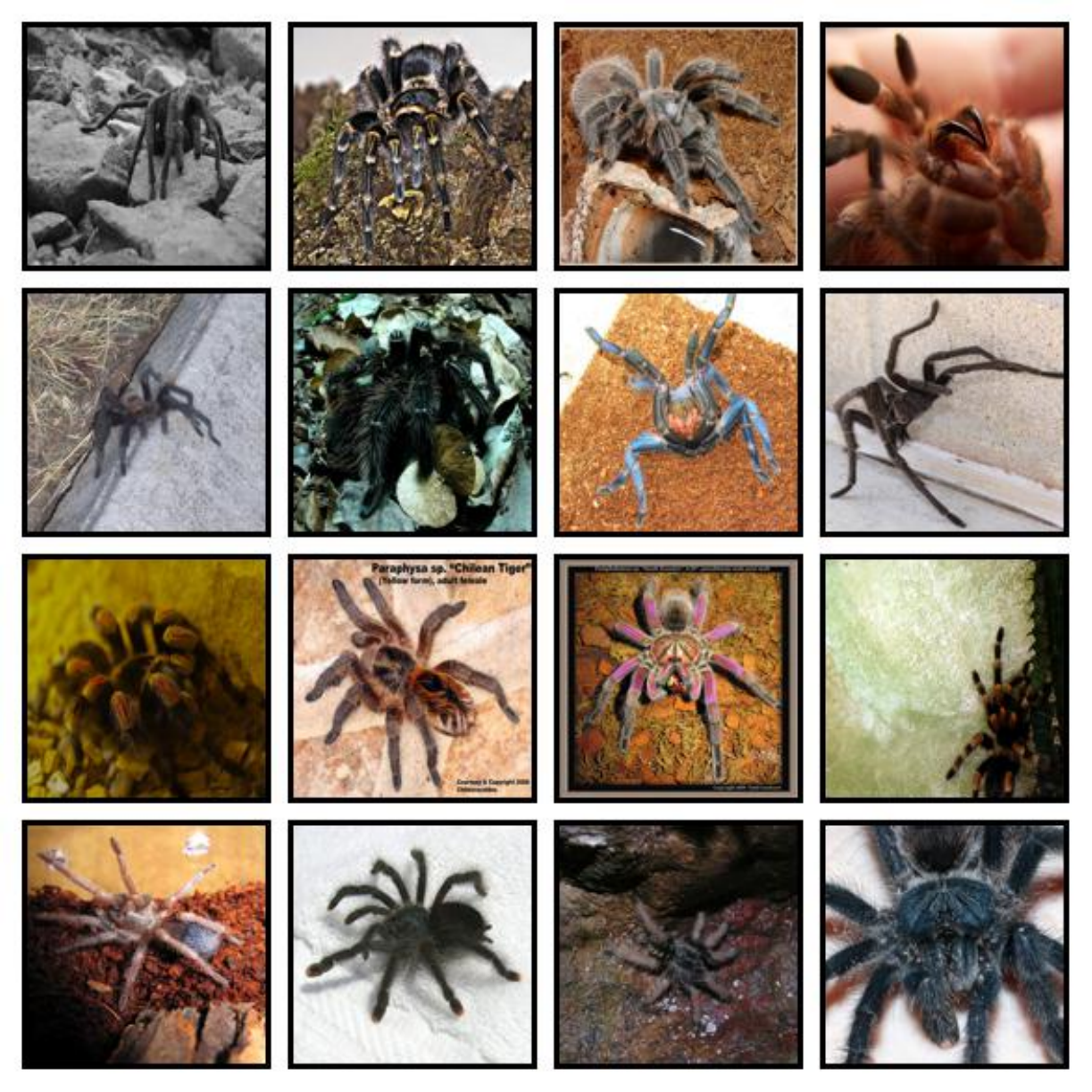} & \includegraphics[width=0.15\linewidth]{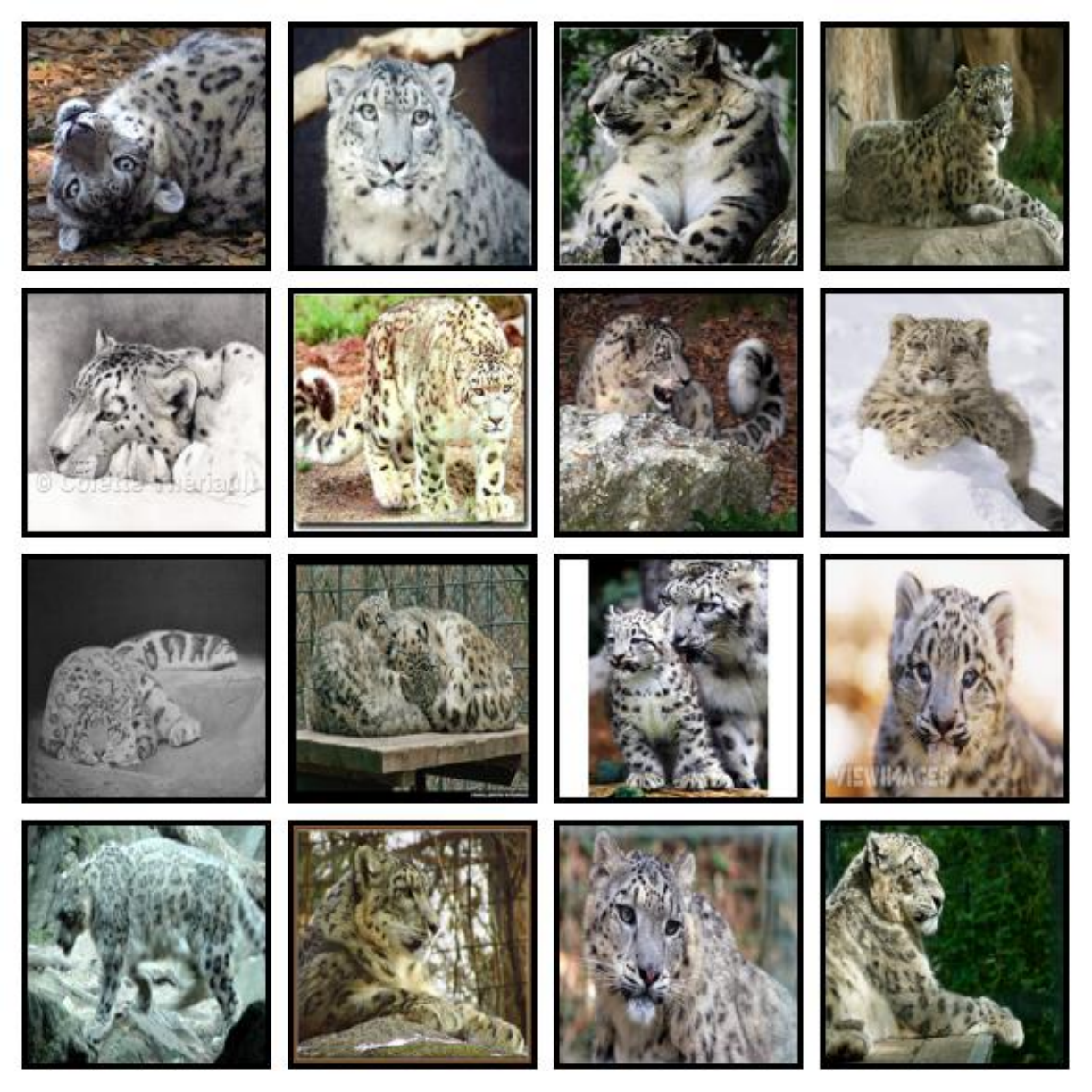} & \includegraphics[width=0.15\linewidth]{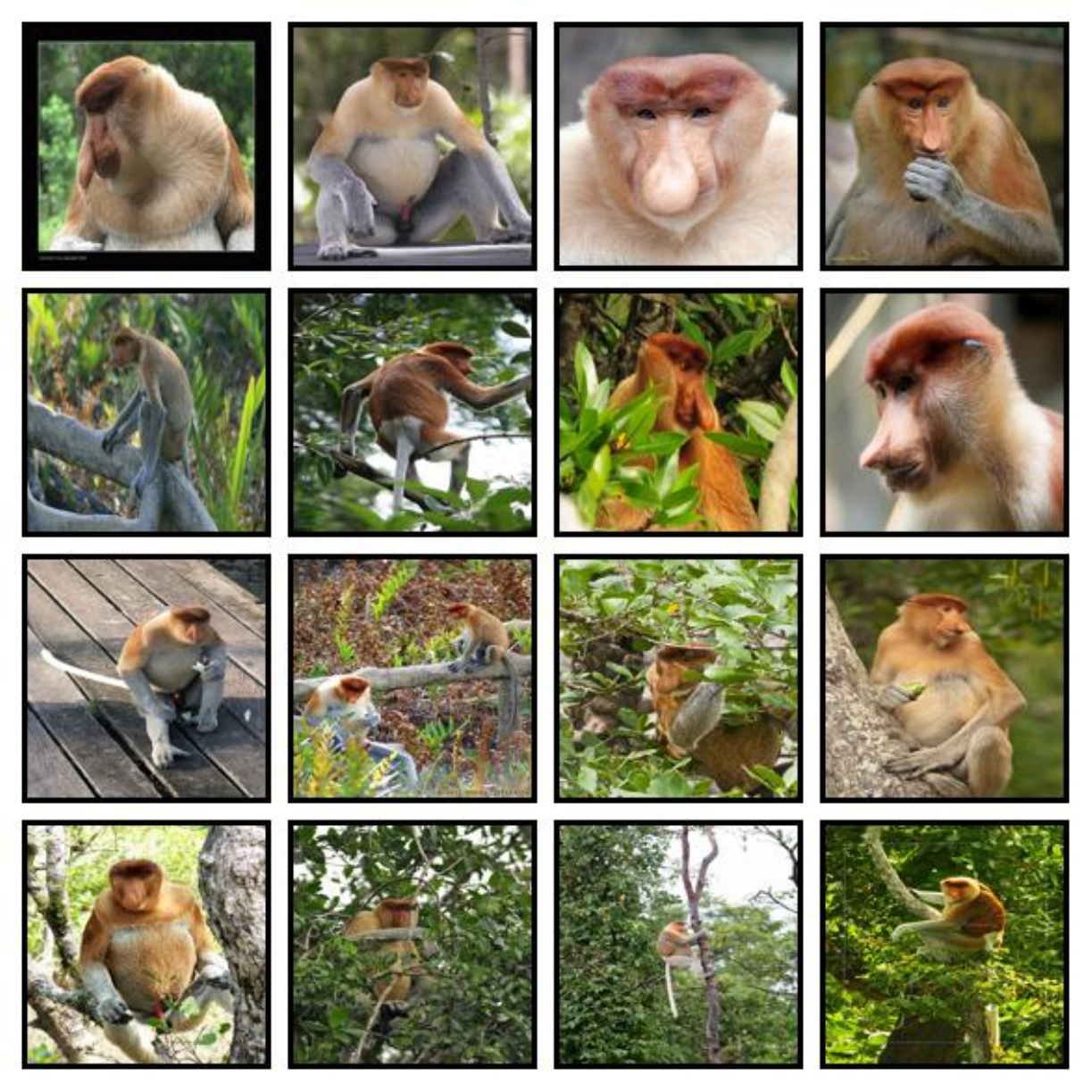} & \includegraphics[width=0.15\linewidth]{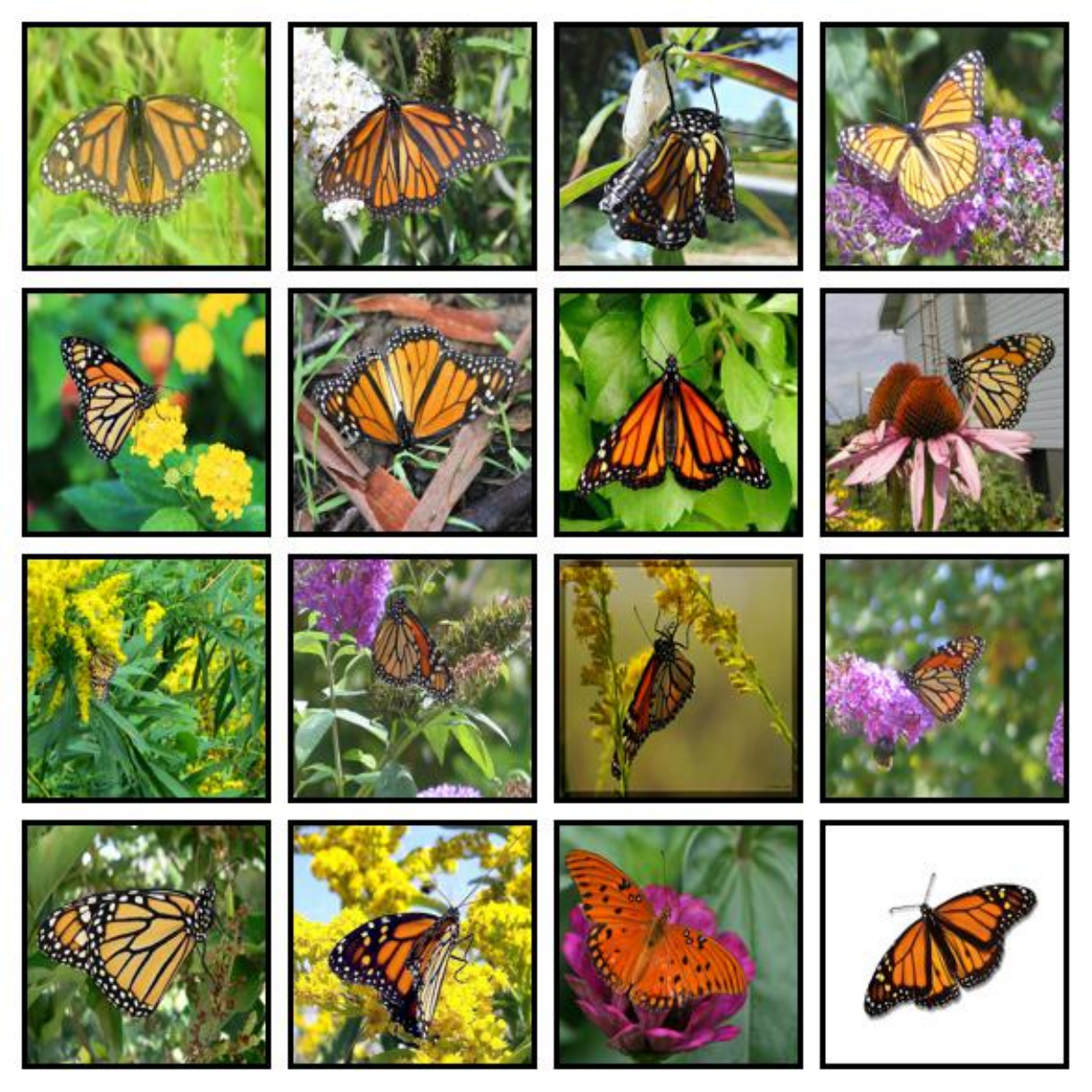} & \includegraphics[width=0.15\linewidth]{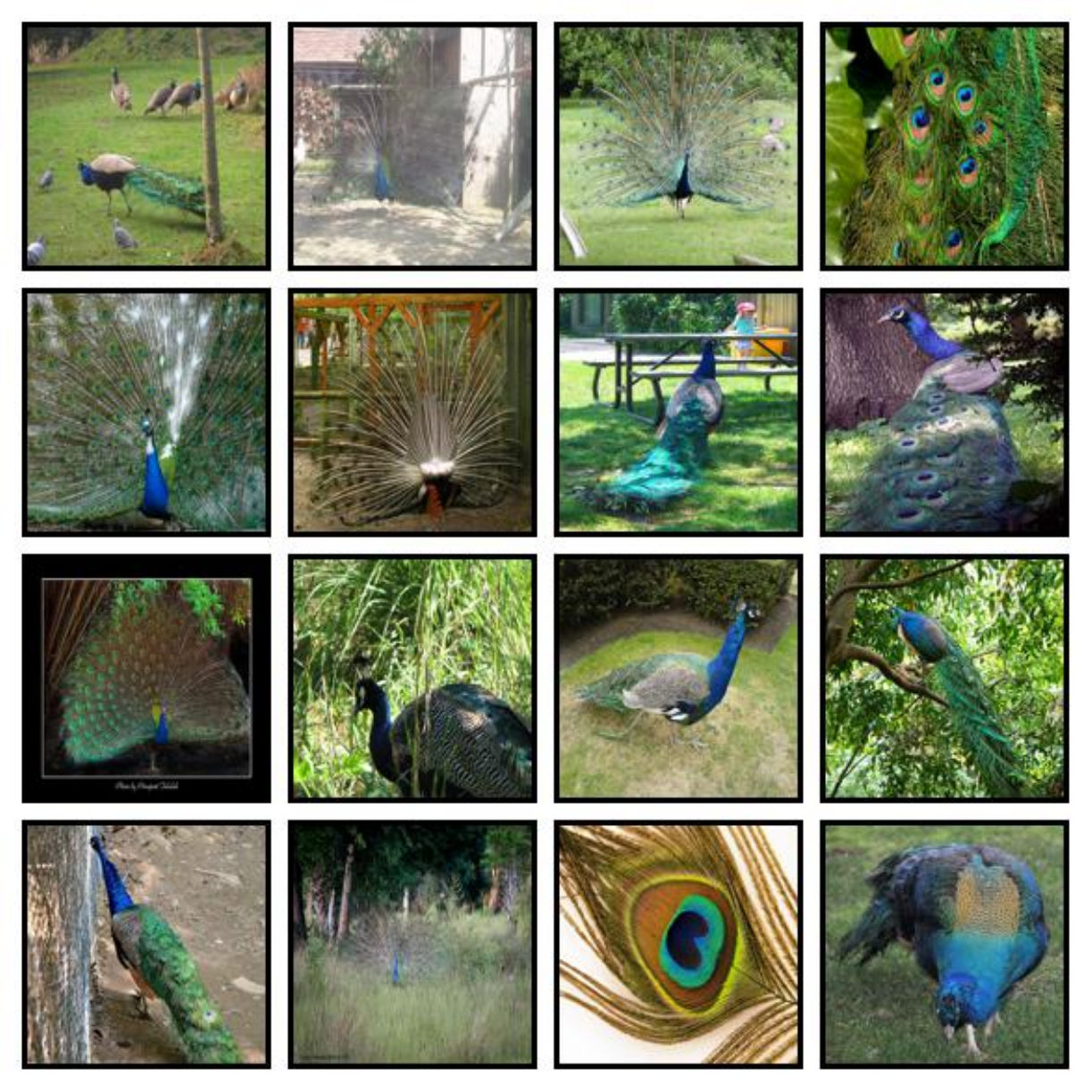} \\
  \includegraphics[width=0.15\linewidth]{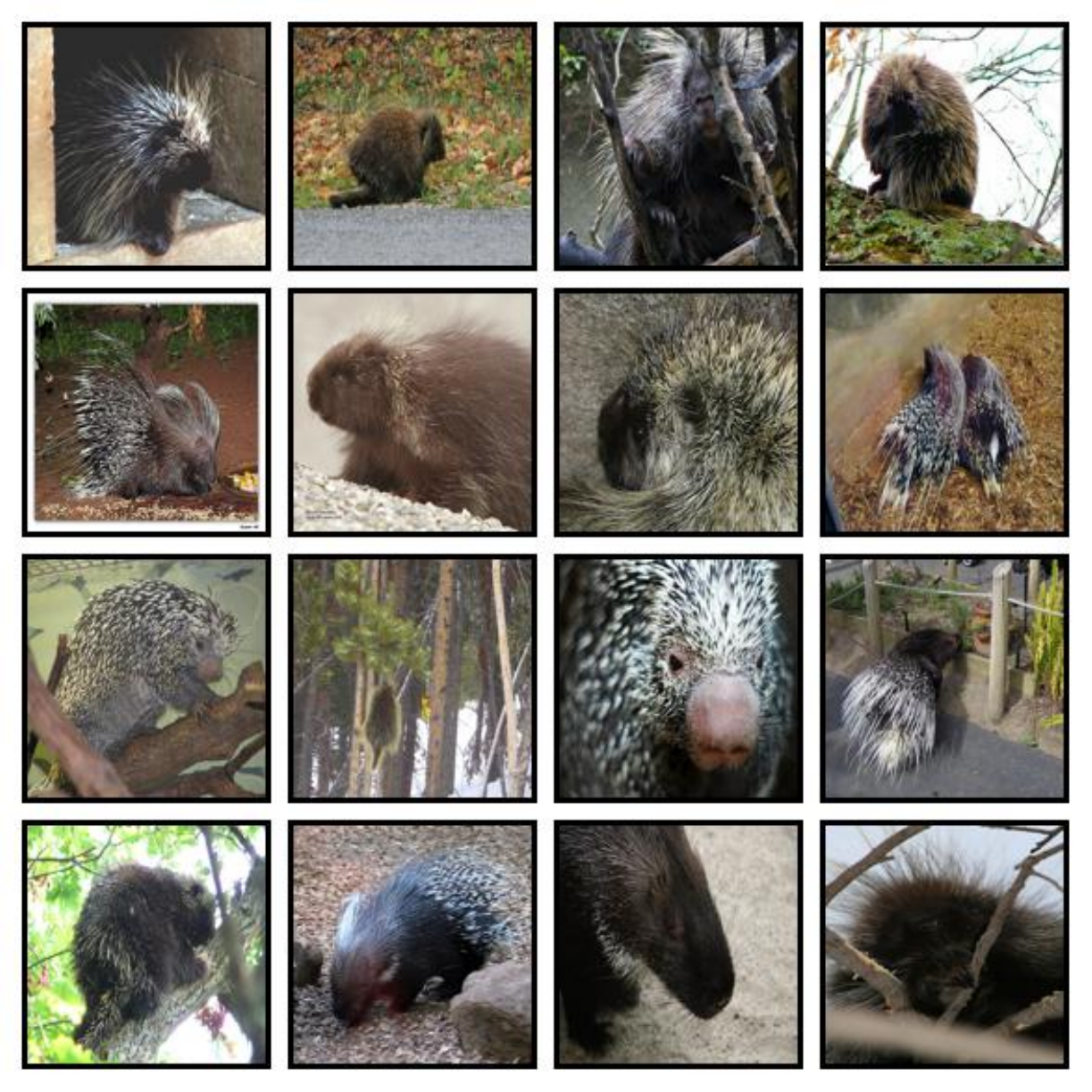} & \includegraphics[width=0.15\linewidth]{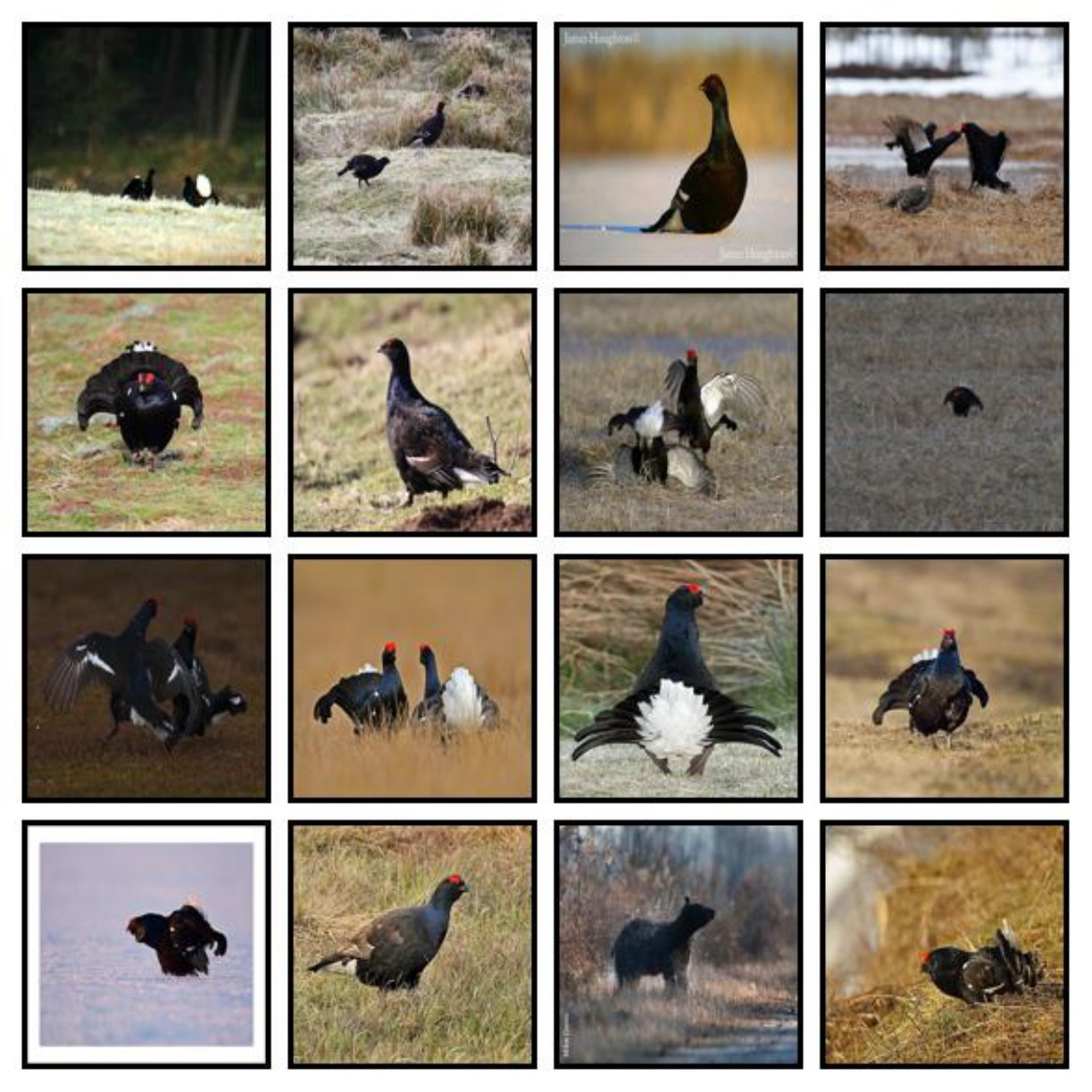} & \includegraphics[width=0.15\linewidth]{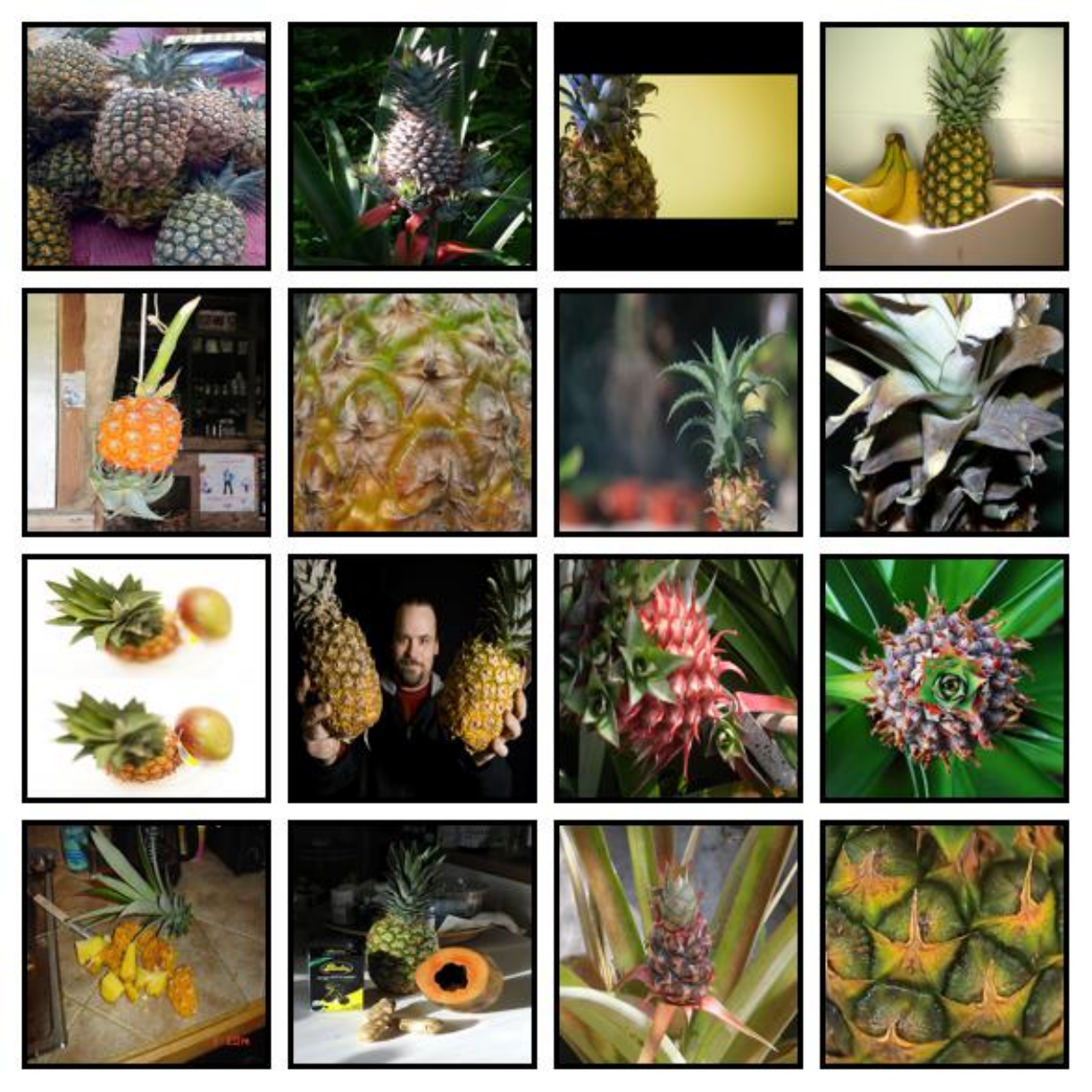} & \includegraphics[width=0.15\linewidth]{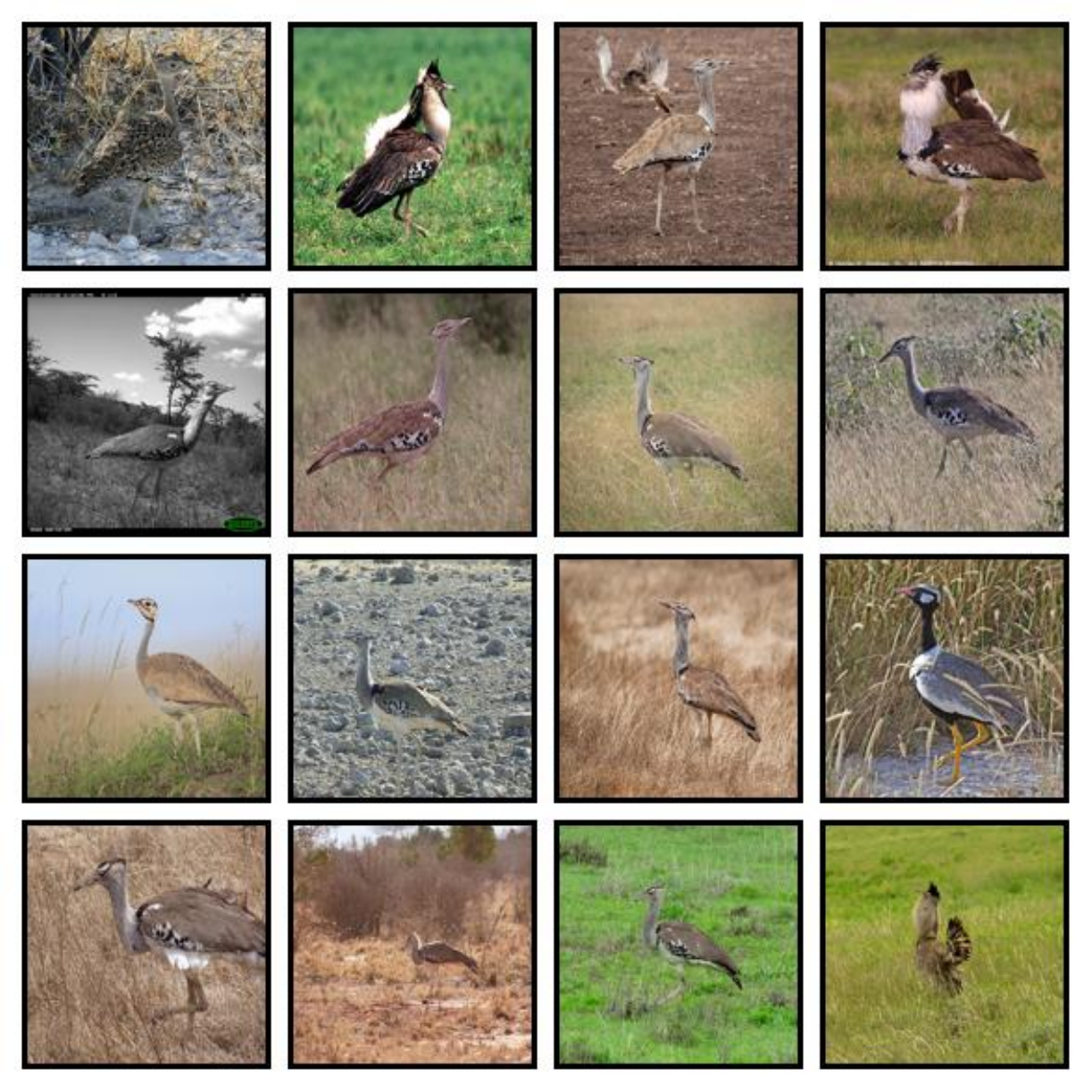} & \includegraphics[width=0.15\linewidth]{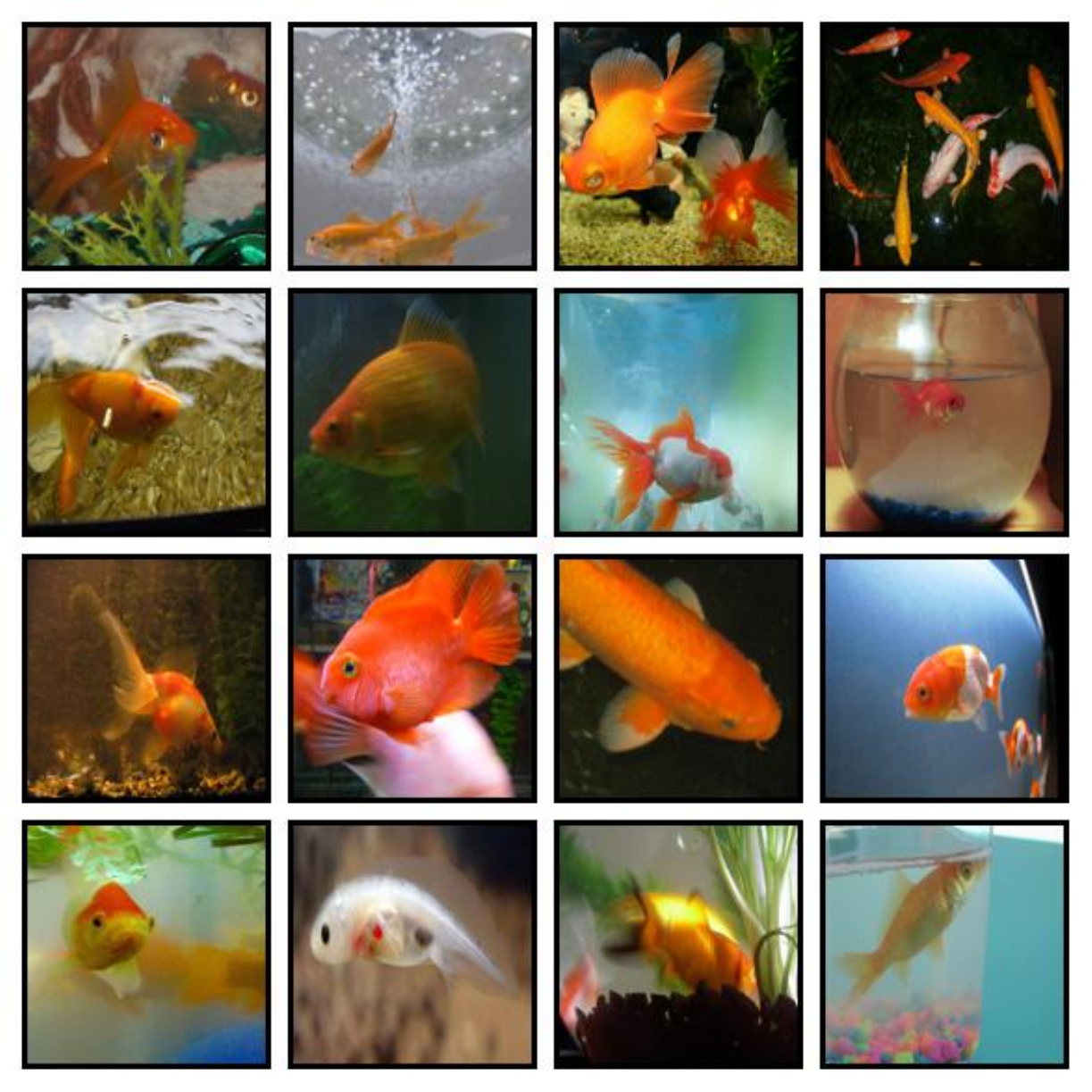} \\
  \includegraphics[width=0.15\linewidth]{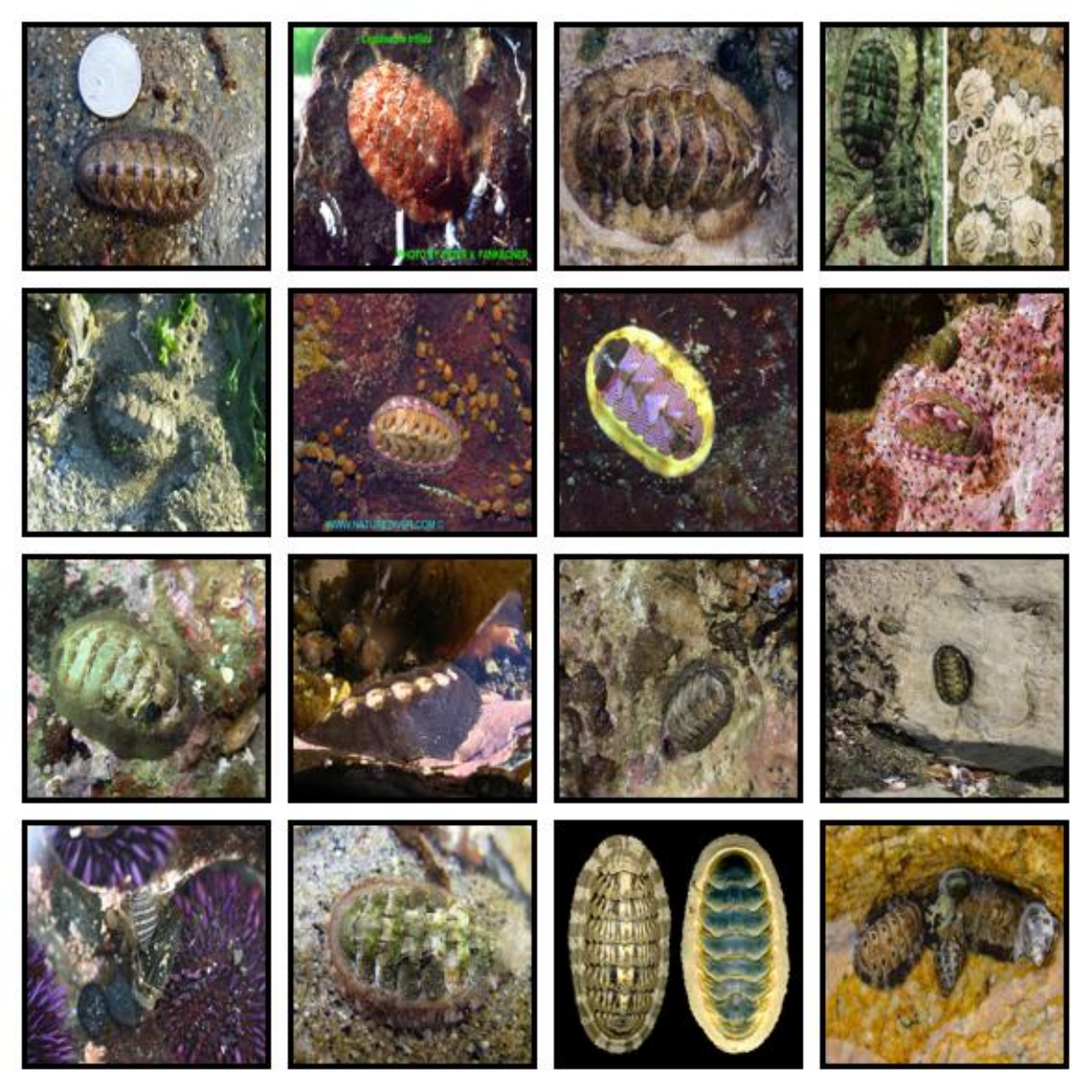} & \includegraphics[width=0.15\linewidth]{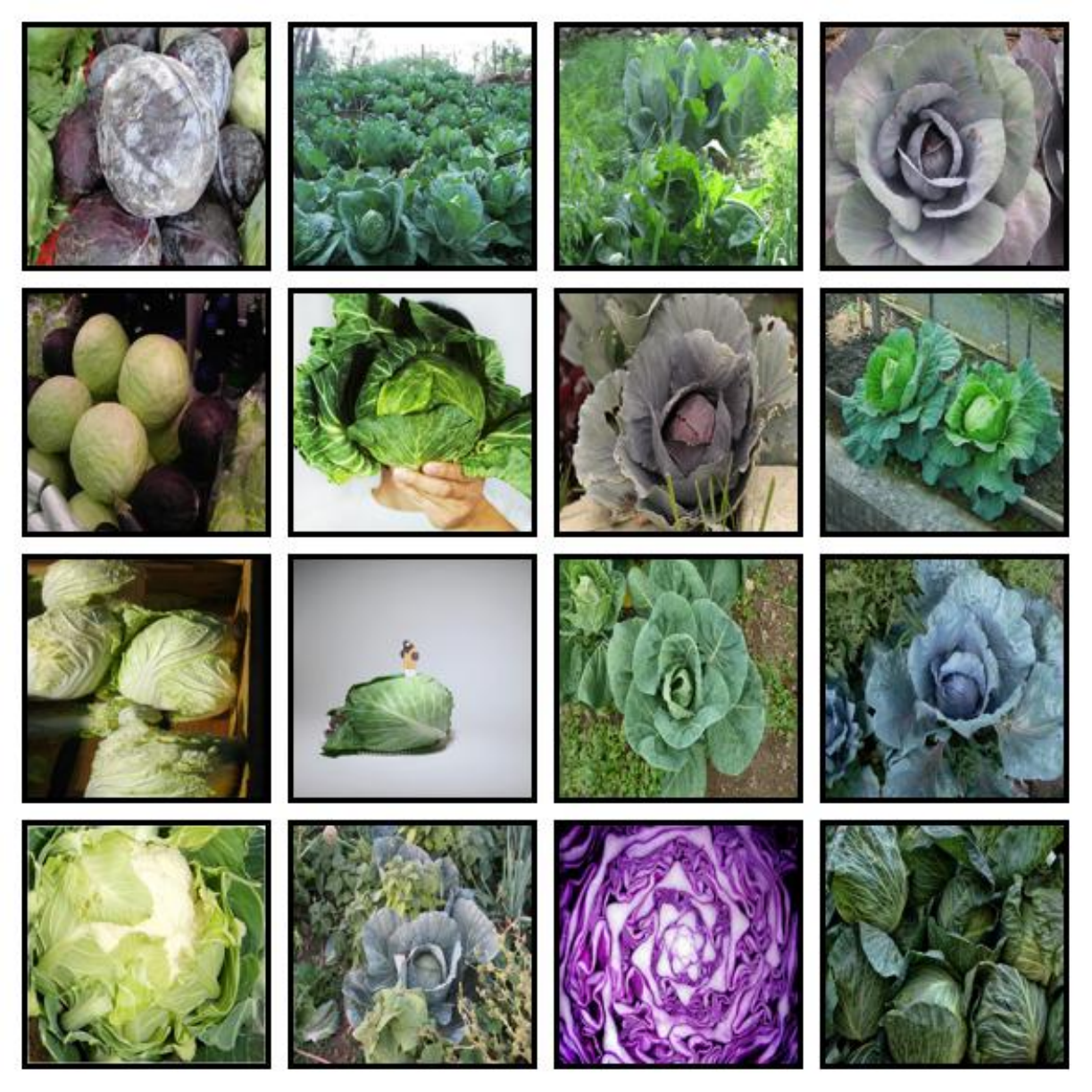} & \includegraphics[width=0.15\linewidth]{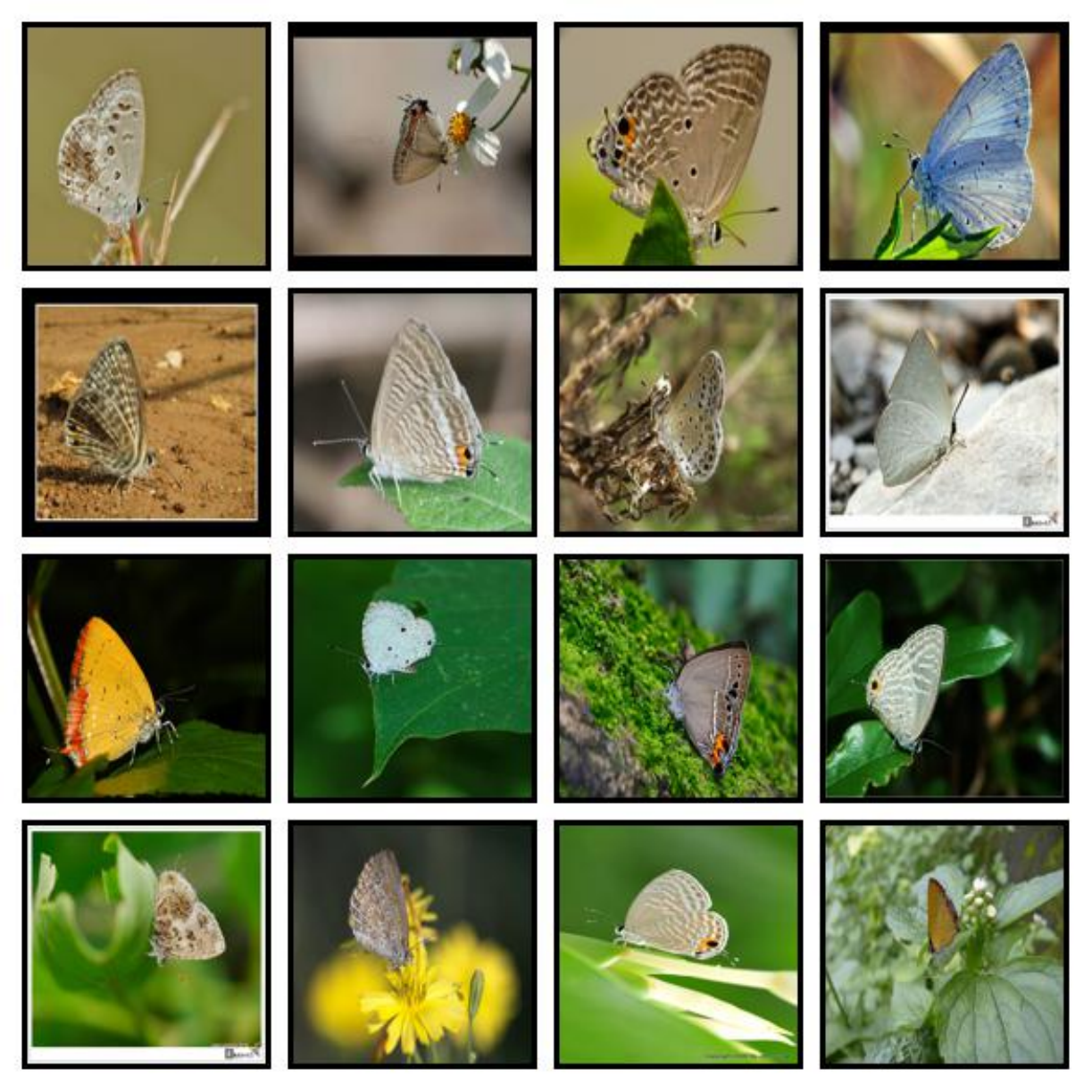} & \includegraphics[width=0.15\linewidth]{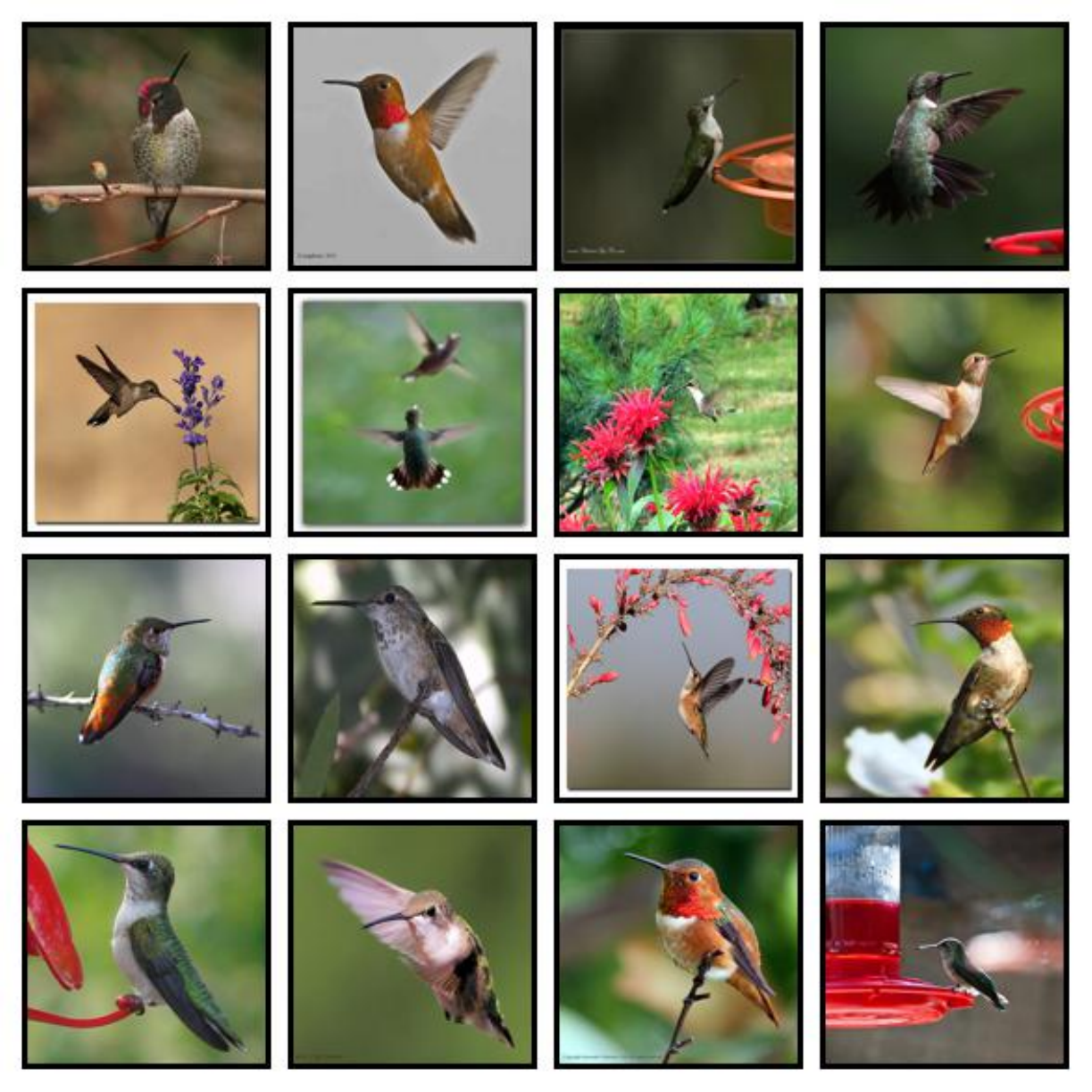} & \includegraphics[width=0.15\linewidth]{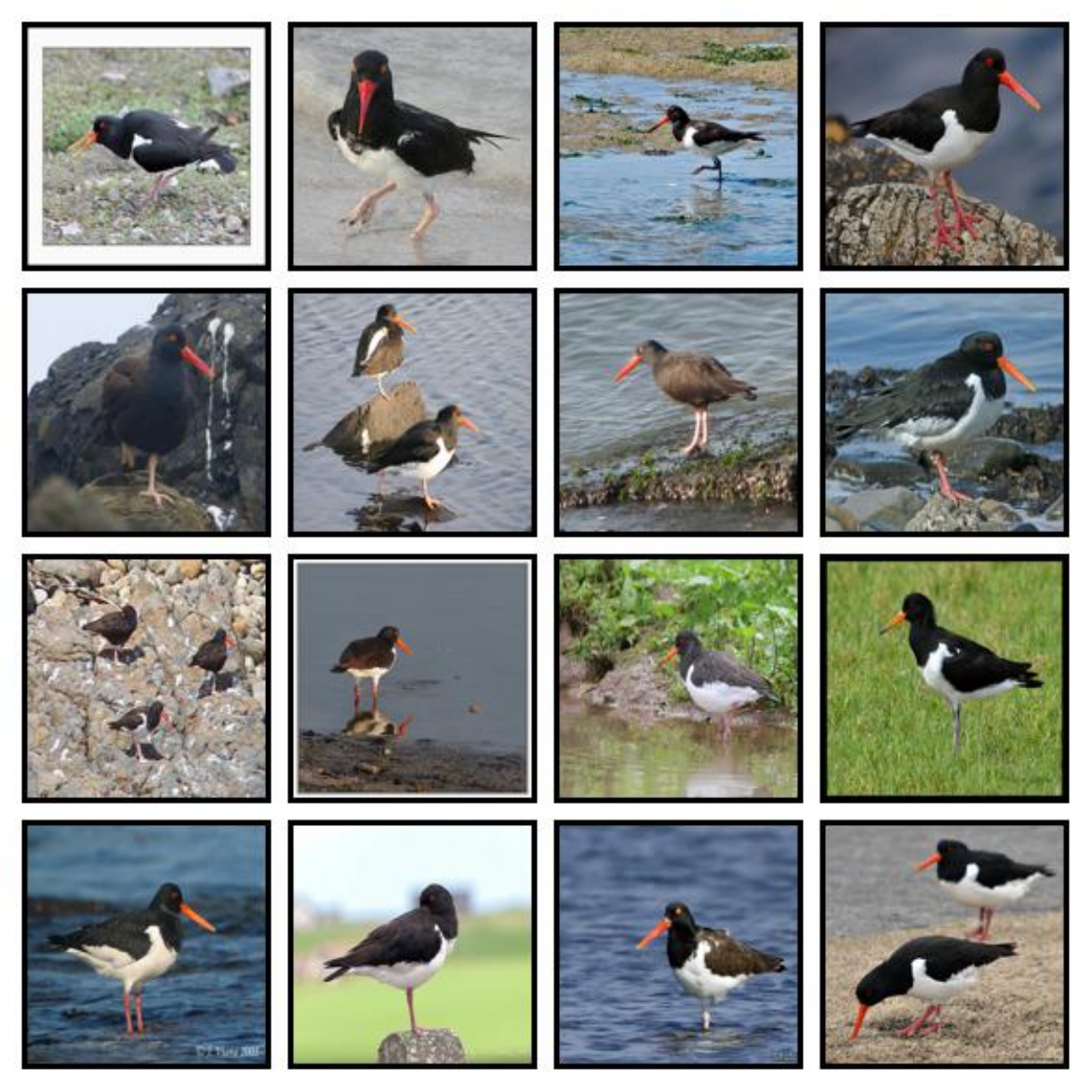} 
\end{tabular}
\caption{\textbf{\textit{High} accuracy classes predicted by \textit{Self-Classifier} on ImageNet validation set (unseen during training)}. Classes are sorted by accuracy, and images are sampled \textbf{randomly} from each predicted class. Note that the predicted classes capture a large variety of different backgrounds and viewpoints. This provides further evidence that \textit{Self-Classifier} learns semantically meaningful classes without any labels.}
\label{fig:all_classes}
\end{figure}

\newpage

\begin{figure}[H]
\centering
\begin{tabular}{ccccccc}
  \includegraphics[width=0.15\linewidth]{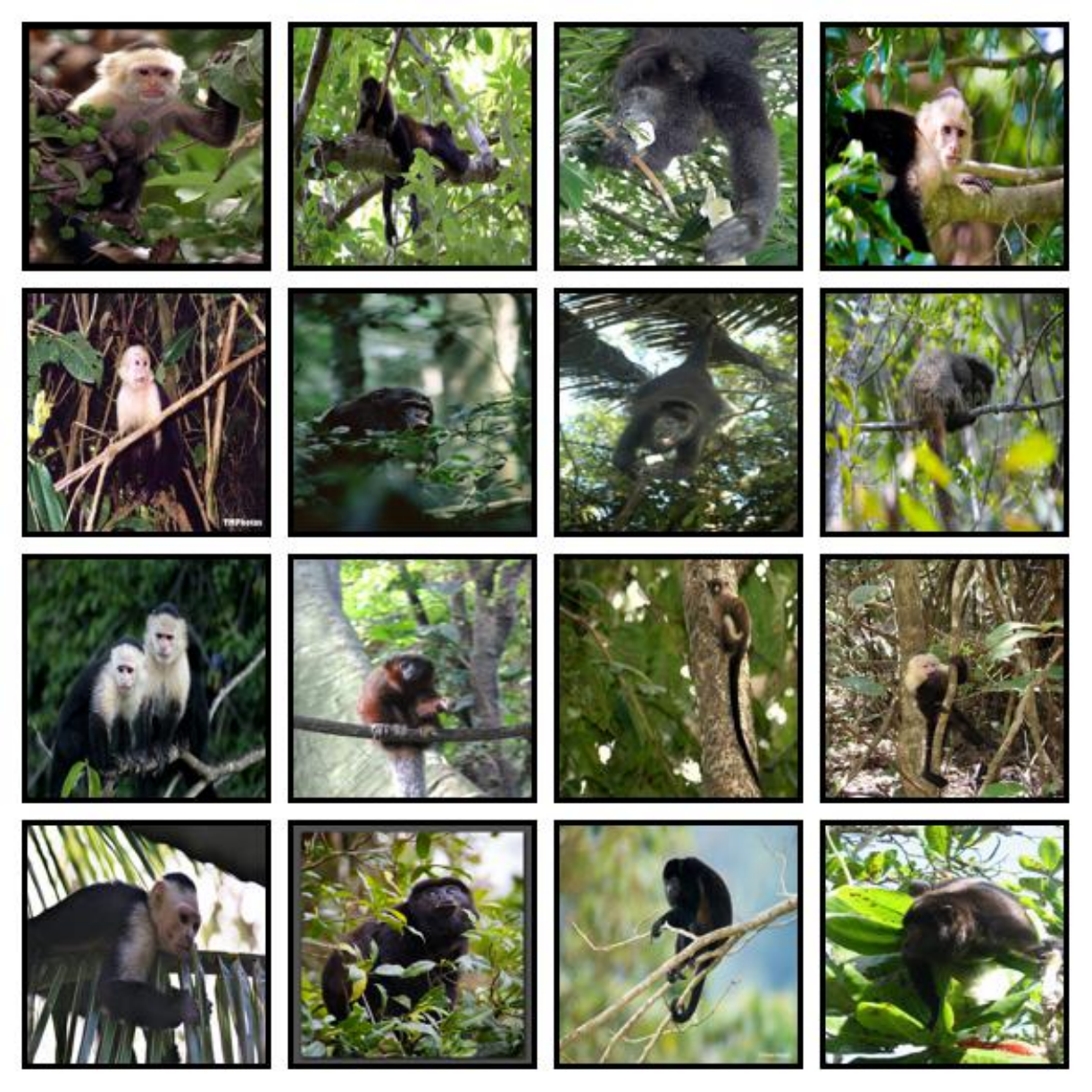} & \includegraphics[width=0.15\linewidth]{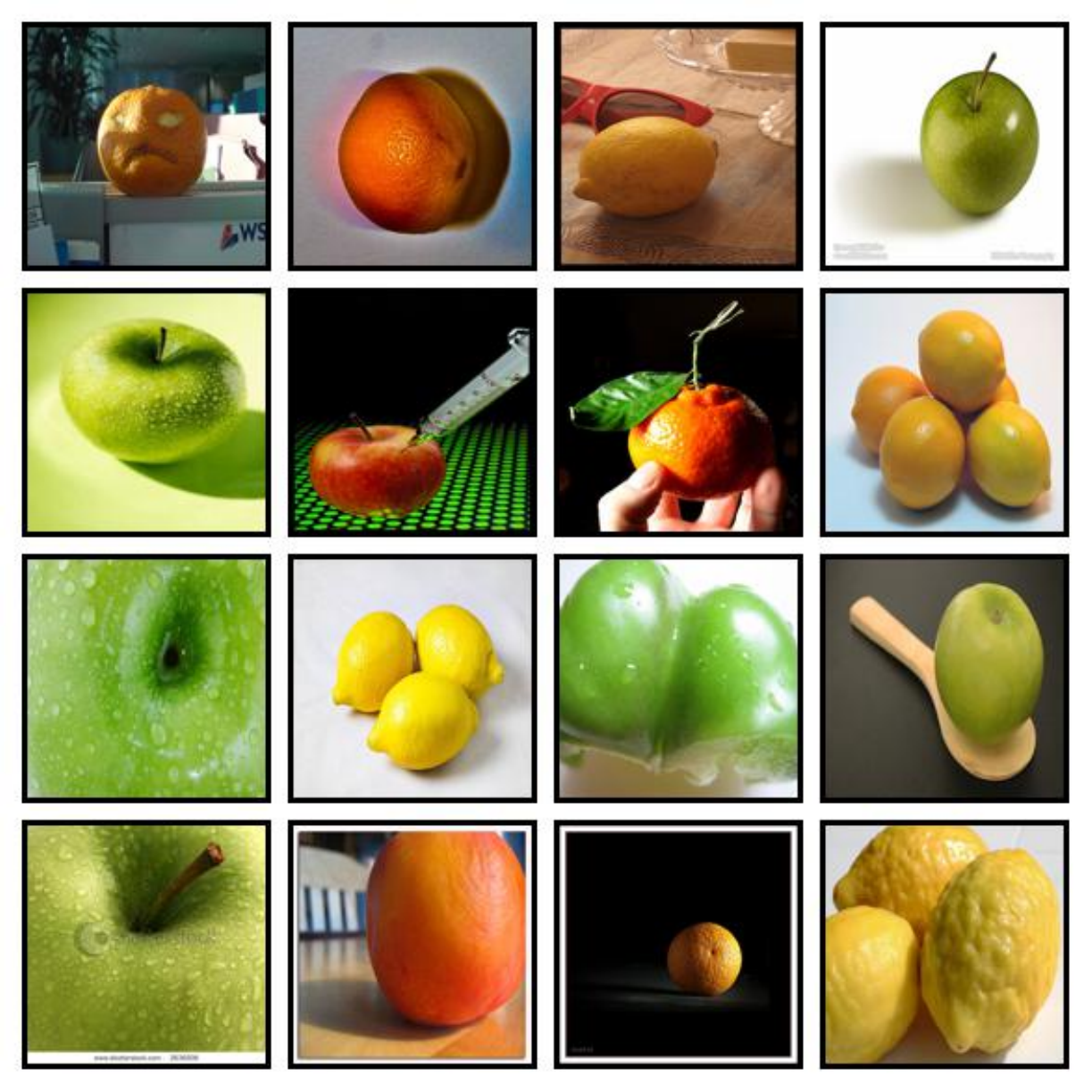} & \includegraphics[width=0.15\linewidth]{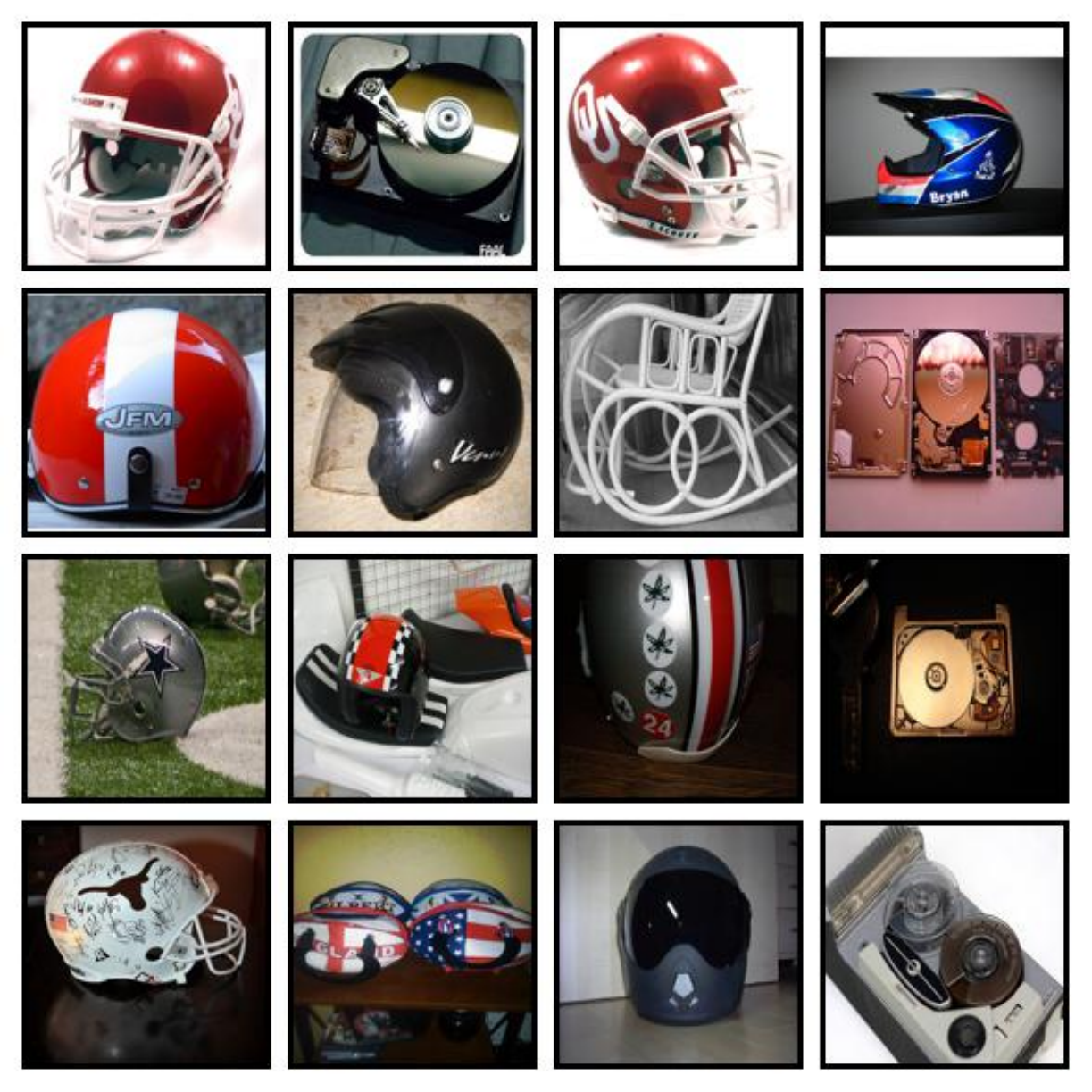} & \includegraphics[width=0.15\linewidth]{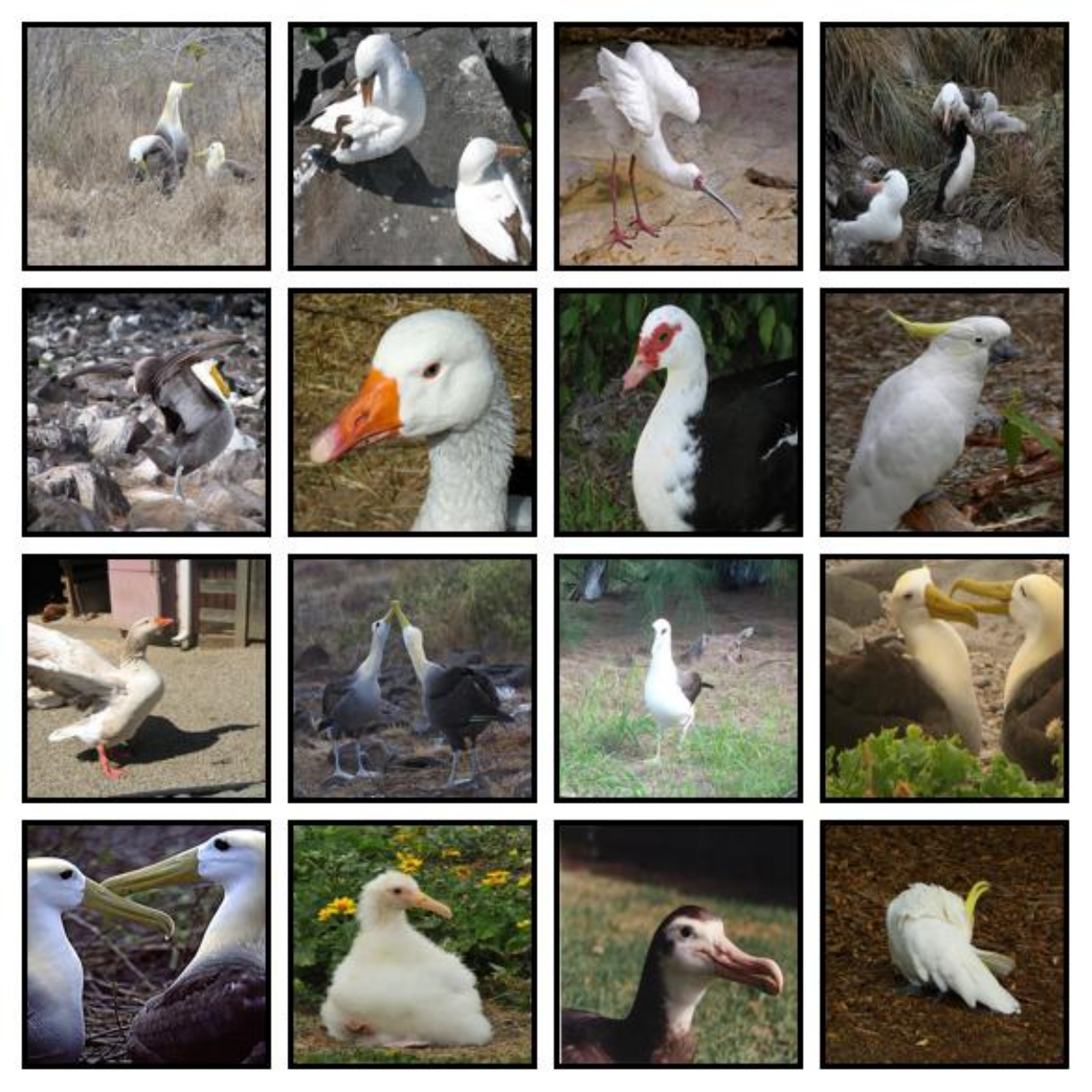} & \includegraphics[width=0.15\linewidth]{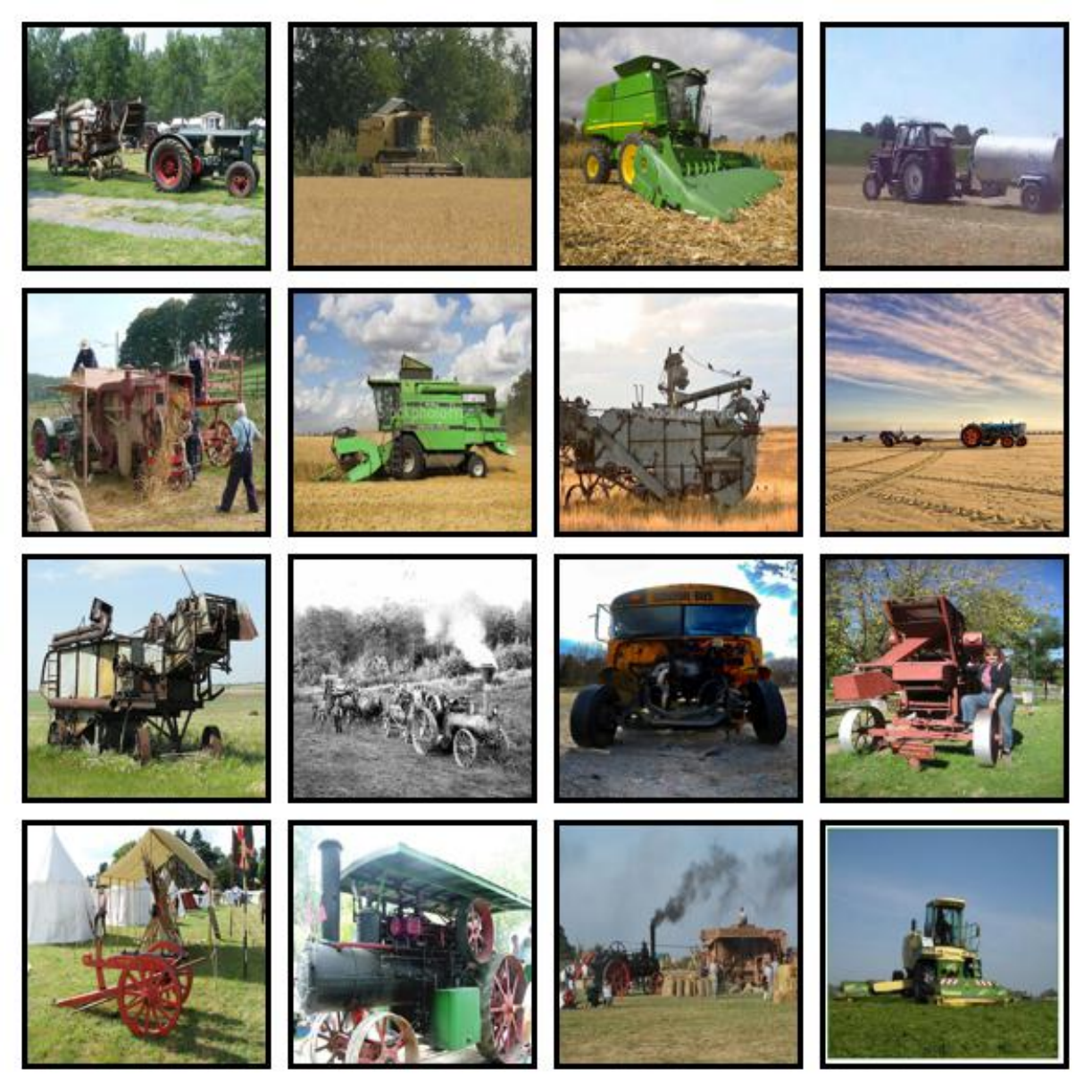} \\
  \includegraphics[width=0.15\linewidth]{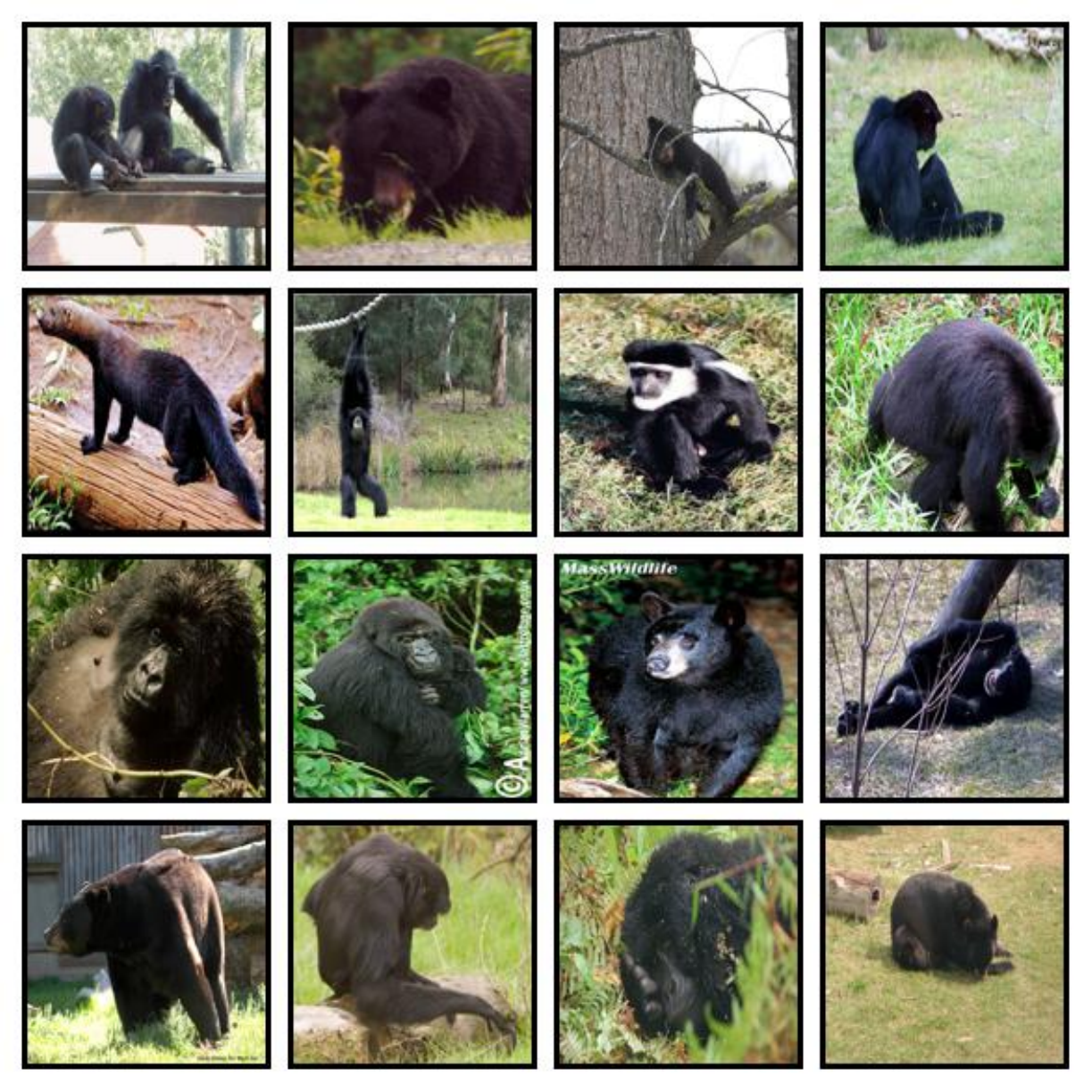} & \includegraphics[width=0.15\linewidth]{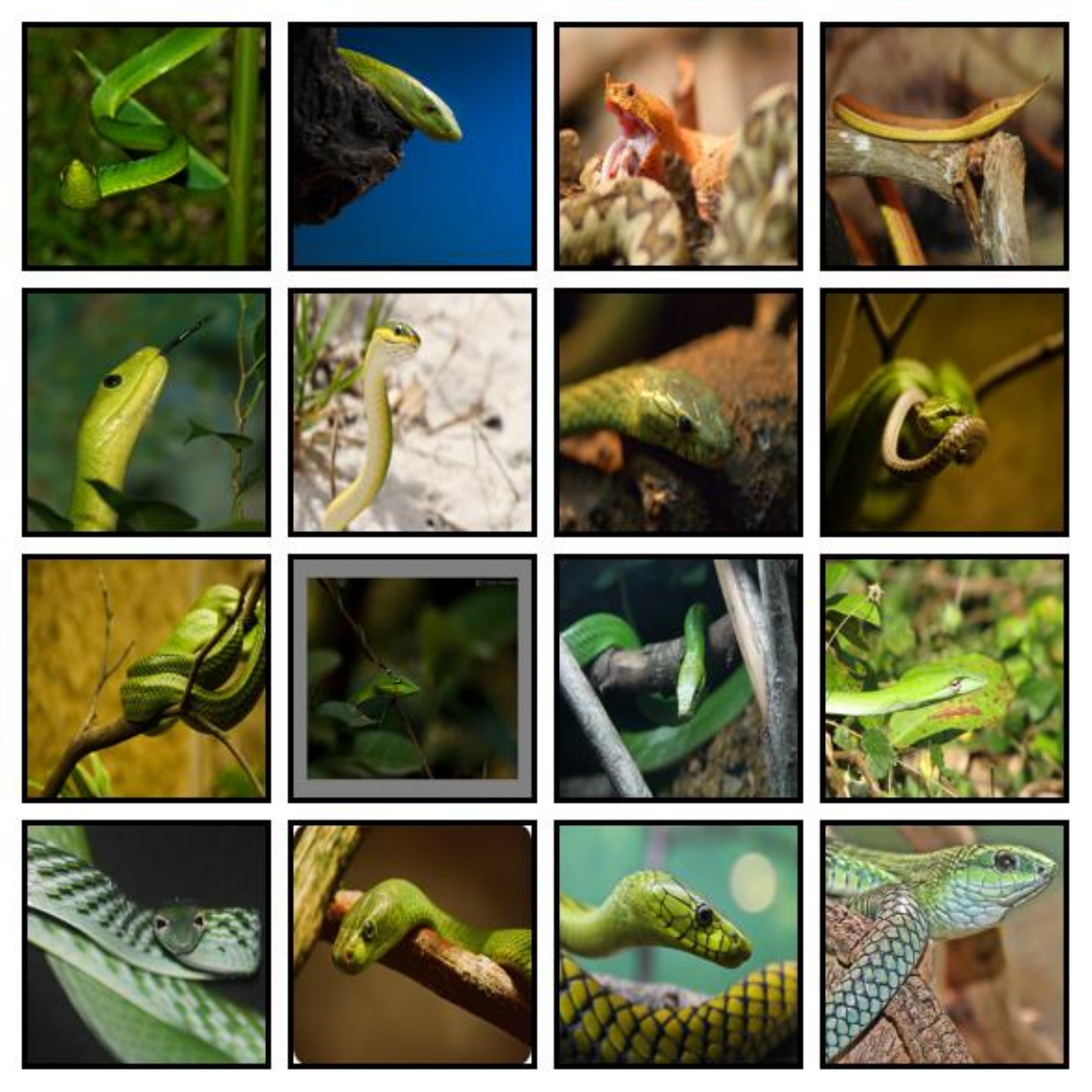} & \includegraphics[width=0.15\linewidth]{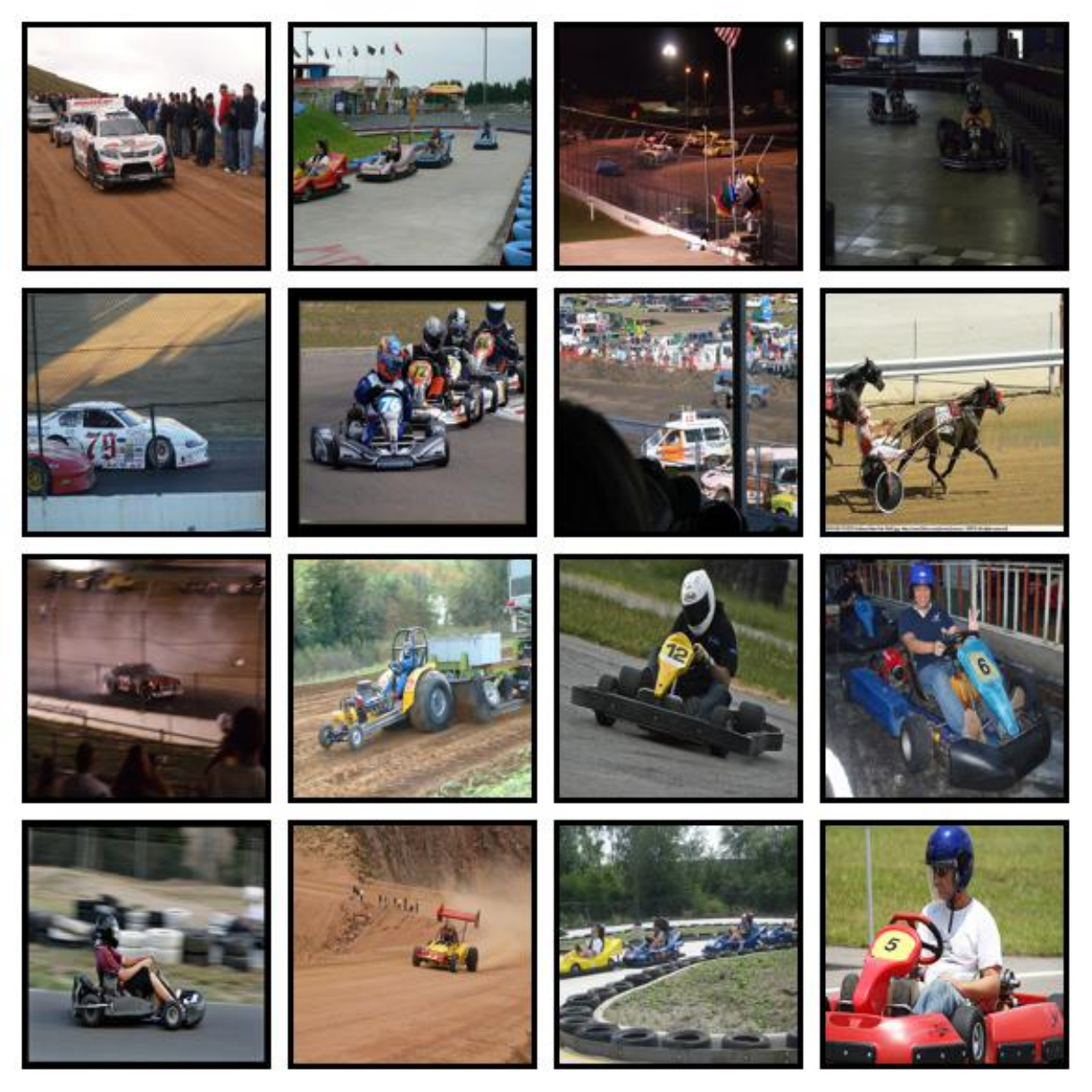} & \includegraphics[width=0.15\linewidth]{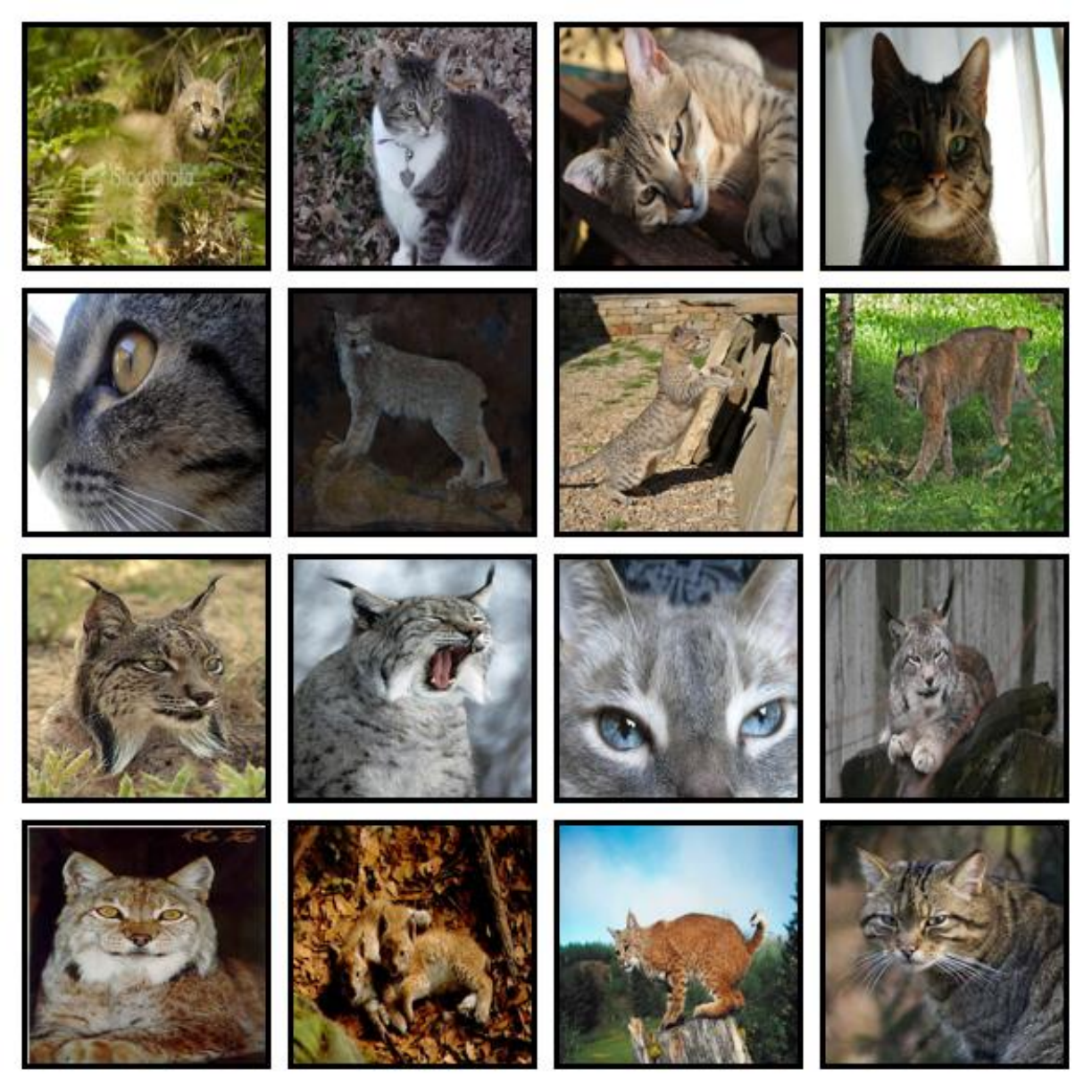} & \includegraphics[width=0.15\linewidth]{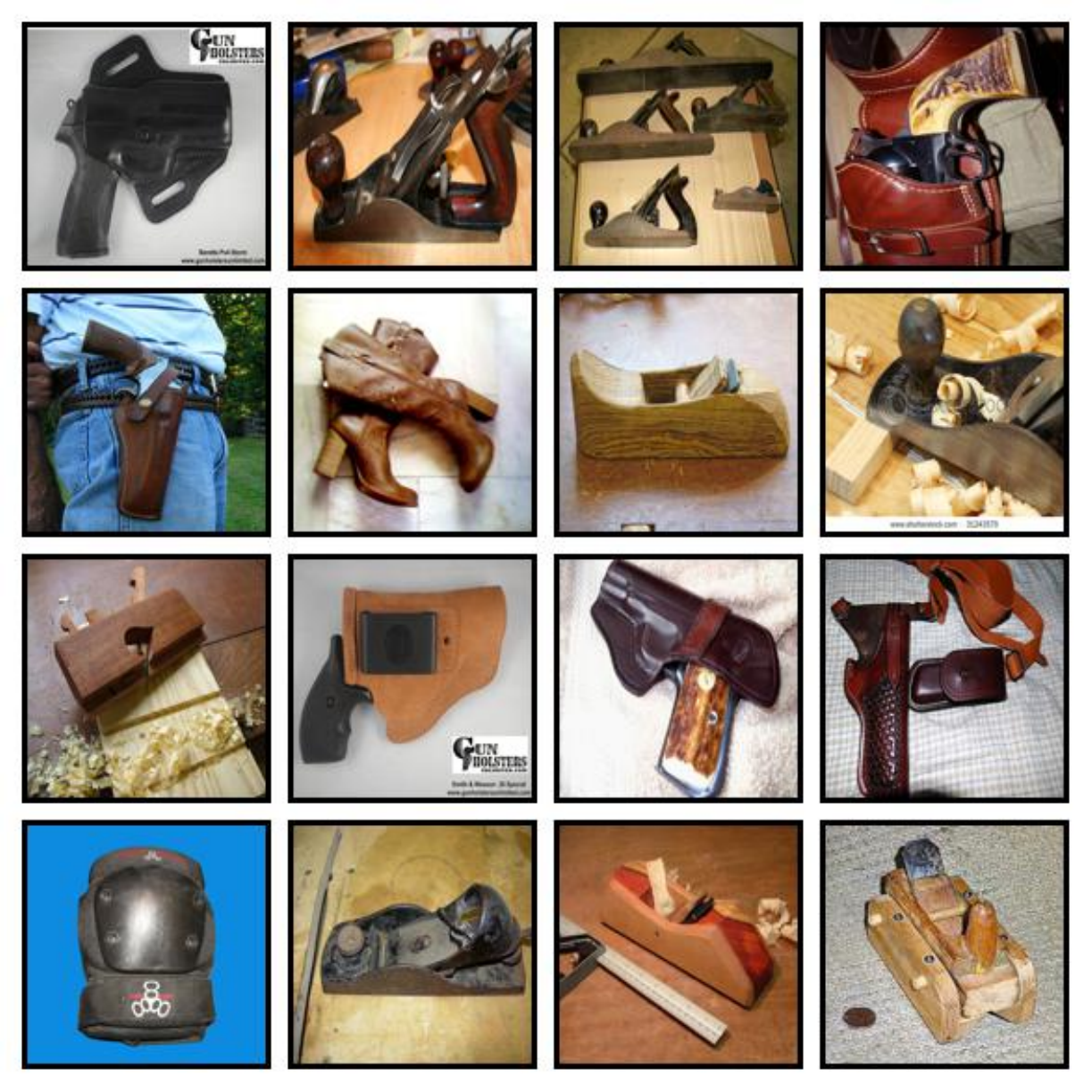} \\ \includegraphics[width=0.15\linewidth]{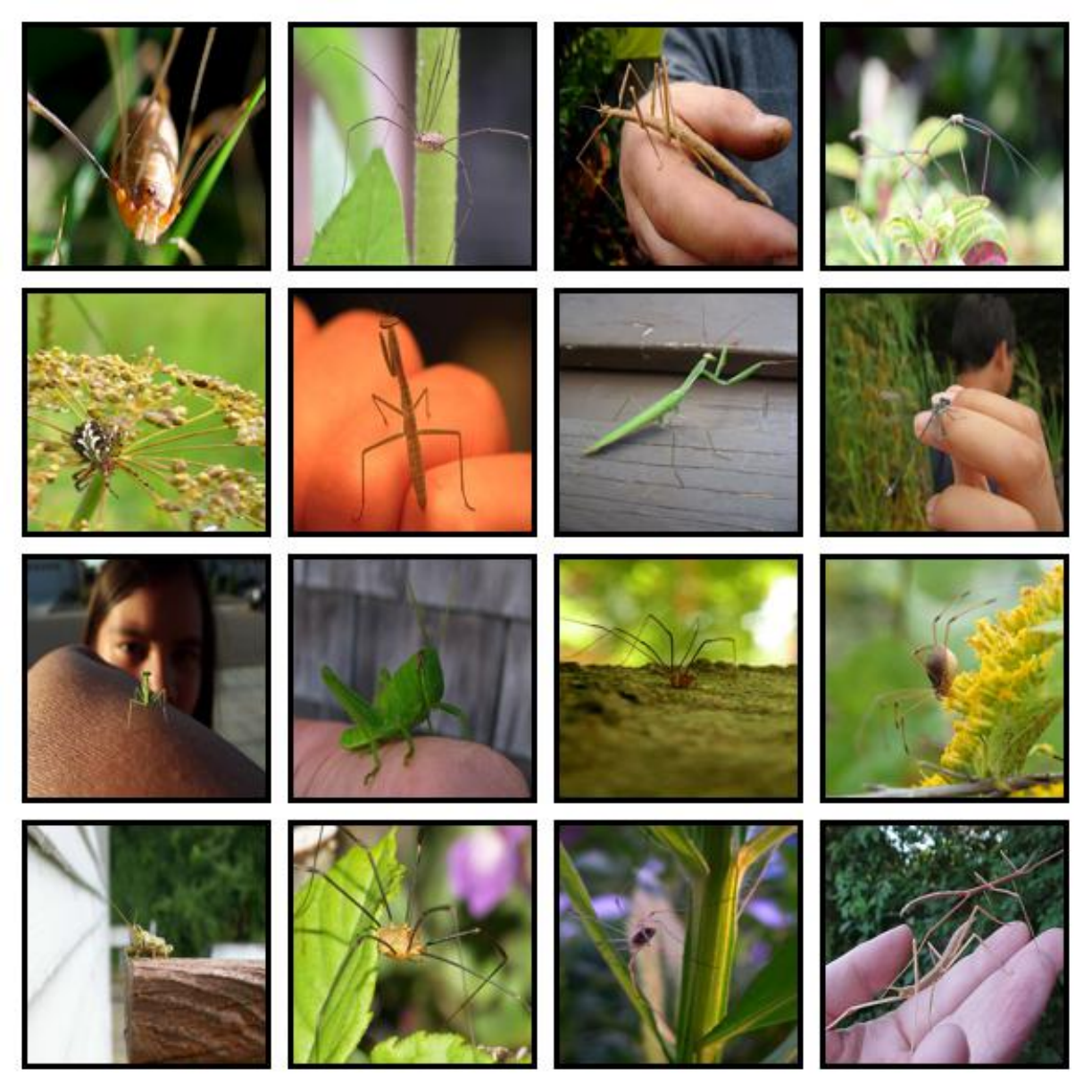} & \includegraphics[width=0.15\linewidth]{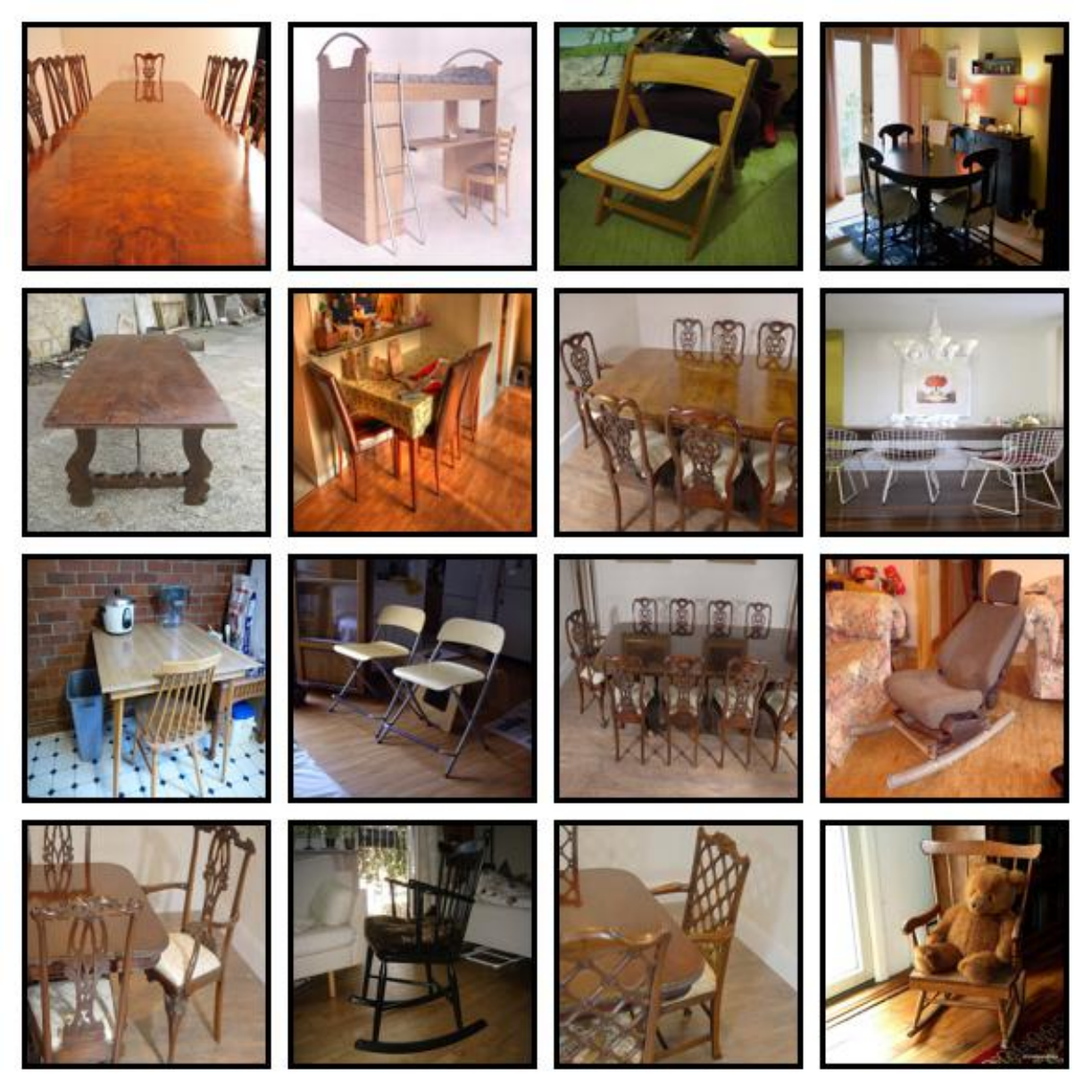} & \includegraphics[width=0.15\linewidth]{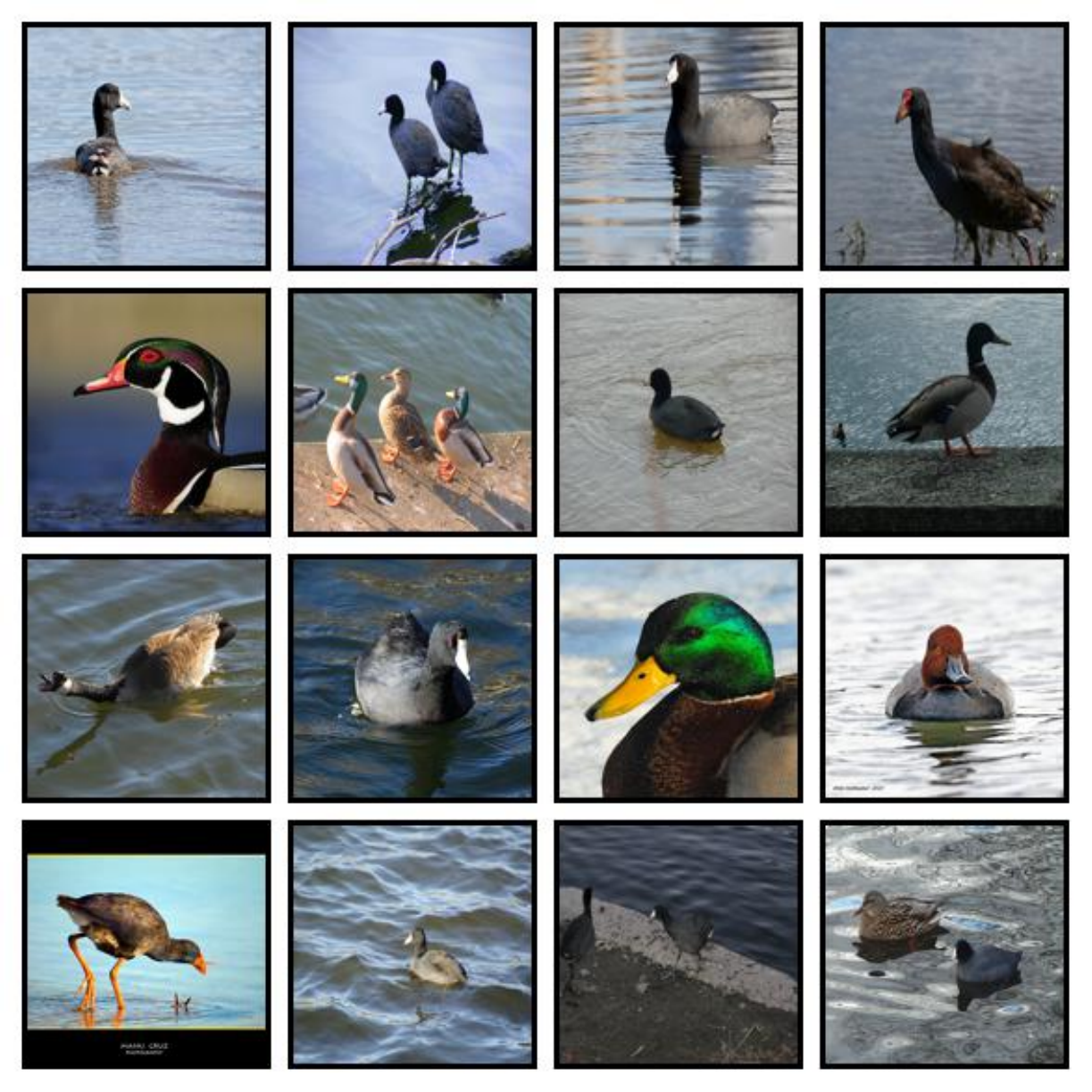} & \includegraphics[width=0.15\linewidth]{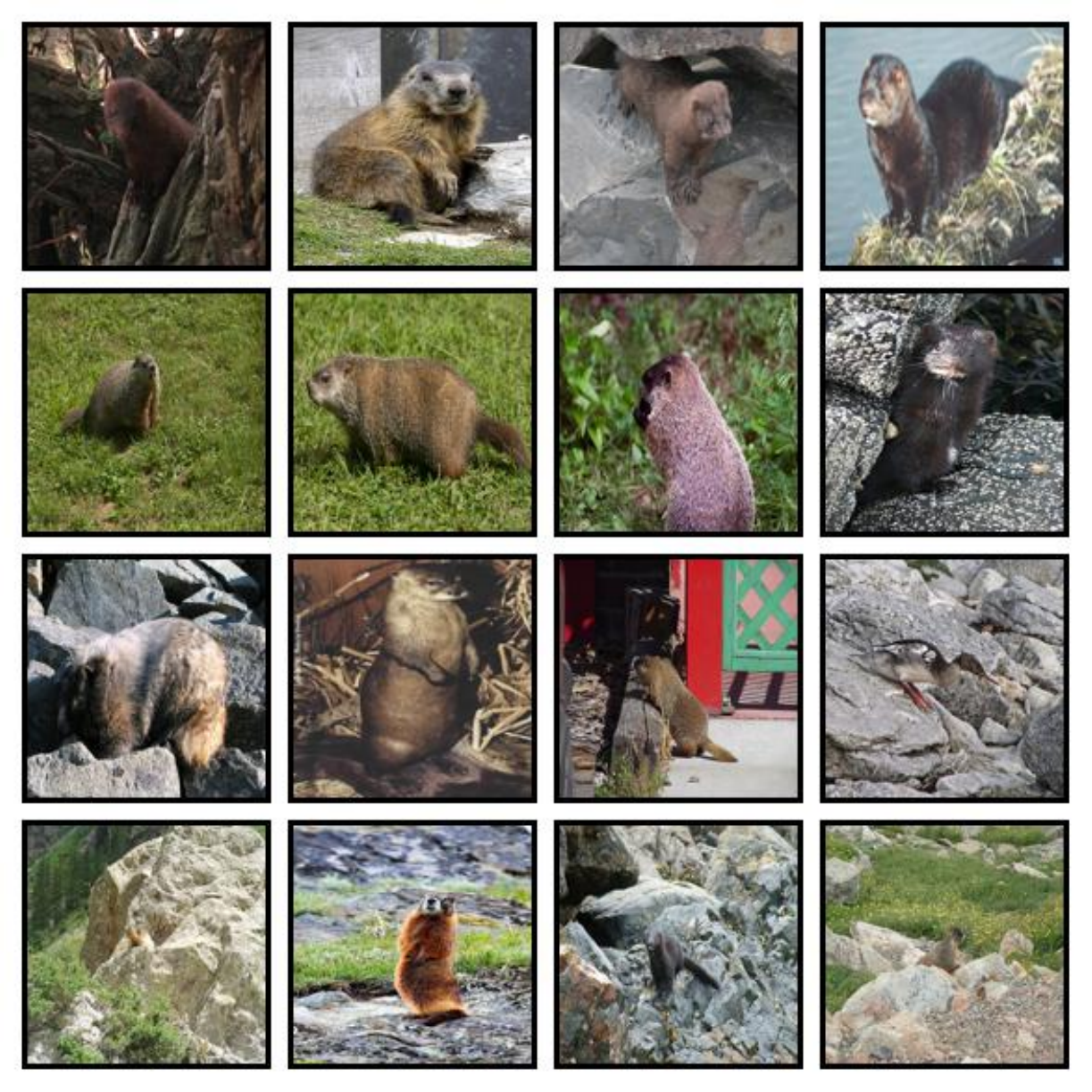} & \includegraphics[width=0.15\linewidth]{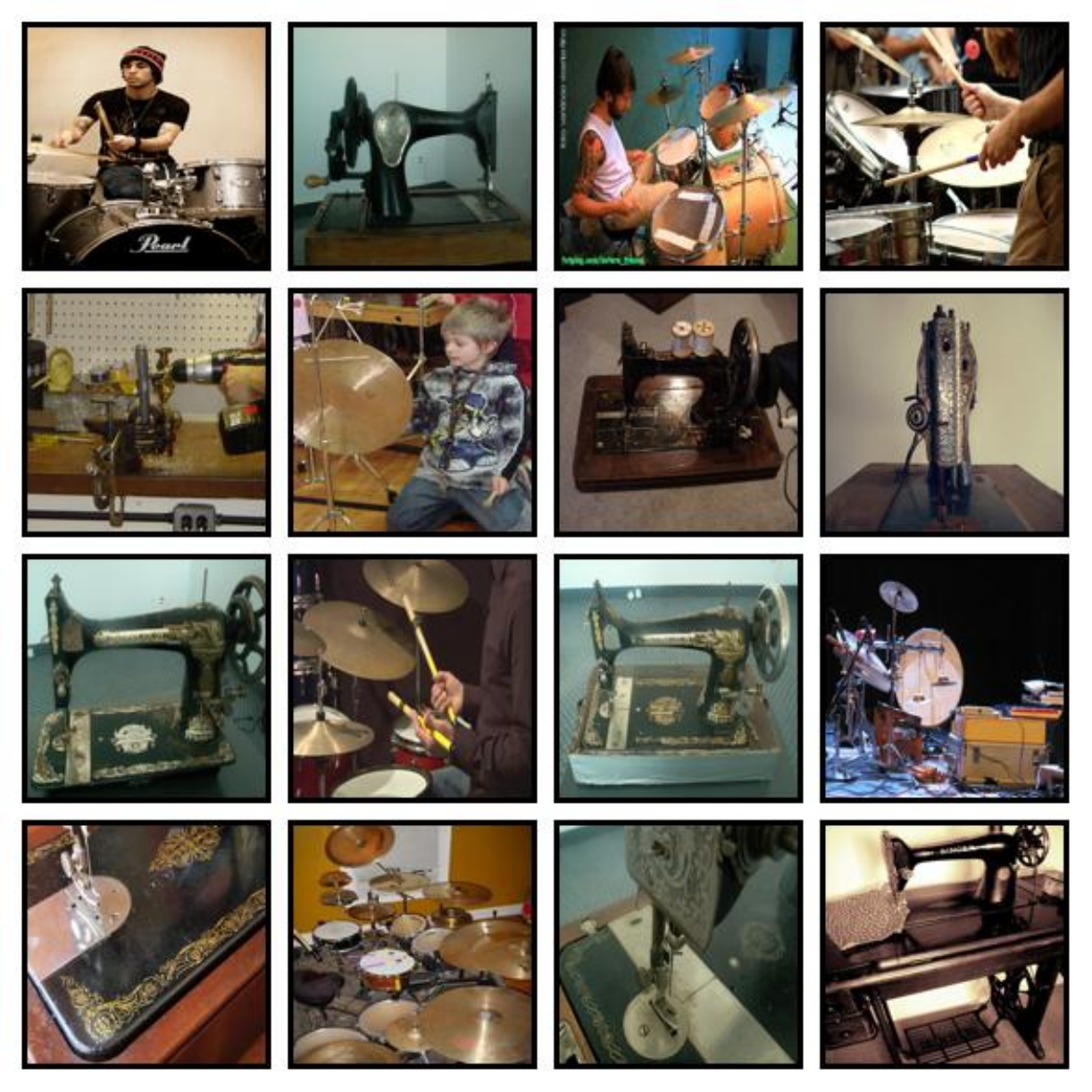} \\
  \includegraphics[width=0.15\linewidth]{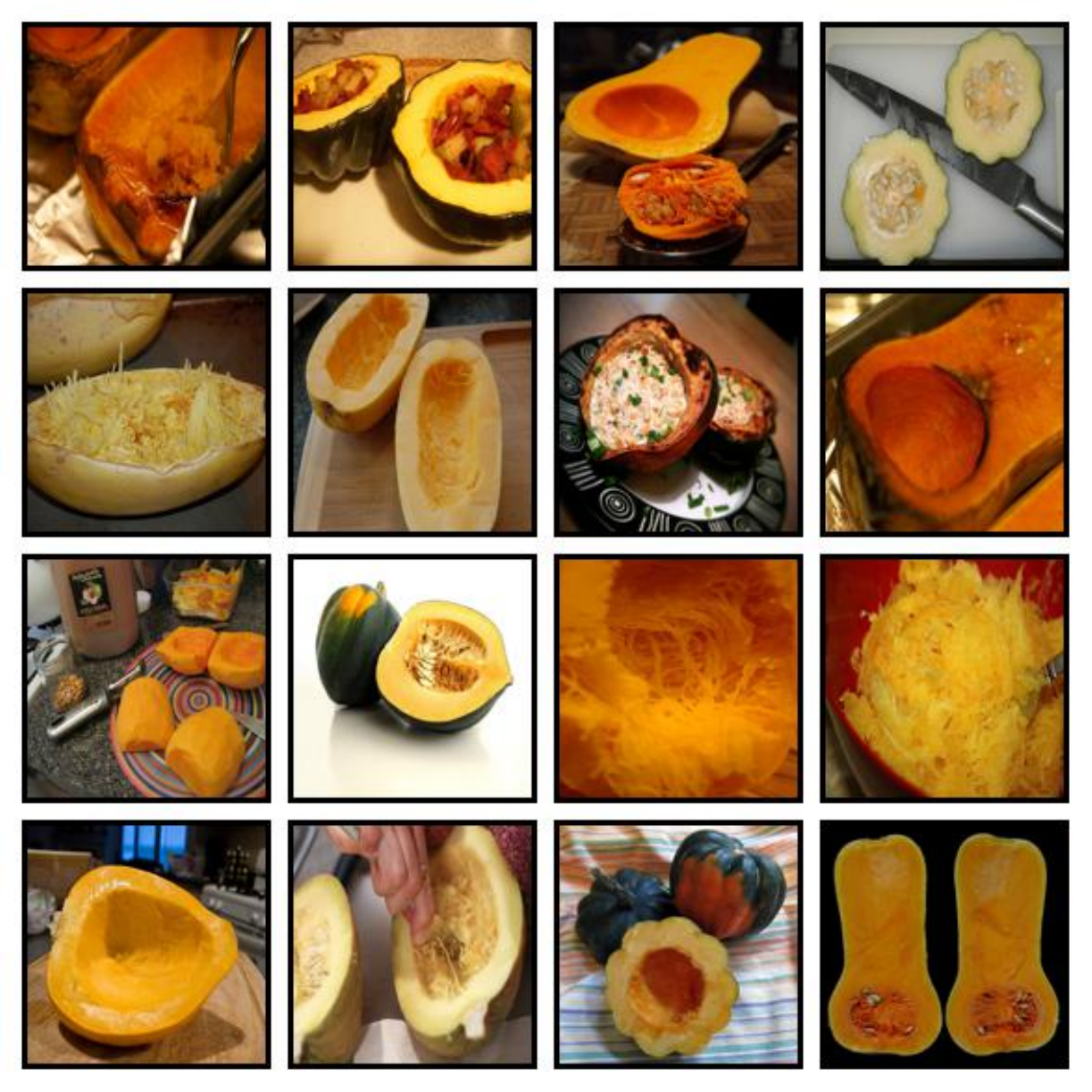} & \includegraphics[width=0.15\linewidth]{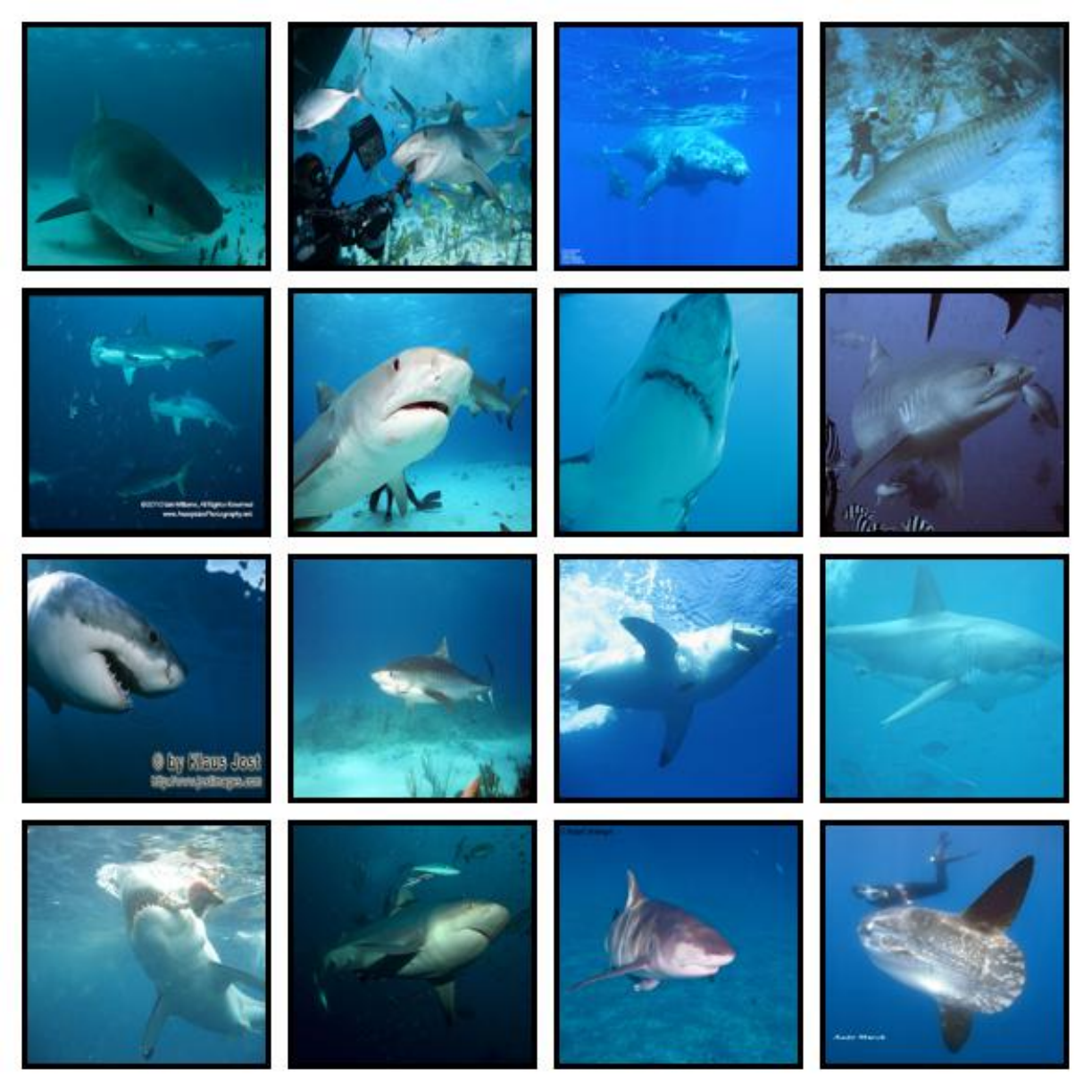} & \includegraphics[width=0.15\linewidth]{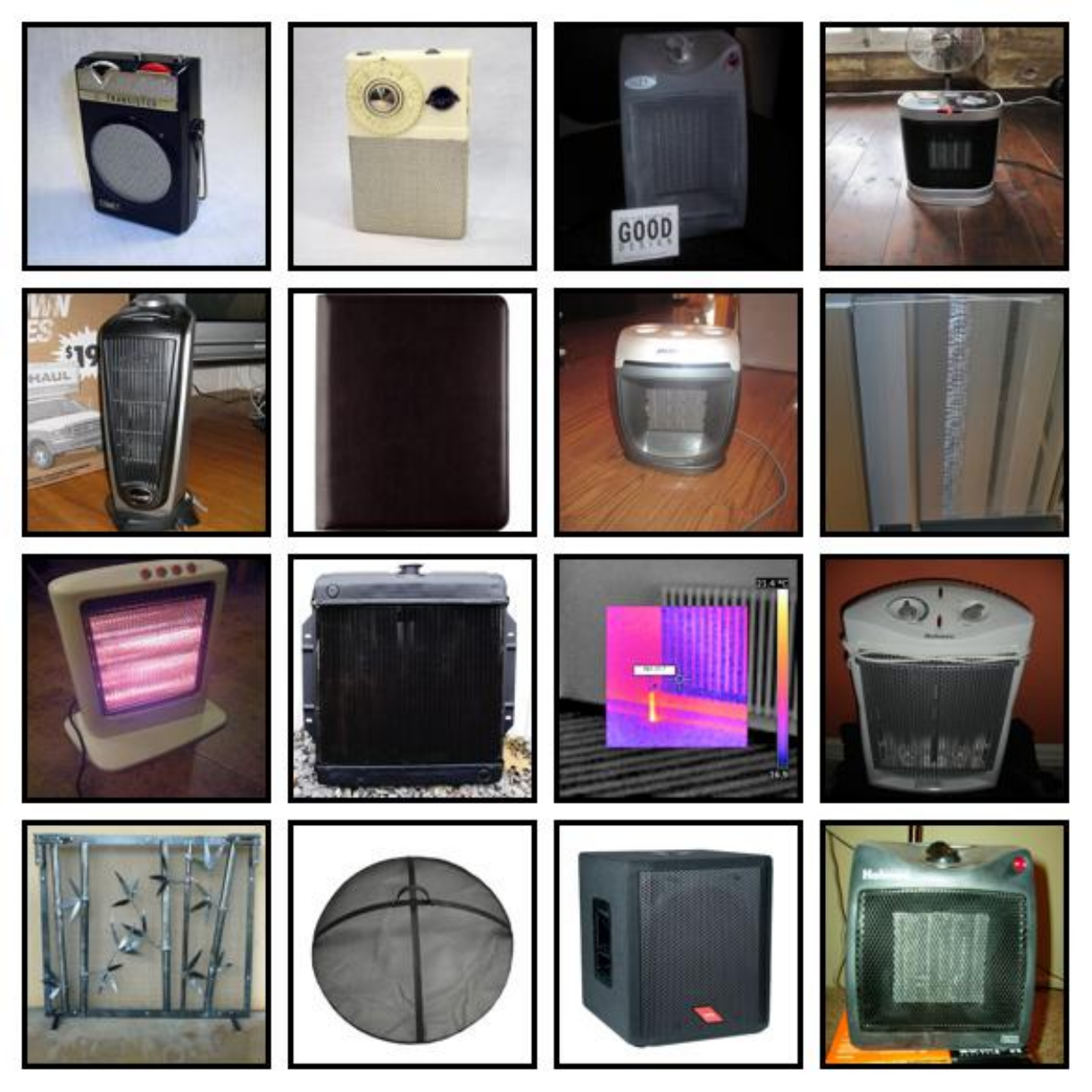} & \includegraphics[width=0.15\linewidth]{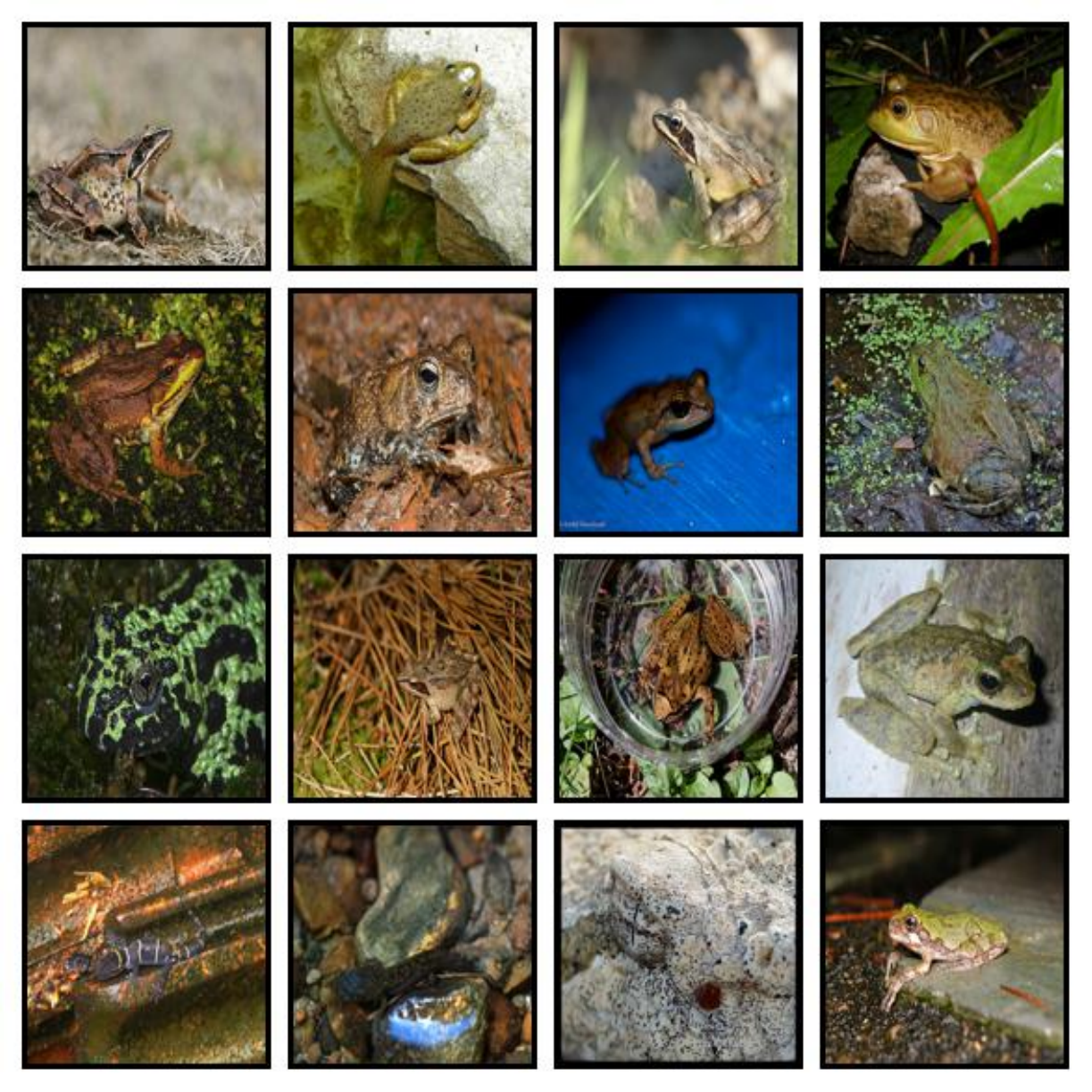} & \includegraphics[width=0.15\linewidth]{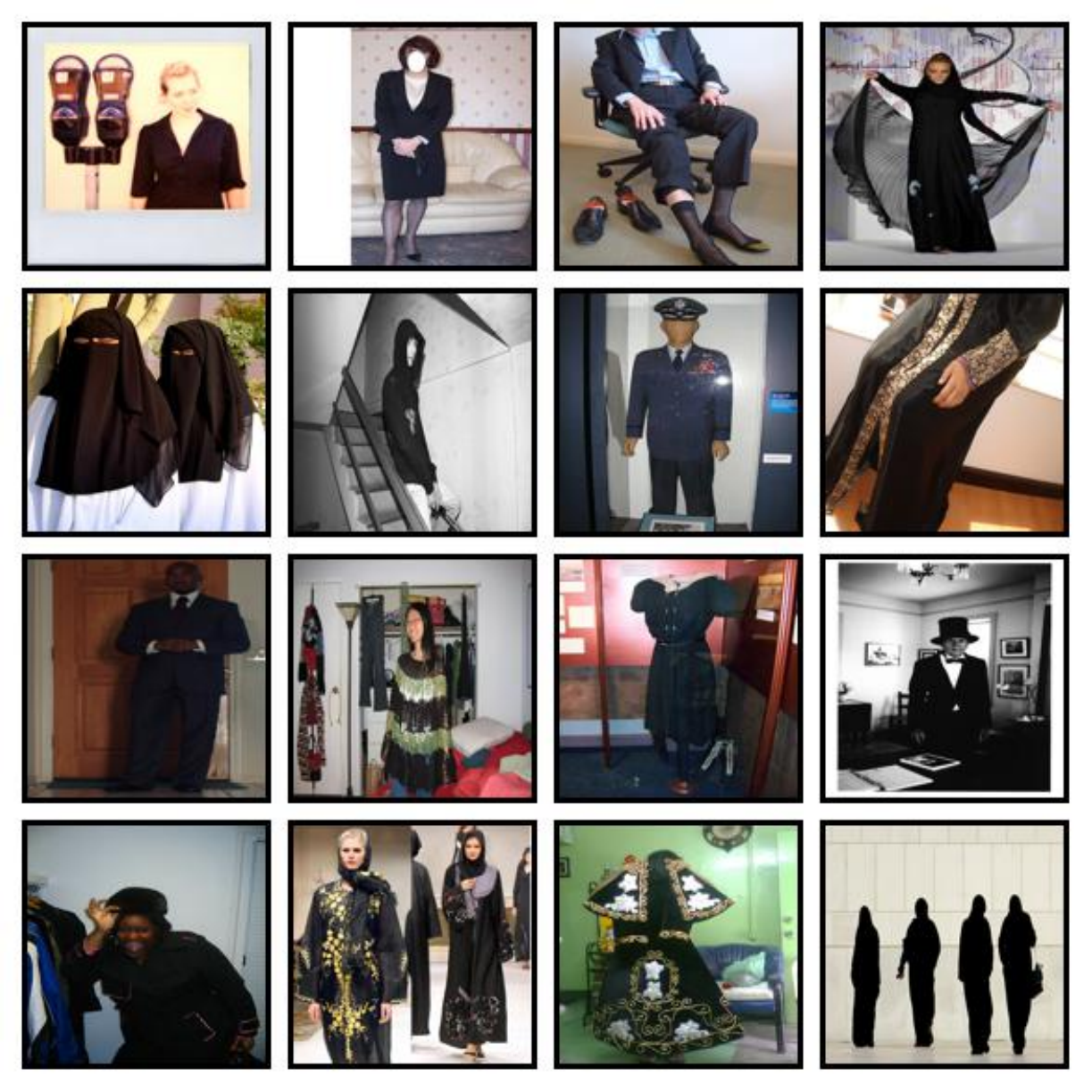} \\
  \includegraphics[width=0.15\linewidth]{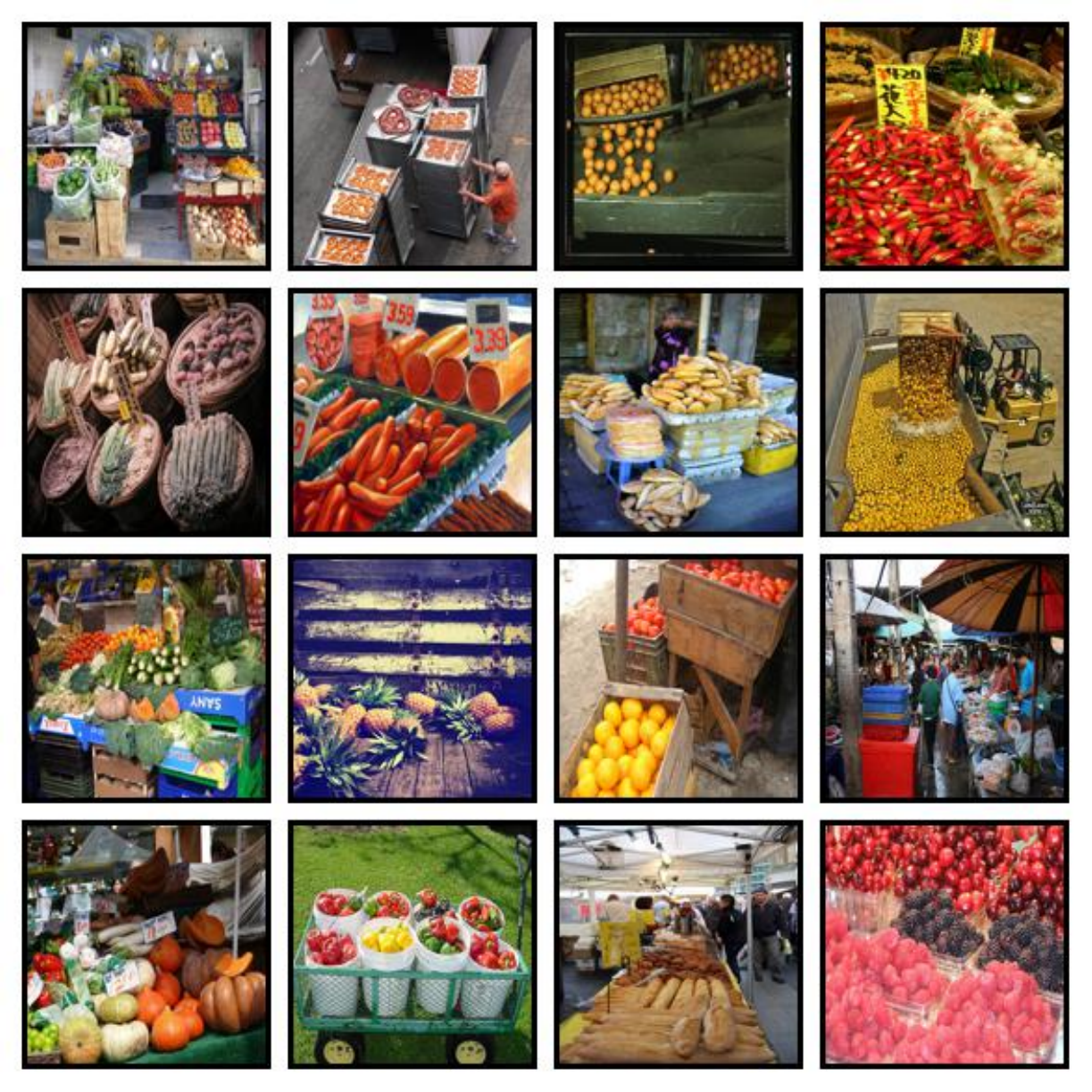} & \includegraphics[width=0.15\linewidth]{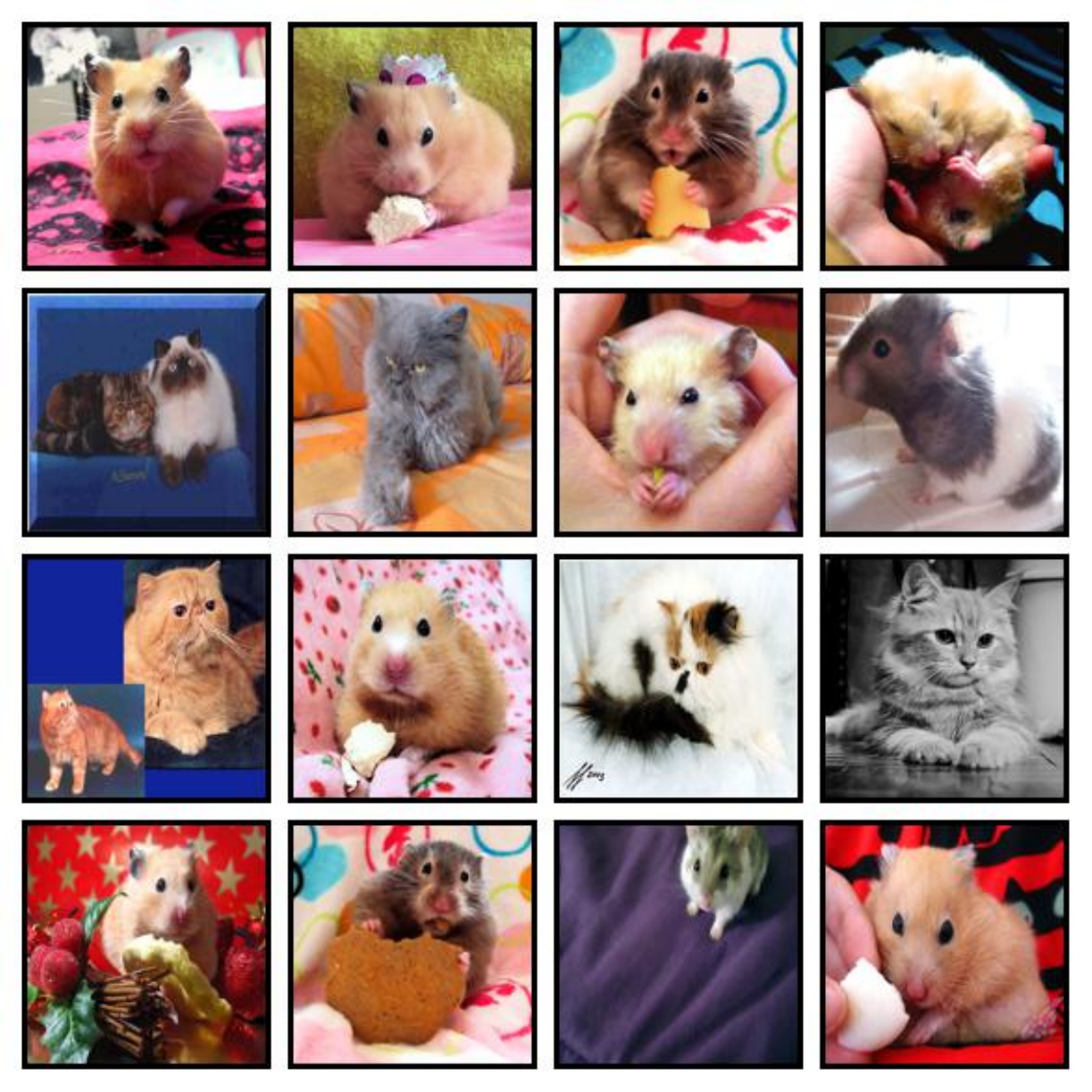} & \includegraphics[width=0.15\linewidth]{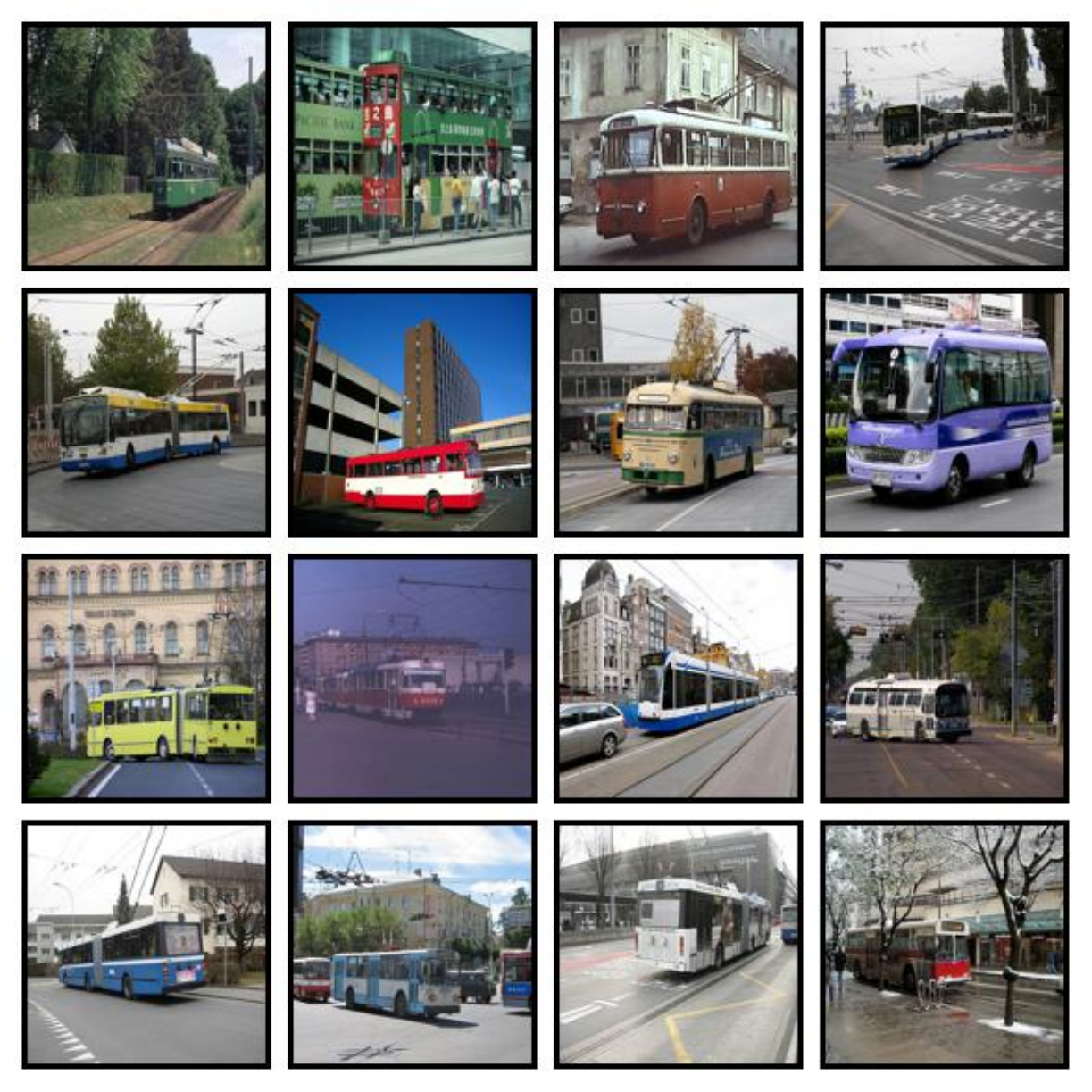} & \includegraphics[width=0.15\linewidth]{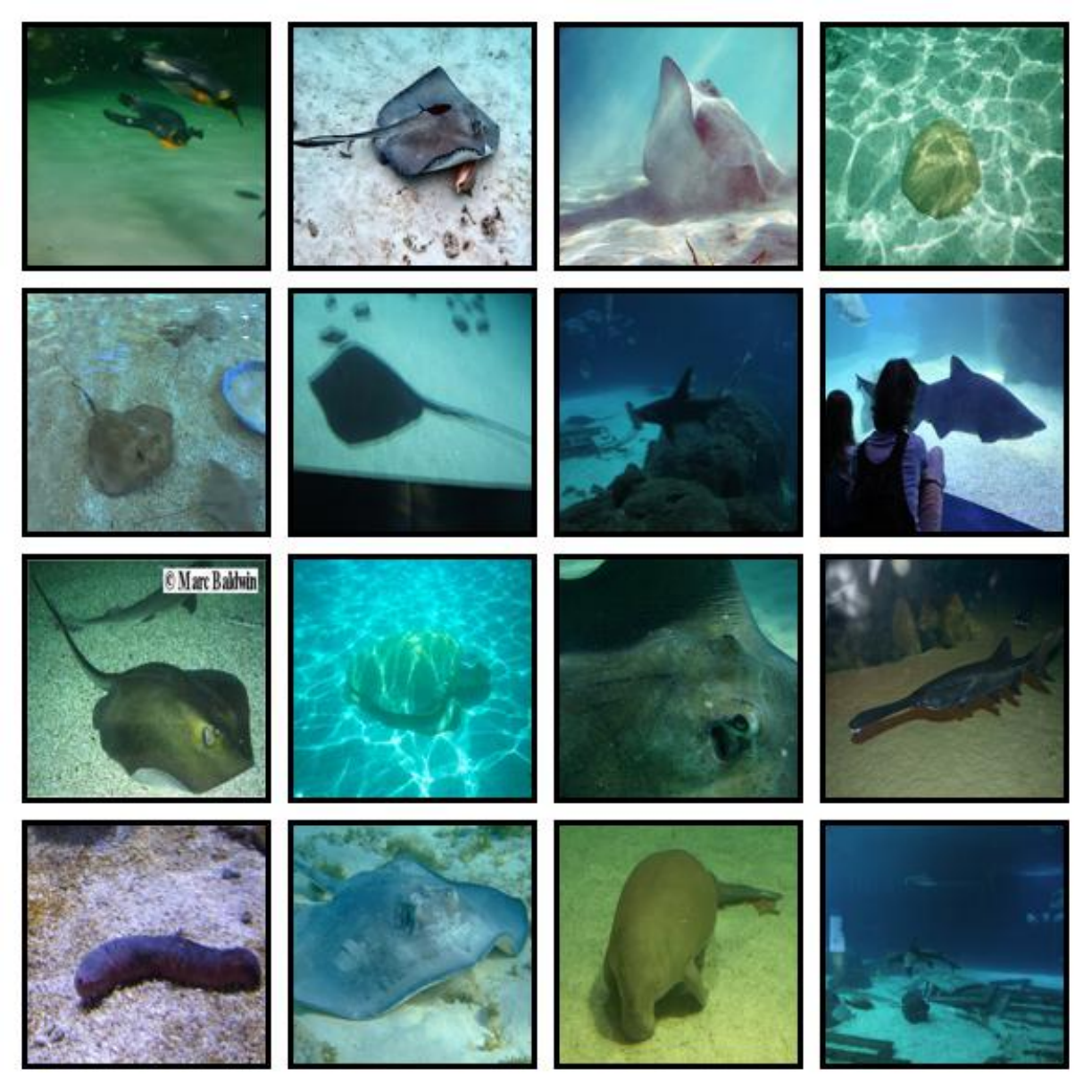} & \includegraphics[width=0.15\linewidth]{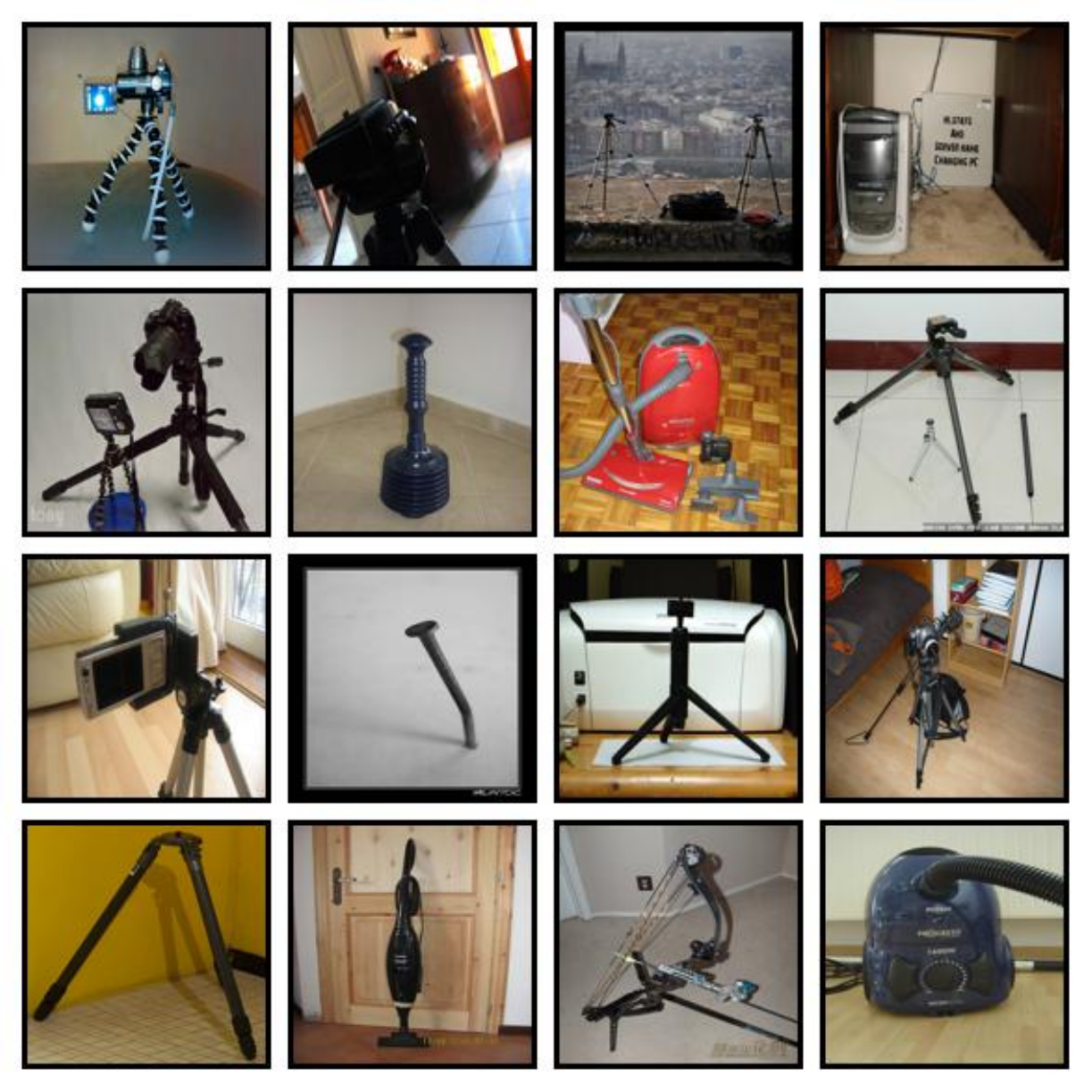} \\
  \includegraphics[width=0.15\linewidth]{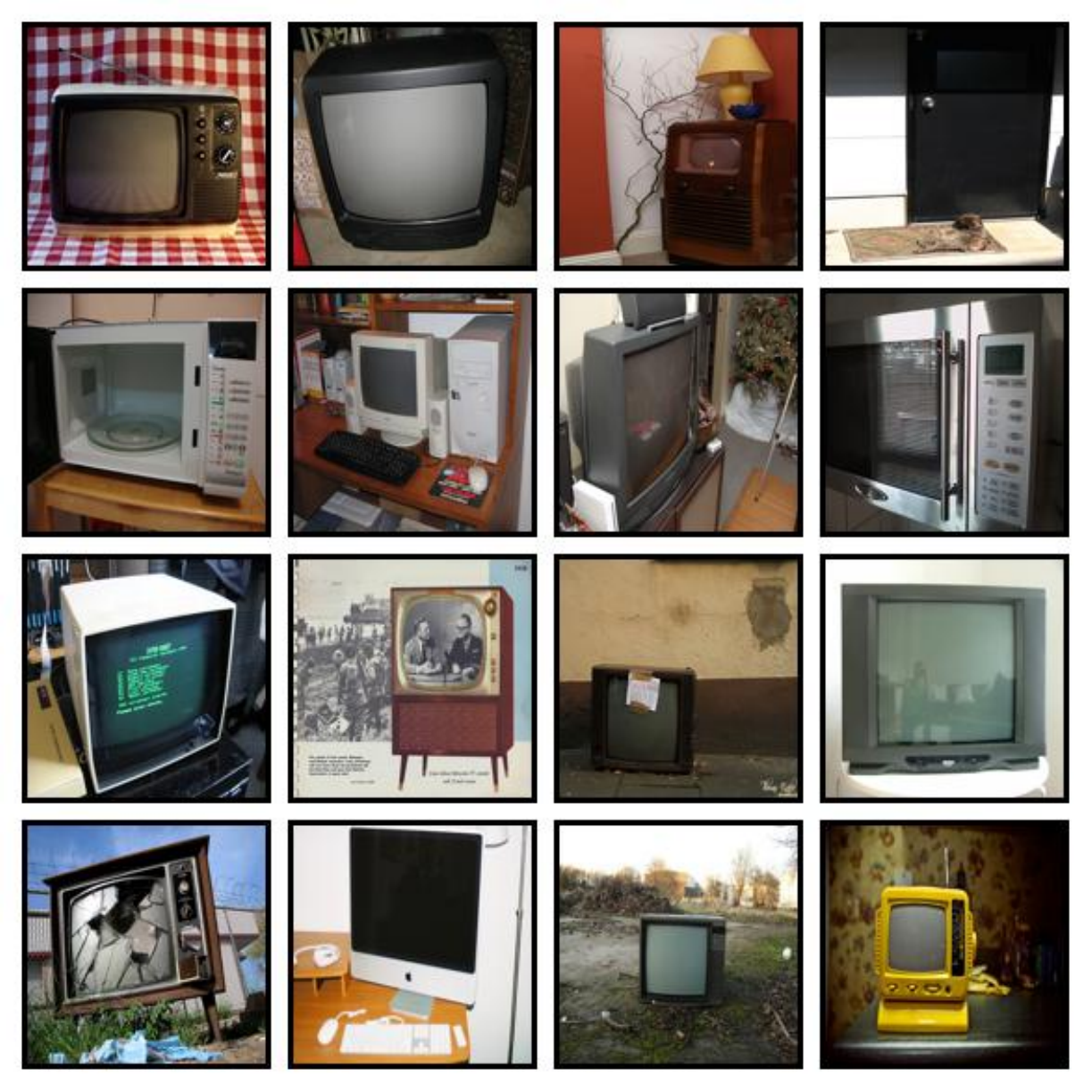} & \includegraphics[width=0.15\linewidth]{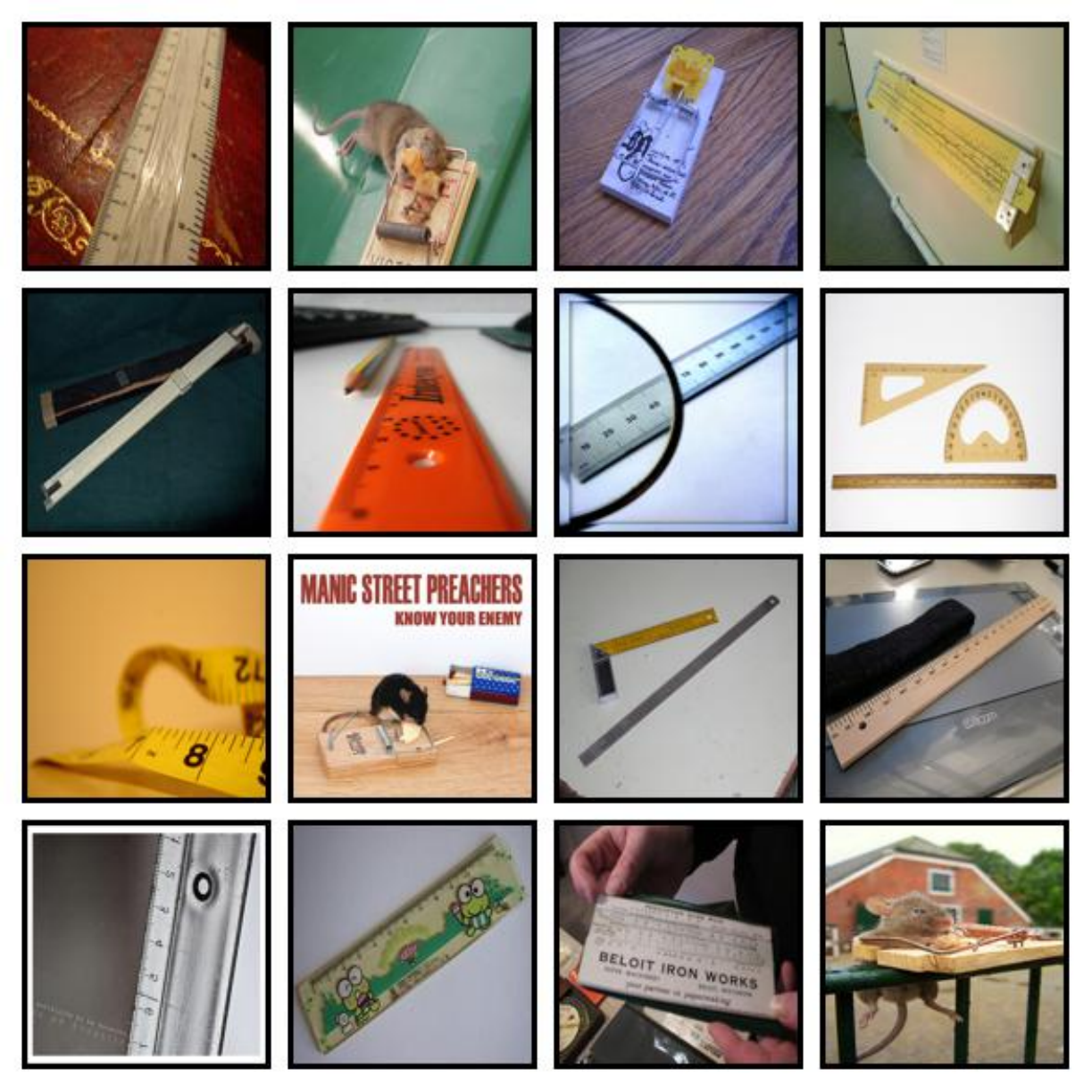} & \includegraphics[width=0.15\linewidth]{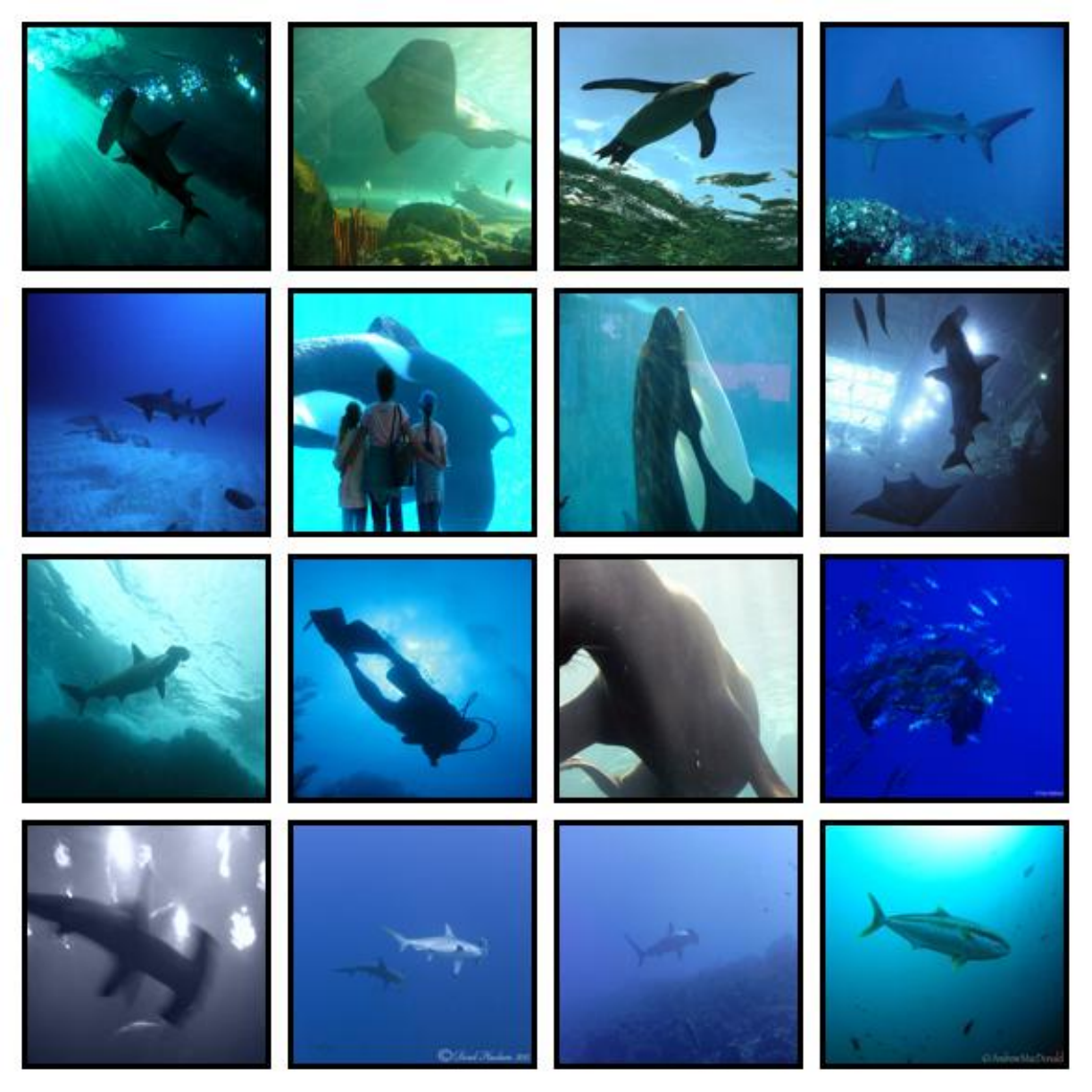} & \includegraphics[width=0.15\linewidth]{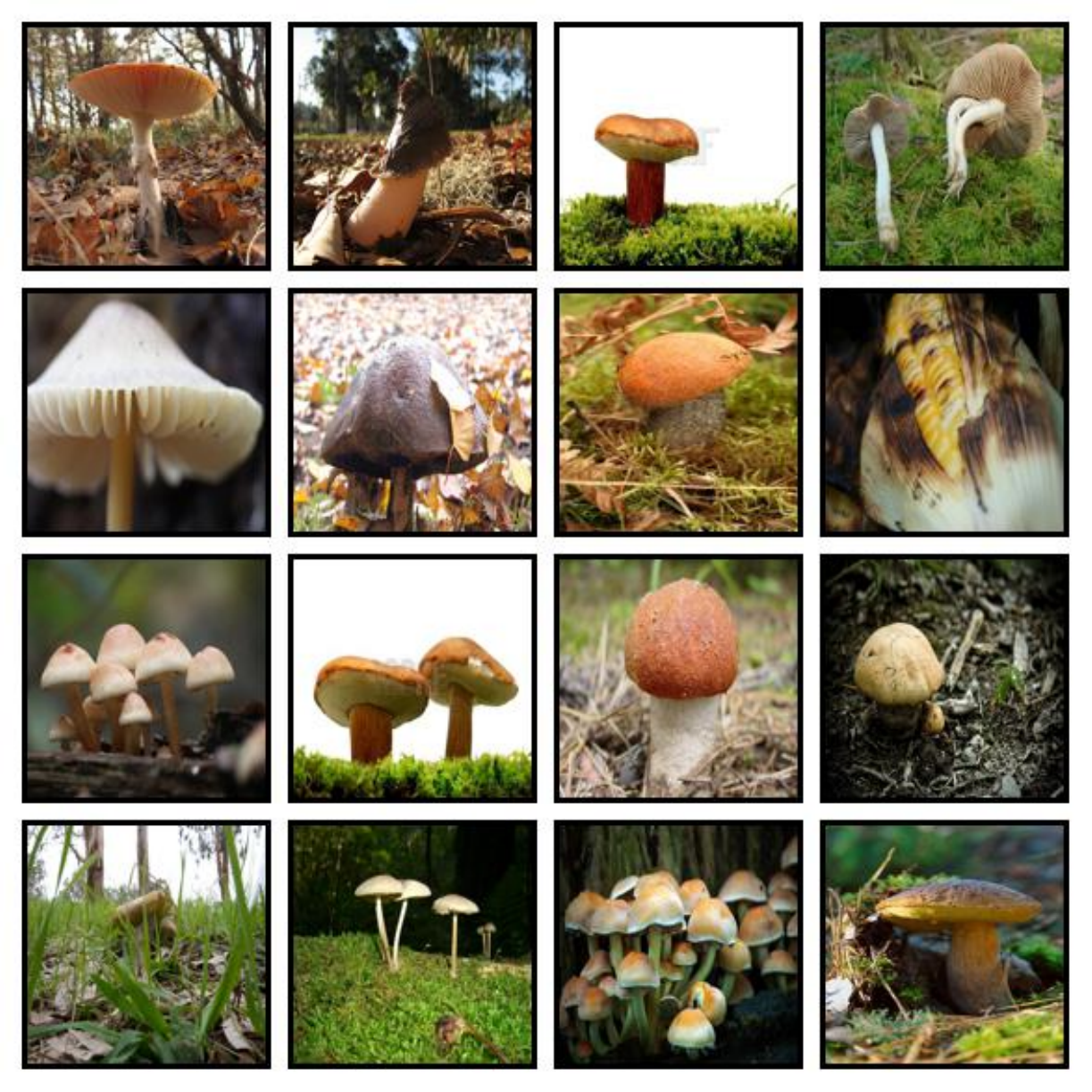} & \includegraphics[width=0.15\linewidth]{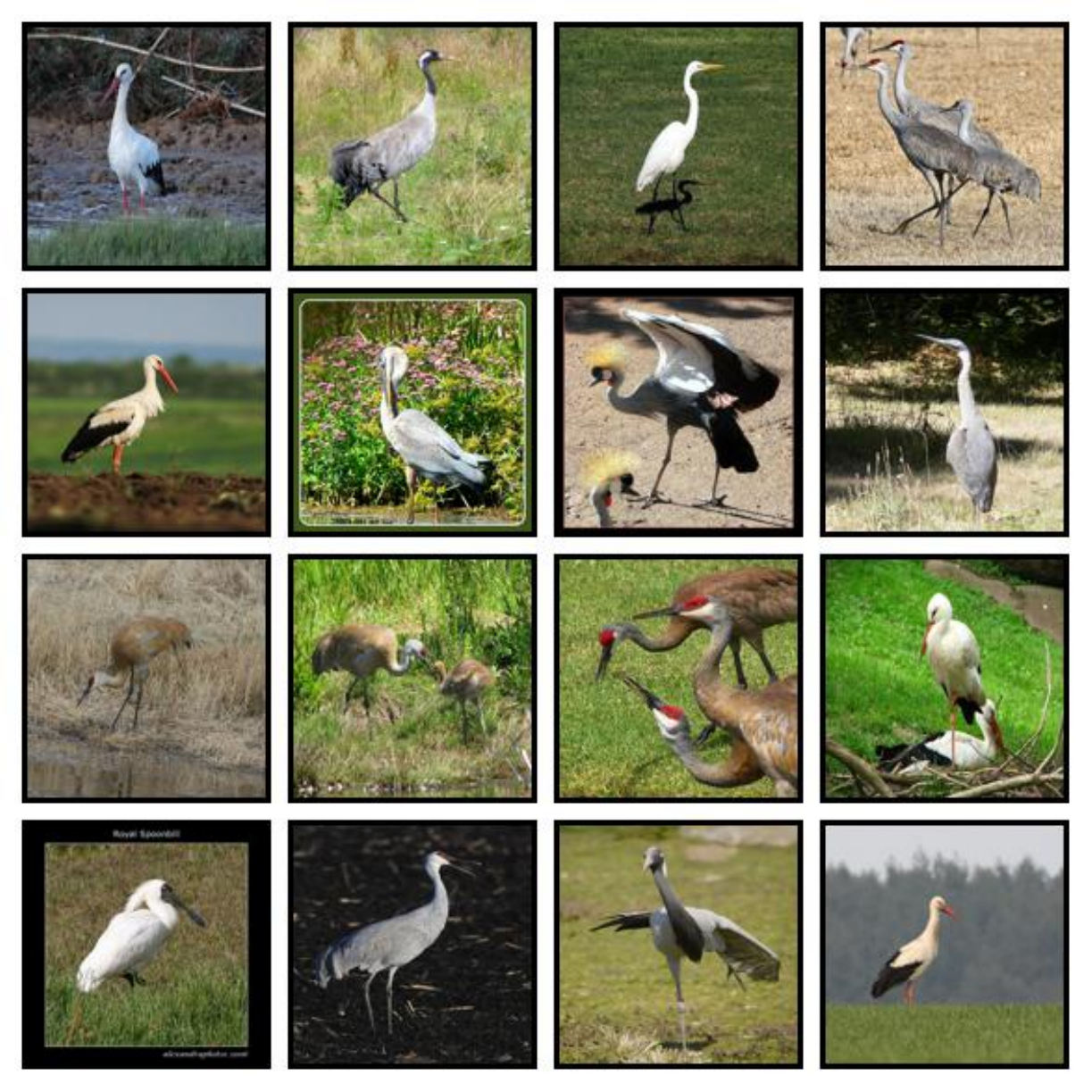} \\
  \includegraphics[width=0.15\linewidth]{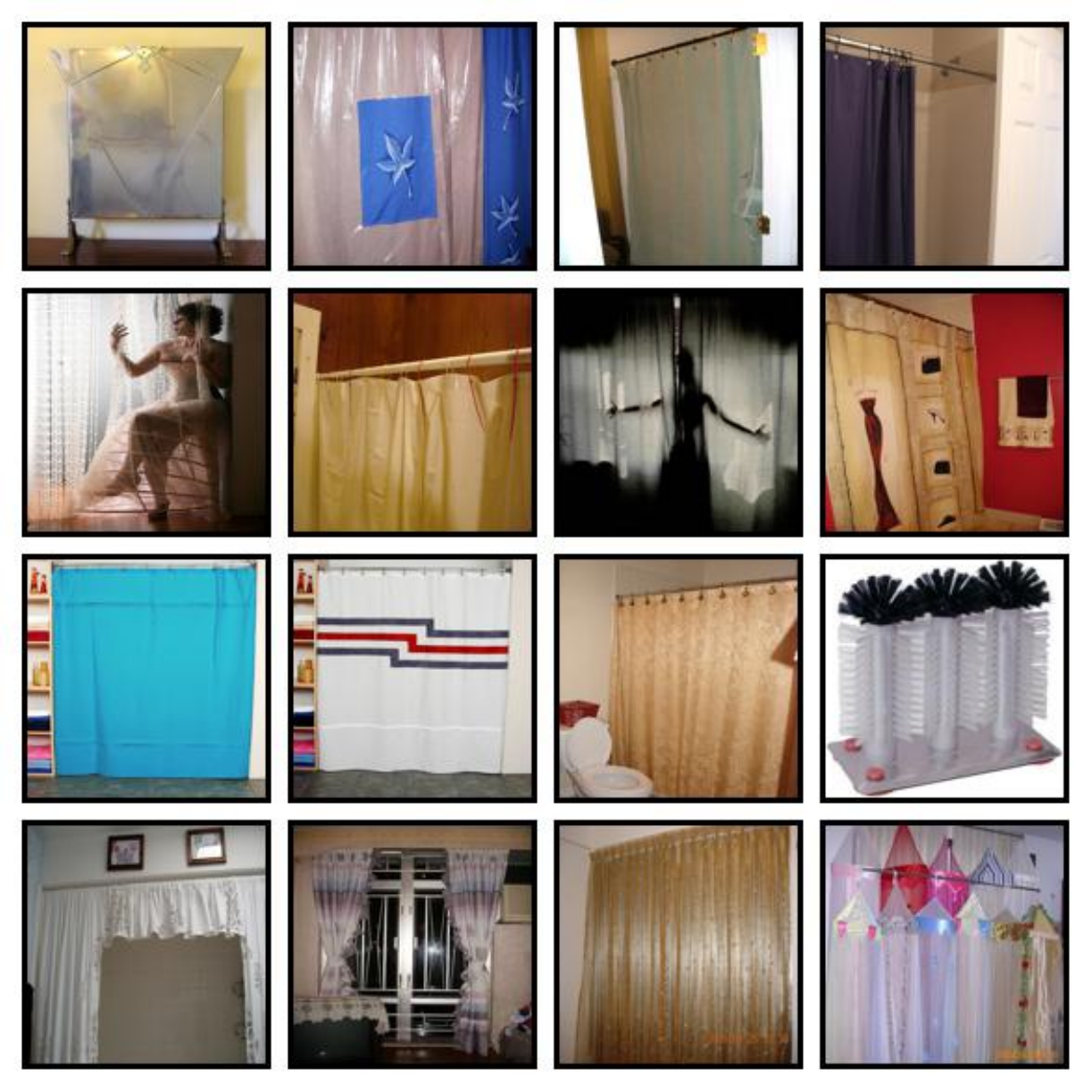} & \includegraphics[width=0.15\linewidth]{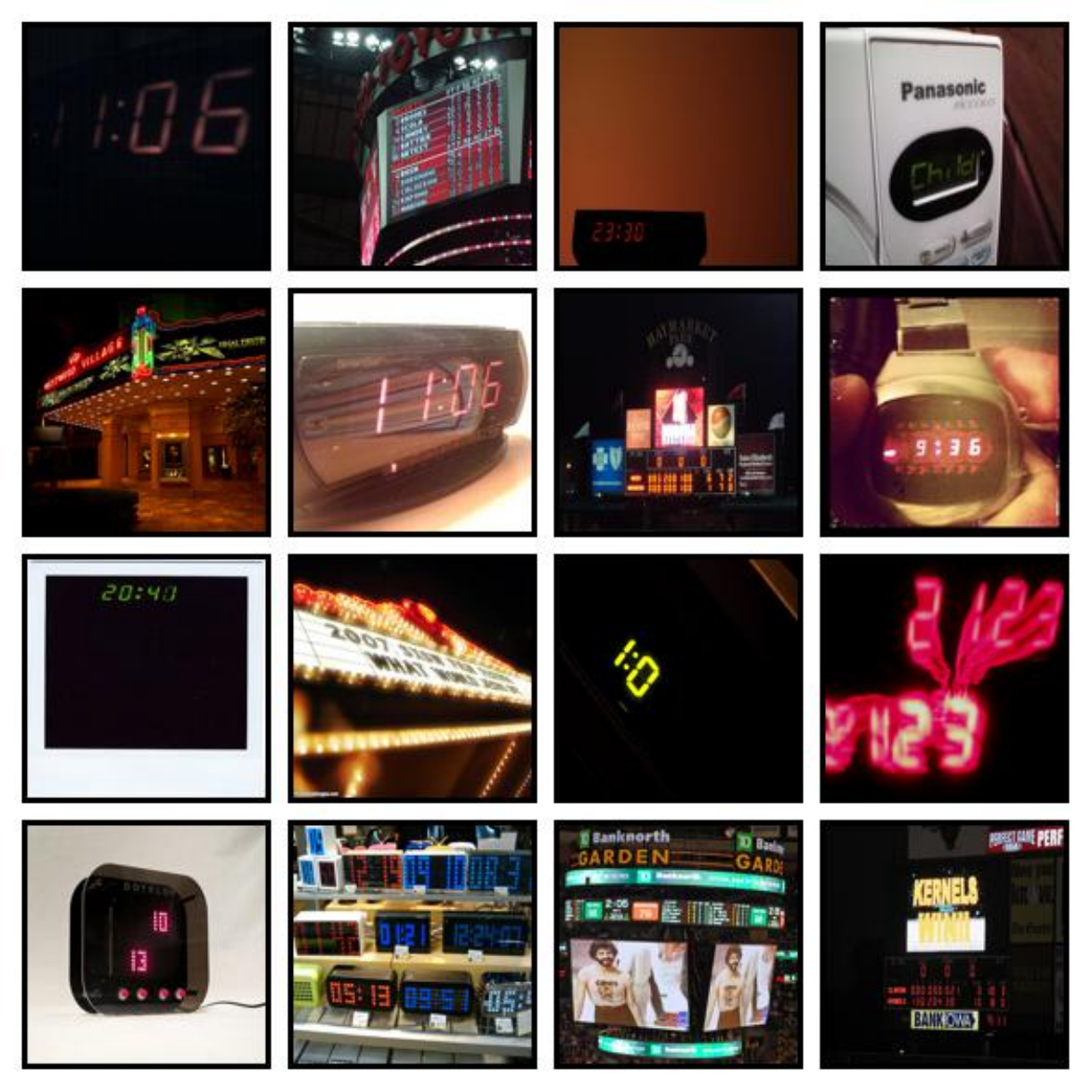} & \includegraphics[width=0.15\linewidth]{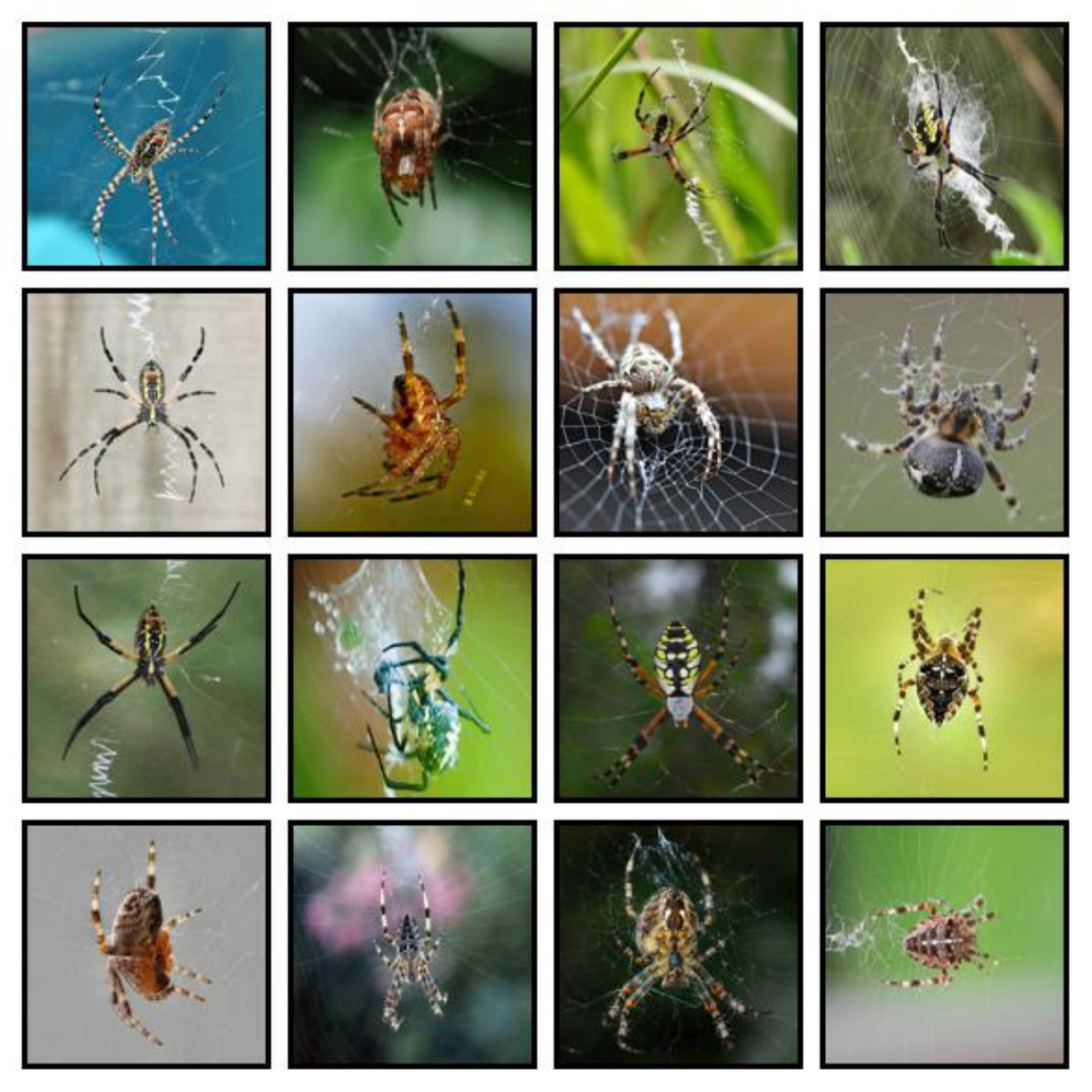} & \includegraphics[width=0.15\linewidth]{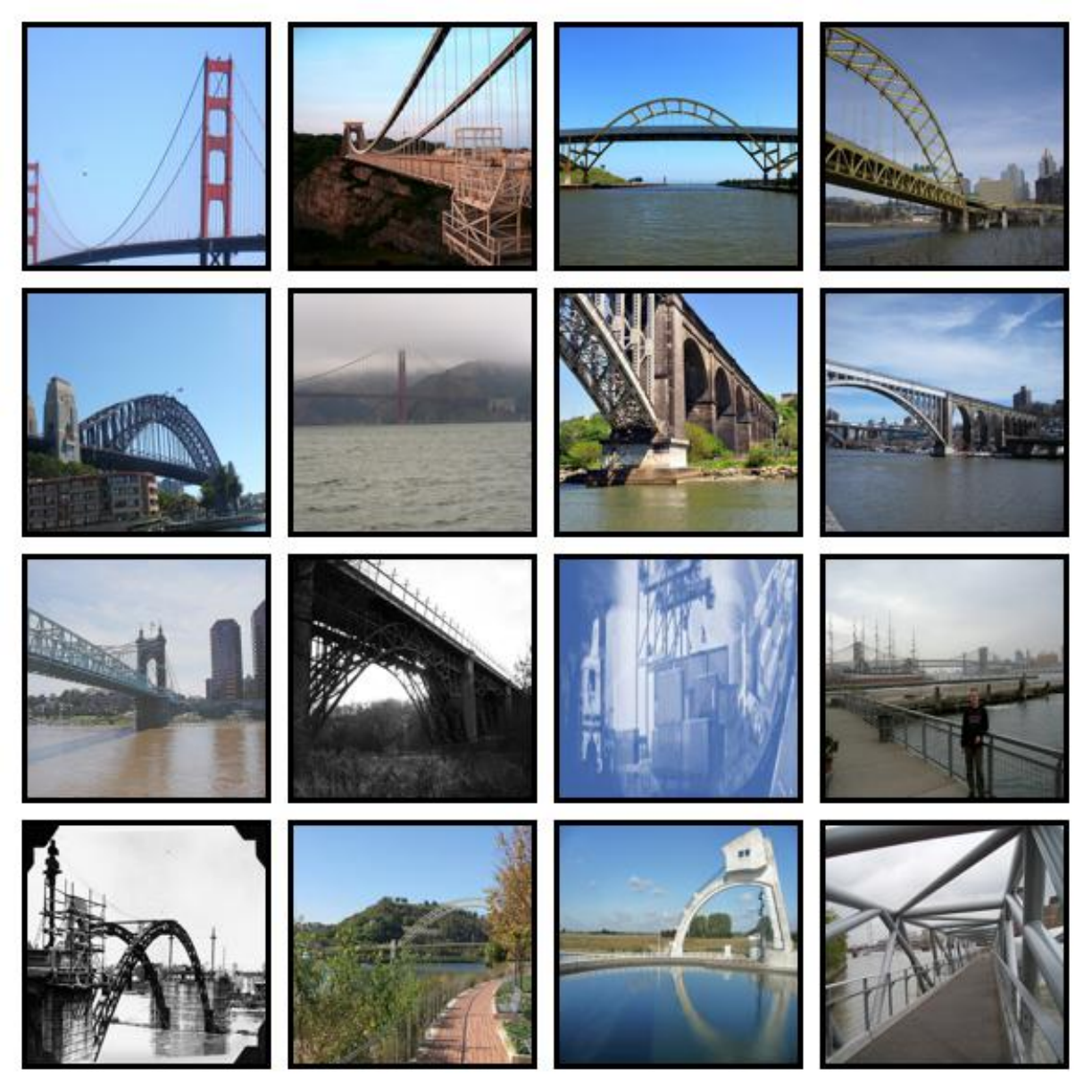} & \includegraphics[width=0.15\linewidth]{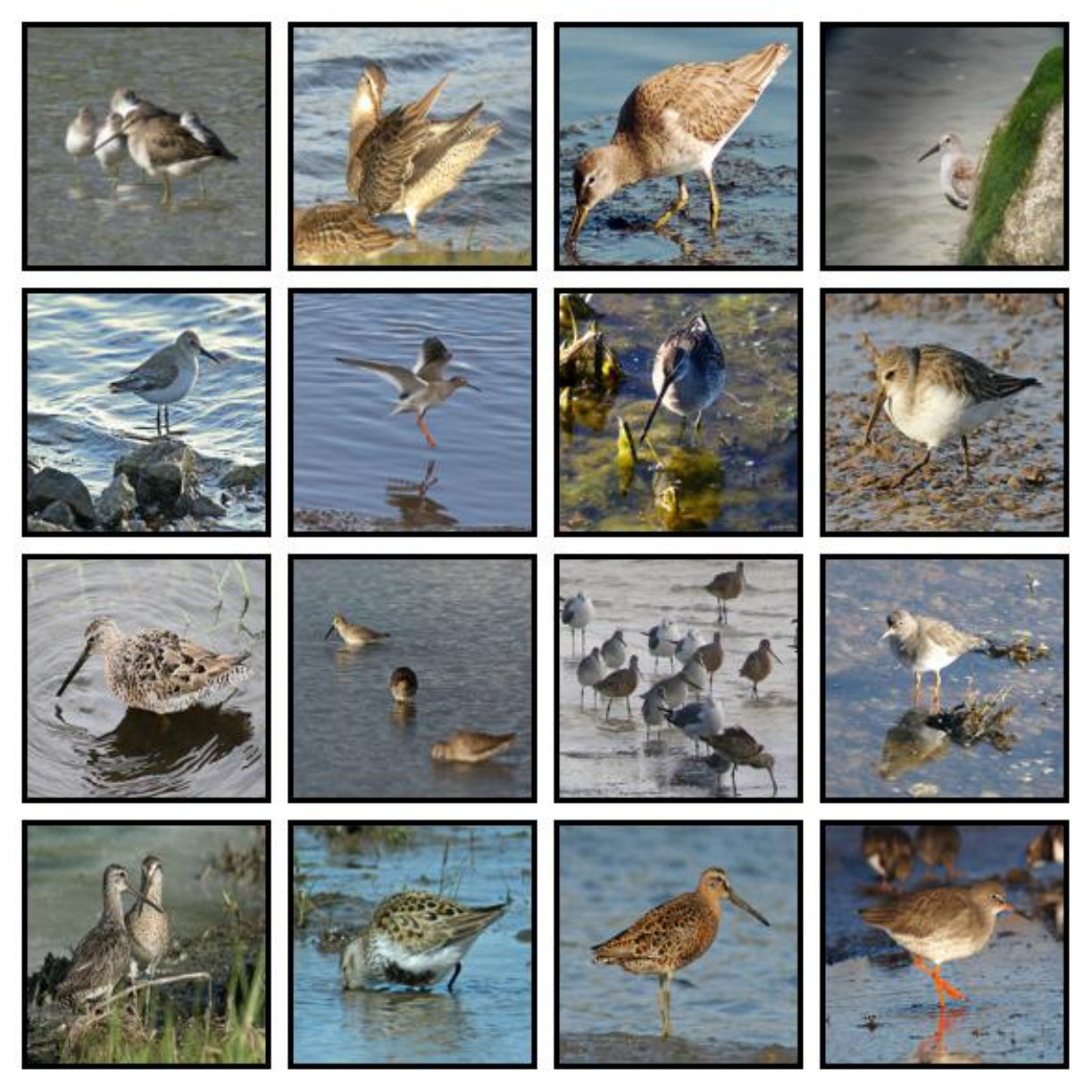} 
\end{tabular}
\caption{\textbf{\textit{Low} accuracy classes predicted by \textit{Self-Classifier} on ImageNet validation set (unseen during training)}. Images are sampled \textbf{randomly} from each predicted class. We see that even for low accuracy classes, unsupervised semantic grouping is reasonably good. A lot of the misclassified samples (with regards to ground truth) is due to subclasses from the same superclass that are clustered together (please see \cref{table:unsupervised_image_classification_superclasses} for empirical results). For example, different types of birds, monkeys, fruits, spyders, sea animals, helmets, screens (TV vs computer screen), etc. In practice, such clustering is good, yet is not aligned with ImageNet ground truth and thus the accuracy of such clusters is low.}
\label{fig:low_acc_classes}
\end{figure}

\newpage

\section{Theoretical Analysis}
\label{section:appendix_theoretical}
In this section, we show mathematically how \textit{Self-Classifier} avoids trivial solutions by design, i.e., a collapsing solution is simply not in the set of optimal solutions of our proposed loss function \cref{eq:ce_ours}. Furthermore, and more importantly, we show that by setting $p(y)$ manually, the required optimal solution can be defined by the user.

\begin{theorem}[Non-Zero Posterior Probability]
\label{theorem:appendix_non_zero_prob}
Let $B$ be a batch of $N$ samples with two views per sample, $(x_1, x_2) \in B$. Let $p(y)$ and $p(x)$ be the class and sample distributions, respectively, where $y \in [C]$. Let \cref{eq:L_sym} be the loss function. Then, each class $y \in [C]$ will have at least one sample $y \in [C]$ with non-zero posterior probability $p(x|y) > 0$ assigned into it, and each sample $x \in [N]$ will have at least one class  $y \in [C]$ with $p(x|y) > 0$.
\end{theorem}

\begin{proof}
Consider the non-negative denominators $\sum_{y \in [C]}{p(x_2|y)}$ and $\sum_{x_1 \in [B]}{p(y|x_1)}$ of the first and the second term of \cref{eq:ce_ours}, respectively. To avoid unbounded loss, both denominators must be strictly positive. Thus, for every $x_2 \in [N]$ there is $y \in [C]$ such as $p(x_2|y) > 0$, and for every $y \in [C]$ there is $x_1 \in [N]$ such that $p(y|x_1) > 0$.
Using the symmetrized objective loss function \cref{eq:L_sym} allows dropping the subscripts $_1$ and $_2$ of $x$.
\end{proof}

\begin{theorem}[Optimal Solution With Uniform Prior]
\label{theorem:appendix_custom_optimal_sol}
Let $B$ be a batch of $N$ samples with two views per sample, $(x_1, x_2) \in B$. Let $p(y)$ and $p(x)$ be the class and sample distribution, respectively, where $y \in [C]$. Then, the uniform probabilities $p(y) = \frac{1}{C}$ and $p(x) = \frac{1}{N}$ constitute a global minimizer of the loss \cref{eq:ce_ours}.
\end{theorem}

\begin{proof}
The first term $\frac{p(x_2|y)}{\sum_{\tilde{y}}{p(x_2|\tilde{y})}}$ summed over $y$ in \cref{eq:ce_ours} is normalized such that the probabilities assigned to each $x_2 \in [N]$ sum to one. Additionally, setting $p(y) = \frac{1}{C}$ and $p(x) = \frac{1}{N}$ renders the argument  $\frac{N}{C}\frac{p(y|x_1)}{\sum_{\tilde{x}_1}{p(y|\tilde{x}_1)}}$ of the logarithm in the second term in \cref{eq:ce_ours} normalized such that the probabilities assigned to each $y \in [C]$ sum to $\frac{N}{C}$. Thus, an optimal zero loss solution is one in which $\forall y \in [C]$, exactly $\frac{N}{C}$ samples are assigned a probability of one, and $\forall x \in [N]$, exactly $1$ class is assigned probability of one, which concludes the proof.
\end{proof}

\section{Unsupervised Classification Accuracy \& Training Loss Vs. Number of Epochs}
\label{section:acc_loss_vs_epochs}
In \cref{fig:clustering_acc} we plot the unsupervised classification accuracy and training loss vs. the number of training epochs of \textit{Self-Classifier}. This figure further illustrates the fact that the training objective of \textit{Self-Classifier} is aligned with semantically meaningful classification of unseen data (ImageNet validation set) despite being trained strictly with unlabeled data.

\begin{figure}[ht]
    \centering
    \includegraphics[width=0.8\linewidth]{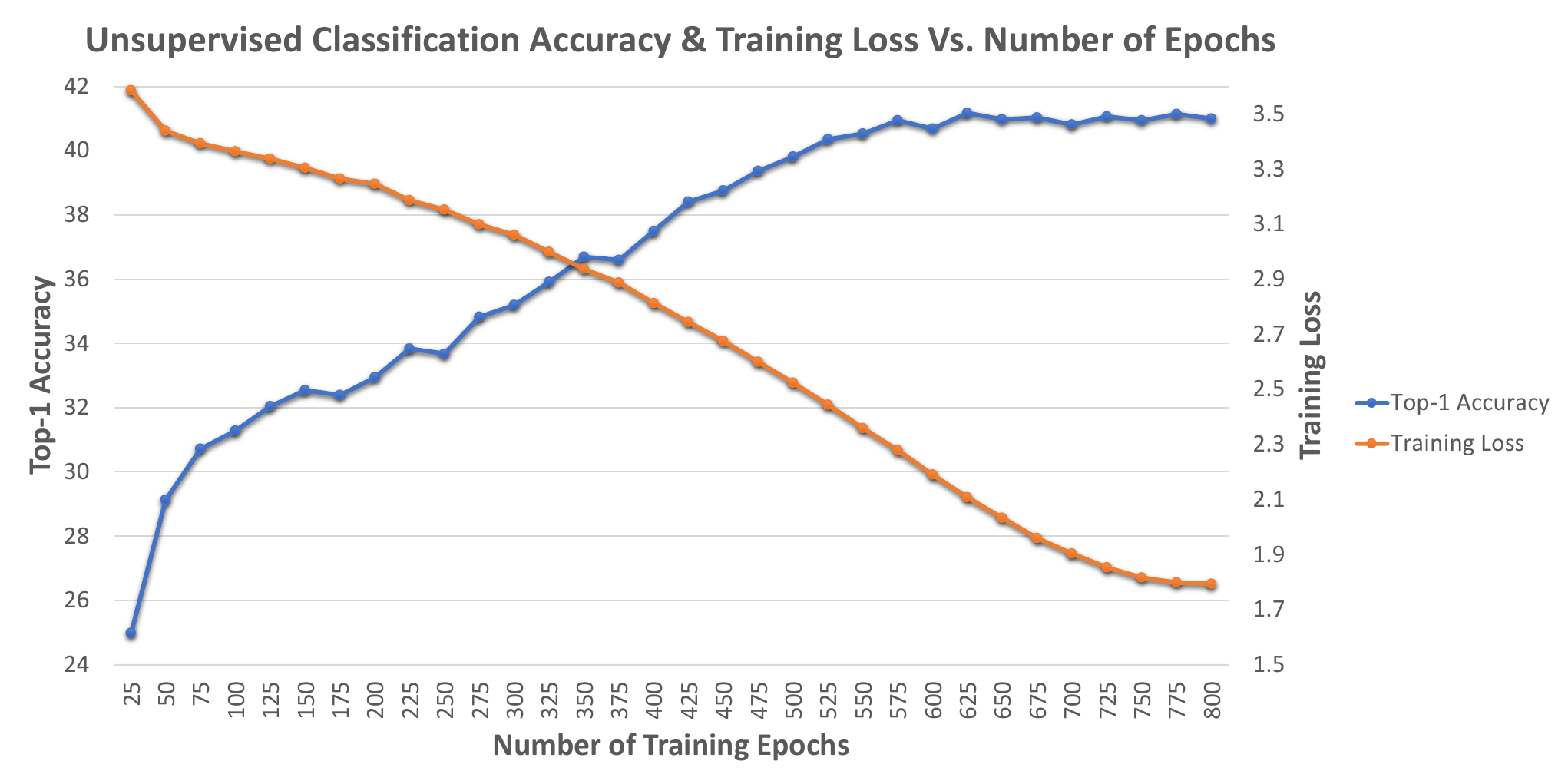}
    \caption{\textbf{Unsupervised classification accuracy \& training loss vs. number of training epochs.}}
    \label{fig:clustering_acc} 
\end{figure}

\newpage
\section{Superclasses Datasets}
\label{section:superclassess_datasets}

In \cref{table:unsupervised_image_classification_superclasses} we report results for different numbers of ImageNet superclasses (10, 29, 128, 466 and 591) resulting from cutting the default (WordNet) ImageNet hierarchy on different levels - levels 2 to 6, respectively. More specifically:
\begin{enumerate}
    \item Level 2 contains 10 superclasses -  \textit{furnishing; structure, place; fungus; conveyance, transport; plant, flora, plant life; apparel, toiletries; animal, animate being, beast, brute, creature, fauna; food, nutrient; paraphernalia; person.}
    \item Level 3 contains 29 superclasses - \textit{echinoderm; bird; wheeled vehicle; furniture, piece of furniture, article of furniture; soft furnishings, accessories; produce, green goods, green groceries, garden truck; mollusk, mollusc, shellfish; vascular plant, tracheophyte; man-made structure, construction; instrument; accessory, accoutrement, accouterment; geological formation, formation; appliance; mammal, mammalian; amphibian; garment; fungus; craft; reptile, reptilian; arthropod; cooked food, prepared food; fish; sled, sledge, sleigh; beverage, drink, drinkable, potable; person; toiletry, toilet articles; equipment; worm; coelenterate, cnidarian.}
    \item Level 4 contains 128 superclasses - \textit{dummy8; dummy6; seat; gastropod, univalve; building, edifice; spacecraft, ballistic capsule, space vehicle; medical instrument; motor vehicle, automotive vehicle; alcohol, alcoholic drink, alcoholic beverage, intoxicant, inebriant; train, railroad train; rodent, gnawer; lagomorph, gnawing mammal; bag; dummy53; dummy79; dummy0; anthozoan, actinozoan; piece of cloth, piece of material; dummy68; tool; screen; barrier; chiton, coat-of-mail shell, sea cradle, polyplacophore; baked goods; insect; neckwear; arachnid, arachnoid; person, individual, someone, somebody, mortal, soul; landing, landing place; bandage, patch; bony fish; area; gallinaceous bird, gallinacean; dummy77; dummy69; handwear, hand wear; dummy67; mountain, mount; coffee, java; bedclothes, bed clothing, bedding; dummy73; dummy14; facial accessories; dummy1; primate; sports equipment; aircraft; dryer, drier; proboscidean, proboscidian; kitchen utensil; centipede; wall unit; vegetable, veggie, veg; headdress, headgear; basidiomycete, basidiomycetous fungi; monotreme, egg-laying mammal; scientific instrument; aquatic mammal; crustacean; archosaur, archosaurian, archosaurian reptile; musical instrument, instrument; tracked vehicle; ascomycete; ungulate, hoofed mammal; chelonian, chelonian reptile; electronic equipment; dummy51; dummy71; dummy11; footwear, legwear; dummy39; dessert, sweet, afters; dummy56; cuculiform bird; fruit; door; carnivore; piciform bird; crocodilian reptile, crocodilian; dummy78; cosmetic; garment; aquatic bird; table; floor cover, floor covering; cartilaginous fish, chondrichthian; sled, sledge, sleigh; coraciiform bird; dummy10; serpentes; rig; vessel, watercraft; cycles; condiment; dummy13; flower; cephalopod, cephalopod mollusk; ratite, ratite bird, flightless bird; reef; spring, fountain, outflow, outpouring, natural spring; apodiform bird; dish; marsupial, pouched mammal; dummy15; tableware; bird of prey, raptor, raptorial bird; measuring instrument, measuring system, measuring device; lamp; passerine, passeriform bird; shore; photographic equipment; home appliance, household appliance; edentate; trilobite; dummy50; bedroom furniture; tank, storage tank; dummy12; armor; dummy7; saurian; weapon, arm, weapon system; dummy72; dummy57; cart; bar; dummy74; padding, cushioning.}
    
    where, \textit{dummy8 is European fire salamander, Salamandra salamandra; eft; common newt, Triturus vulgaris; spotted salamander, Ambystoma maculatum; axolotl, mud puppy, Ambystoma mexicanum. dummy6 is African grey, African gray, Psittacus erithacus; sulphur-crested cockatoo, Kakatoe galerita, Cacatua galerita; lorikeet; macaw. dummy52 is plastic bag; purse; backpack, back pack, knapsack, packsack, rucksack, haversack; mailbag, postbag. dummy53 is buckle. dummy79 is chest. dummy0 is cock. dummy68 is cliff, drop, drop-off. dummy77 is brass, memorial tablet, plaque; triumphal arch; megalith, megalithic structure. dummy69 is promontory, headland, head, foreland. dummy67 is fountain. dummy73 is viaduct; suspension bridge; steel arch bridge. dummy14 is flatworm, platyhelminth. dummy34 is mask; sunglasses, dark glasses, shades; gasmask, respirator, gas helmet; ski mask. dummy1 is hen. dummy19 is gyromitra. dummy51 is crutch. dummy71 is valley, vale. dummy11 is starfish, sea star. dummy39 is scabbard; holster. dummy56 is shower curtain; theater curtain, theatre curtain. dummy78 is totem pole; pedestal, plinth, footstall; obelisk. dummy10 is jellyfish. dummy9 is ringneck snake, ring-necked snake, ring snake; king snake, kingsnake; thunder snake, worm snake, Carphophis amoenus; green mamba; boa constrictor, Constrictor constrictor; vine snake; garter snake, grass snake; horned viper, cerastes, sand viper, horned asp, Cerastes cornutus; diamondback, diamondback rattlesnake, Crotalus adamanteus; hognose snake, puff adder, sand viper; sidewinder, horned rattlesnake, Crotalus cerastes; Indian cobra, Naja naja; sea snake; water snake; green snake, grass snake; night snake, Hypsiglena torquata; rock python, rock snake, Python sebae. dummy30 is mountain bike, all-terrain bike, off-roader; unicycle, monocycle; tricycle, trike, velocipede; bicycle-built-for-two, tandem bicycle, tandem. dummy13 is sea urchin. dummy15 is nematode, nematode worm, roundworm. dummy50 is handkerchief, hankie, hanky, hankey. dummy12 is sea cucumber, holothurian. dummy35 is chain mail, ring mail, mail, chain armor, chain armour, ring armor, ring armour; breastplate, aegis, egis; cuirass; bulletproof vest. dummy7 is tree frog, tree-frog; bullfrog, Rana catesbeiana; tailed frog, bell toad, ribbed toad, tailed toad, Ascaphus trui. dummy72 is traffic light, traffic signal, stoplight; scoreboard; street sign. dummy57 is umbrella. dummy74 is maze, labyrinth.}
    \item Level 5 contains 466 superclasses - \textit{jellyfish; owl, bird of Minerva, bird of night, hooter; coral fungus; screwdriver; rotisserie; sharpener; dummy45; flatworm, platyhelminth; car, railcar, railway car, railroad car; timepiece, timekeeper, horologe; bottle; chest of drawers, chest, bureau, dresser; eel; towel; fan; sandbar, sand bar; tripod; orchid, orchidaceous plant; squirrel; sturgeon; plate rack; hip, rose hip, rosehip; cicada, cicala; balloon; otter; plate; seashore, coast, seacoast, sea-coast; phasianid; rabbit, coney, cony; warthog; opener; plow, plough; bun, roll; dip; compass; lemur; rug, carpet, carpeting; waffle iron; alp; gymnastic apparatus, exerciser; wild dog; pudding, pud; junco, snowbird; harvestman, daddy longlegs, Phalangium opilio; column, pillar; horse cart, horse-cart; earthstar; mug; bookcase; leafhopper; lesser panda, red panda, panda, bear cat, cat bear, Ailurus fulgens; window shade; punch; mercantile establishment, retail store, sales outlet, outlet; bannister, banister, balustrade, balusters, handrail; espresso; cheetah, chetah, Acinonyx jubatus; clog, geta, patten, sabot; acorn; platypus, duckbill, duckbilled platypus, duck-billed platypus, Ornithorhynchus anatinus; mink; remote control, remote; percussion instrument, percussive instrument; hyena, hyaena; duck; dam, dike, dyke; parrot; wine, vino; alligator, gator; abacus; spatula; spectacles, specs, eyeglasses, glasses; curtain, drape, drapery, mantle, pall; memorial, monument; snipe; sock; buckeye, horse chestnut, conker; monitor; glove; sheath; mosquito net; brambling, Fringilla montifringilla; beaver; snorkel; armadillo; crayfish, crawfish, crawdad, crawdaddy; strawberry; space heater; cricket; custard apple; computer keyboard, keypad; telephone, phone, telephone set; black-footed ferret, ferret, Mustela nigripes; robin, American robin, Turdus migratorius; titmouse, tit; amphibian, amphibious vehicle; tape player; bobsled, bobsleigh, bob; toaster; penguin; cuckoo; magpie; bed; hamster; coral; cannon; mat; snail; home theater, home theatre; wind instrument, wind; goldfish, Carassius auratus; undergarment, unmentionable; patio, terrace; tank, army tank, armored combat vehicle, armoured combat vehicle; snowmobile; pin; oystercatcher, oyster catcher; mongoose; scarf; gyromitra; isopod; swan; adhesive bandage; firearm, piece, small-arm; water ouzel, dipper; bell cote, bell cot; boot; dummy46; snake, serpent, ophidian; albatross, mollymawk; damselfly; marmot; weasel; coffee maker; submarine, pigboat, sub, U-boat; saltshaker, salt shaker; cap; beetle; slug; lock; whale; wild boar, boar, Sus scrofa; ray; makeup, make-up, war paint; limpkin, Aramus pictus; geyser; measuring cup; lizard; jug; photocopier; radio, wireless; buckle; plunger, plumber's helper; hairpiece, false hair, postiche; potato, white potato, Irish potato, murphy, spud, tater; lion, king of beasts, Panthera leo; sloth, tree sloth; passenger train; shoe; dummy49; shovel; salmon; stork; helmet; bustard; daisy; wardrobe, closet, press; desk; puffer, pufferfish, blowfish, globefish; cooker; tights, leotards; hen; mouse, computer mouse; dwelling, home, domicile, abode, habitation, dwelling house; reservoir; binoculars, field glasses, opera glasses; echidna, spiny anteater, anteater; jacamar; bag; entertainment center; fly; bow; skunk, polecat, wood pussy; domestic cat, house cat, Felis domesticus, Felis catus; dogsled, dog sled, dog sleigh; stew; cattle, cows, kine, oxen, Bos taurus; lighter, light, igniter, ignitor; dock, dockage, docking facility; tiger, Panthera tigris; porcupine, hedgehog; cream, ointment, emollient; chain saw, chainsaw; crab; lynx, catamount; knife; necktie, tie; monkey; broccoli; bib; stove; missile; sewing machine; car, auto, automobile, machine, motorcar; screw; grasshopper, hopper; chair; hare; display, video display; crane; lobster; audio system, sound system; bus, autobus, coach, charabanc, double-decker, jitney, motorbus, motorcoach, omnibus, passenger vehicle; outbuilding; jackfruit, jak, jack; lens cap, lens cover; apple; projectile, missile; coral reef; chest; meerkat, mierkat; spoon; crutch; mushroom; squash; coat; sheep; cliff, drop, drop-off; pot; television, television system; jaguar, panther, Panthera onca, Felis onca; scale, weighing machine; llama; acarine; sea slug, nudibranch; valley, vale; oxcart; bee eater; body armor, body armour, suit of armor, suit of armour, coat of mail, cataphract; rapeseed; window screen; lawn mower, mower; cucumber, cuke; fox; butterfly; racket, racquet; cassette player; grouse; hand blower, blow dryer, blow drier, hair dryer, hair drier; giant panda, panda, panda bear, coon bear, Ailuropoda melanoleuca; power drill; volcano; seal; pomegranate; refrigerator, icebox; toucan; horse, Equus caballus; salamander; damselfish, demoiselle; sea urchin; stinkhorn, carrion fungus; sweater, jumper; sandpiper; stocking; dummy47; shoji; spoonbill; hummingbird; frozen dessert; stage; pizza, pizza pie; sea cow, sirenian mammal, sirenian; microscope; kangaroo; goldfinch, Carduelis carduelis; cougar, puma, catamount, mountain lion, painter, panther, Felis concolor; mantis, mantid; reel; fire screen, fireguard; iron, smoothing iron; jewelry, jewellery; bee; nail; heron; gar, garfish, garpike, billfish, Lepisosteus osseus; pelican; maze, labyrinth; projector; cracker; torch; phalanger, opossum, possum; sandwich; boat; fig; cauliflower; vulture; quilt, comforter, comfort, puff; cushion; factory, mill, manufacturing plant, manufactory; ear, spike, capitulum; CD player; pool table, billiard table, snooker table; barometer; ape; butterfly fish; modem; skirt; fountain; soup; sliding door; hippopotamus, hippo, river horse, Hippopotamus amphibius; indigo bunting, indigo finch, indigo bird, Passerina cyanea; groom, bridegroom; antelope; pepper; microwave, microwave oven; bowl; cabbage, chou; shirt; eagle, bird of Jove; wombat; paintbrush; Old World buffalo, buffalo; goose; stethoscope; microphone, mike; lacewing, lacewing fly; burrito; guinea pig, Cavia cobaya; elephant; washer, automatic washer, washing machine; scorpion; airship, dirigible; hair spray; handkerchief, hankie, hanky, hankey; spider; scuba diver; sofa, couch, lounge; dragonfly, darning needle, devil's darning needle, sewing needle, snake feeder, snake doctor, mosquito hawk, skeeter hawk; breakwater, groin, groyne, mole, bulwark, seawall, jetty; jinrikisha, ricksha, rickshaw; locomotive, engine, locomotive engine, railway locomotive; toilet seat; roof; knee pad; dishrag, dishcloth; keyboard instrument; hat, chapeau, lid; apron; streetcar, tram, tramcar, trolley, trolley car; tricycle, trike, velocipede; shark; table lamp; banana; jay; clip; ski; baby bed, baby's bed; ball; barracouta, snoek; weight, free weight, exercising weight; fence, fencing; nematode, nematode worm, roundworm; lampshade, lamp shade; unicycle, monocycle; hermit crab; camel; frog, toad, toad frog, anuran, batrachian, salientian; sauce; promontory, headland, head, foreland; bridge, span; motorcycle, bike; ballplayer, baseball player; artichoke, globe artichoke; nightwear, sleepwear, nightclothes; orange; airplane, aeroplane, plane; place of worship, house of prayer, house of God, house of worship; polecat, fitch, foulmart, foumart, Mustela putorius; potpie; sea cucumber, holothurian; oscilloscope, scope, cathode-ray oscilloscope, CRO; camera, photographic camera; bear; house finch, linnet, Carpodacus mexicanus; starfish, sea star; dining table, board; wild sheep; badger; mushroom; pen; carpenter's kit, tool kit; scorpaenid, scorpaenid fish; digital computer; zebra; pineapple, ananas; drilling platform, offshore rig; truck, motortruck; signboard, sign; cock; bison; vacuum, vacuum cleaner; nightingale, Luscinia megarhynchos; hornbill; military uniform; cabinet; perfume, essence; chambered nautilus, pearly nautilus, nautilus; lakeside, lakeshore; cardoon; snow leopard, ounce, Panthera uncia; meat loaf, meatloaf; dress, frock; wolf; dishwasher, dish washer, dishwashing machine; brace; tray; tower; eraser; ostrich, Struthio camelus; theater, theatre, house; glass, drinking glass; warplane, military plane; sea anemone, anemone; hog, pig, grunter, squealer, Sus scrofa; ant, emmet, pismire; bicycle, bike, wheel, cycle; crocodile; lemon; leopard, Panthera pardus; meter; dinosaur; suit, suit of clothes; hammer; goat, caprine animal; trouser, pant; cockroach, roach; half track; loaf of bread, loaf; conch; gallinule, marsh hen, water hen, swamphen; dog, domestic dog, Canis familiaris; turtle; gate; corn; flamingo; face mask; hawk; ladle; pan, cooking pan; cocktail shaker; umbrella; rail; space shuttle; printer; stringed instrument; syringe; swimsuit, swimwear, bathing suit, swimming costume, bathing costume; measuring stick, measure, measuring rod; handcart, pushcart, cart, go-cart; ship; plover; joystick; walking stick, walkingstick, stick insect; tench, Tinca tinca; telescope, scope.}
    
    where, dummy45 is library. dummy46 is planetarium. dummy49 is restaurant, eating house, eating place, eatery. dummy47 is prison, prison house. dummy20 is hen-of-the-woods, hen of the woods, Polyporus frondosus, Grifola frondosa; agaric; bolete.
    \item Level 6 contains 591 superclasses - \textit{German shepherd, German shepherd dog, German police dog, alsatian; hen-of-the-woods, hen of the woods, Polyporus frondosus, Grifola frondosa; Maltese dog, Maltese terrier, Maltese; pirate, pirate ship; cellular telephone, cellular phone, cellphone, cell, mobile phone; red fox, Vulpes vulpes; Dandie Dinmont, Dandie Dinmont terrier; sweet pepper; suspension bridge; spotted salamander, Ambystoma maculatum; kuvasz; coho, cohoe, coho salmon, blue jack, silver salmon, Oncorhynchus kisutch; violin, fiddle; gown; lory; schooner; organ, pipe organ; Saluki, gazelle hound; hot pot, hotpot; Border collie; briard; Windsor tie; kite; Norwich terrier; mask; china cabinet, china closet; water jug; tractor; bonnet, poke bonnet; vault; boxer; Bernese mountain dog; tarantula; Tibetan terrier, chrysanthemum dog; marmoset; ice cream, icecream; eggnog; cloak; king crab, Alaska crab, Alaskan king crab, Alaska king crab, Paralithodes camtschatica; python; rail fence; accordion, piano accordion, squeeze box; birdhouse; Samoyed, Samoyede; caldron, cauldron; winter squash; night snake, Hypsiglena torquata; fox squirrel, eastern fox squirrel, Sciurus niger; agaric; head cabbage; thunder snake, worm snake, Carphophis amoenus; Arabian camel, dromedary, Camelus dromedarius; padlock; stingray; Australian terrier; dalmatian, coach dog, carriage dog; shower curtain; spiny lobster, langouste, rock lobster, crawfish, crayfish, sea crawfish; minibike, motorbike; lipstick, lip rouge; fox terrier; Egyptian cat; egret; cornet, horn, trumpet, trump; tow truck, tow car, wrecker; ski mask; espresso maker; Rottweiler; pickelhaube; vizsla, Hungarian pointer; basenji; kit fox, Vulpes macrotis; frying pan, frypan, skillet; whiskey jug; automatic firearm, automatic gun, automatic weapon; weevil; safety pin; cardigan; coyote, prairie wolf, brush wolf, Canis latrans; Shih-Tzu; siamang, Hylobates syndactylus, Symphalangus syndactylus; redshank, Tringa totanus; Eskimo dog, husky; West Highland white terrier; dugong, Dugong dugon; timer; harmonica, mouth organ, harp, mouth harp; malamute, malemute, Alaskan malamute; laptop, laptop computer; red wolf, maned wolf, Canis rufus, Canis niger; bulbul; mamba; chiffonier, commode; shoe shop, shoe-shop, shoe store; axolotl, mud puppy, Ambystoma mexicanum; mosque; black and gold garden spider, Argiope aurantia; snowplow, snowplough; red wine; bulletproof vest; hognose snake, puff adder, sand viper; recreational vehicle, RV, R.V.; bagel, beigel; Ibizan hound, Ibizan Podenco; street sign; ballpoint, ballpoint pen, ballpen, Biro; lumbermill, sawmill; king snake, kingsnake; kelpie; rifle; cup; picket fence, paling; golf ball; bakery, bakeshop, bakehouse; trolleybus, trolley coach, trackless trolley; Gordon setter; indri, indris, Indri indri, Indri brevicaudatus; tailed frog, bell toad, ribbed toad, tailed toad, Ascaphus trui; lotion; fountain pen; black-and-tan coonhound; cobra; cowboy hat, ten-gallon hat; ringneck snake, ring-necked snake, ring snake; wreck; loudspeaker, speaker, speaker unit, loudspeaker system, speaker system; partridge; dial telephone, dial phone; electric fan, blower; bloodhound, sleuthhound; wallaby, brush kangaroo; iguanid, iguanid lizard; sailboat, sailing boat; dome; Greater Swiss Mountain dog; airliner; cab, hack, taxi, taxicab; macaque; sandal; sulphur butterfly, sulfur butterfly; spaghetti sauce, pasta sauce; bathing cap, swimming cap; gibbon, Hylobates lar; African elephant, Loxodonta africana; soup bowl; miniature pinscher; bittern; pullover, slipover; mobile home, manufactured home; purple gallinule; Band Aid; Italian greyhound; cockatoo; poncho; Doberman, Doberman pinscher; prairie chicken, prairie grouse, prairie fowl; hammerhead, hammerhead shark; capuchin, ringtail, Cebus capucinus; purse; lifeboat; poodle, poodle dog; French horn, horn; shower cap; grey whale, gray whale, devilfish, Eschrichtius gibbosus, Eschrichtius robustus; sloth bear, Melursus ursinus, Ursus ursinus; tabby, tabby cat; African hunting dog, hyena dog, Cape hunting dog, Lycaon pictus; drum, membranophone, tympan; church, church building; pencil sharpener; long-horned beetle, longicorn, longicorn beetle; canoe; Persian cat; folding chair; green snake, grass snake; prison, prison house; rock crab, Cancer irroratus; ocarina, sweet potato; gecko; box turtle, box tortoise; timber wolf, grey wolf, gray wolf, Canis lupus; overskirt; screen, CRT screen; wirehair, wirehaired terrier, wire-haired terrier; wig; English foxhound; tree frog, tree-frog; thatch, thatched roof; ibex, Capra ibex; pajama, pyjama, pj's, jammies; harp; pistol, handgun, side arm, shooting iron; Lhasa, Lhasa apso; brass, memorial tablet, plaque; Weimaraner; wooden spoon; tiger cat; terrapin; fire engine, fire truck; library; fireboat; bolo tie, bolo, bola tie, bola; lady's slipper, lady-slipper, ladies' slipper, slipper orchid; electric ray, crampfish, numbfish, torpedo; Norfolk terrier; vestment; pipe; coffee mug; racer, race car, racing car; scarabaeid beetle, scarabaeid, scarabaean; redbone; totem pole; miniskirt, mini; Boston bull, Boston terrier; swimming trunks, bathing trunks; barber chair; file, file cabinet, filing cabinet; grocery store, grocery, food market, market; limousine, limo; punching bag, punch bag, punching ball, punchball; face powder; balance beam, beam; warship, war vessel, combat ship; scoreboard; guacamole; hand-held computer, hand-held microcomputer; four-poster; neck brace; dowitcher; sports car, sport car; komondor; pay-phone, pay-station; wok; barrow, garden cart, lawn cart, wheelbarrow; fur coat; anguid lizard; Pekinese, Pekingese, Peke; grey fox, gray fox, Urocyon cinereoargenteus; little blue heron, Egretta caerulea; rugby ball; anemone fish; eating apple, dessert apple; maillot; Labrador retriever; raincoat, waterproof; Brittany spaniel; pickup, pickup truck; tick; vine snake; bikini, two-piece; langur; odometer, hodometer, mileometer, milometer; croquet ball; running shoe; convertible, sofa bed; theater curtain, theatre curtain; pug, pug-dog; pop bottle, soda bottle; breastplate, aegis, egis; kimono; white stork, Ciconia ciconia; ladybug, ladybeetle, lady beetle, ladybird, ladybird beetle; palace; koala, koala bear, kangaroo bear, native bear, Phascolarctos cinereus; true frog, ranid; Irish setter, red setter; schipperke; chameleon, chamaeleon; triumphal arch; tiger beetle; garden spider, Aranea diademata; Rhodesian ridgeback; great grey owl, great gray owl, Strix nebulosa; Crock Pot; golfcart, golf cart; maraca; soccer ball; colobus, colobus monkey; bath towel; bow tie, bow-tie, bowtie; steam locomotive; summer squash; water bottle; bicycle-built-for-two, tandem bicycle, tandem; Afghan hound, Afghan; ptarmigan; toy terrier; yurt; European fire salamander, Salamandra salamandra; water tower; red-backed sandpiper, dunlin, Erolia alpina; restaurant, eating house, eating place, eatery; stole; ping-pong ball; minivan; academic gown, academic robe, judge's robe; steel drum; gorilla, Gorilla gorilla; cello, violoncello; wood rabbit, cottontail, cottontail rabbit; rock beauty, Holocanthus tricolor; black grouse; toyshop; garter snake, grass snake; trailer truck, tractor trailer, trucking rig, rig, articulated lorry, semi; rubber eraser, rubber, pencil eraser; hamburger, beefburger, burger; black widow, Latrodectus mactans; goblet; quill, quill pen; school bus; sax, saxophone; hair slide; ram, tup; steel arch bridge; magnetic compass; Walker hound, Walker foxhound; lycaenid, lycaenid butterfly; Irish wolfhound; agamid, agamid lizard; cliff dwelling; impala, Aepyceros melampus; EntleBucher; sarong; jersey, T-shirt, tee shirt; fiddler crab; Scotch terrier, Scottish terrier, Scottie; curly-coated retriever; ice bear, polar bear, Ursus Maritimus, Thalarctos maritimus; sea lion; rule, ruler; crash helmet; medicine chest, medicine cabinet; shopping cart; chocolate sauce, chocolate syrup; gondola; barn; wolf spider, hunting spider; Yorkshire terrier; ruffed grouse, partridge, Bonasa umbellus; rocking chair, rocker; freight car; go-kart; Mexican hairless; bolete; obelisk; speedboat; horned viper, cerastes, sand viper, horned asp, Cerastes cornutus; barbell; Great Pyrenees; lab coat, laboratory coat; Model T; mashed potato; flat-coated retriever; bullet train, bullet; black stork, Ciconia nigra; mitten; greenhouse, nursery, glasshouse; Loafer; sunscreen, sunblock, sun blocker; passenger car, coach, carriage; coot; Sussex spaniel; Dungeness crab, Cancer magister; Leonberg; Madagascar cat, ring-tailed lemur, Lemur catta; tobacco shop, tobacconist shop, tobacconist; jean, blue jean, denim; parallel bars, bars; scabbard; Angora, Angora rabbit; combination lock; dumbbell; collie; watch, ticker; sundial; hartebeest; prayer rug, prayer mat; American black bear, black bear, Ursus americanus, Euarctos americanus; football helmet; pierid, pierid butterfly; Norwegian elkhound, elkhound; iPod; king penguin, Aptenodytes patagonica; chair of state; guenon, guenon monkey; water spaniel; shrine; traffic light, traffic signal, stoplight; Arctic fox, white fox, Alopex lagopus; Siamese cat, Siamese; notebook, notebook computer; toy spaniel; patas, hussar monkey, Erythrocebus patas; spider monkey, Ateles geoffroyi; cocker spaniel, English cocker spaniel, cocker; barbershop; chickadee; water buffalo, water ox, Asiatic buffalo, Bubalus bubalis; cinema, movie theater, movie theatre, movie house, picture palace; moving van; soft-coated wheaten terrier; forklift; Chihuahua; Pomeranian; corgi, Welsh corgi; ruddy turnstone, Arenaria interpres; bassinet; peacock; jeep, landrover; sea snake; lacertid lizard, lacertid; mountain bike, all-terrain bike, off-roader; cairn, cairn terrier; necklace; barn spider, Araneus cavaticus; Kerry blue terrier; Polaroid camera, Polaroid Land camera; macaw; bullterrier, bull terrier; Saint Bernard, St Bernard; desktop computer; sandglass; can opener, tin opener; mountain sheep; flute, transverse flute; chain mail, ring mail, mail, chain armor, chain armour, ring armor, ring armour; chow, chow chow; tile roof; bull mastiff; Indian elephant, Elephas maximus; American alligator, Alligator mississipiensis; paper towel; sorrel; howler monkey, howler; Japanese spaniel; cleaver, meat cleaver, chopper; minibus; Great Dane; pedestal, plinth, footstall; silky terrier, Sydney silky; proboscis monkey, Nasalis larvatus; Siberian husky; convertible; quail; police van, police wagon, paddy wagon, patrol wagon, wagon, black Maria; merganser, fish duck, sawbill, sheldrake; marimba, xylophone; Airedale, Airedale terrier; abaya; springer spaniel, springer; piano, pianoforte, forte-piano; African grey, African gray, Psittacus erithacus; beer bottle; hand glass, simple microscope, magnifying glass; titi, titi monkey; African crocodile, Nile crocodile, Crocodylus niloticus; lionfish; gong, tam-tam; mortarboard; griffon, Brussels griffon, Belgian griffon; Dutch oven; megalith, megalithic structure; newt, triton; Newfoundland, Newfoundland dog; bookshop, bookstore, bookstall; clock; beagle; plastic bag; pillow; orangutan, orang, orangutang, Pongo pygmaeus; Border terrier; coucal; sunglasses, dark glasses, shades; white wolf, Arctic wolf, Canis lupus tundrarum; danaid, danaid butterfly; passenger ship; feather boa, boa; chime, bell, gong; garbage truck, dustcart; consomme; rattlesnake, rattler; ringlet, ringlet butterfly; teapot; holster; baseball; guitar; bearskin, busby, shako; Irish terrier; reflex camera; water snake; golden retriever; ox; bald eagle, American eagle, Haliaeetus leucocephalus; trombone; doormat, welcome mat; bluetick; beacon, lighthouse, beacon light, pharos; trifle; bottle opener; gazelle; killer whale, killer, orca, grampus, sea wolf, Orcinus orca; hoopskirt, crinoline; brain coral; borzoi, Russian wolfhound; Scottish deerhound, deerhound; crib, cot; motor scooter, scooter; volleyball; affenpinscher, monkey pinscher, monkey dog; three-toed sloth, ai, Bradypus tridactylus; clumber, clumber spaniel; castle; planetarium; cowboy boot; mixing bowl; chimpanzee, chimp, Pan troglodytes; whippet; shed; gasmask, respirator, gas helmet; banjo; cradle; confectionery, confectionary, candy store; Old English sheepdog, bobtail; viaduct; yawl; English setter; German short-haired pointer; pretzel; Christmas stocking; radio telescope, radio reflector; sombrero; drake; cuirass; tiger shark, Galeocerdo cuvieri; leaf beetle, chrysomelid; ground beetle, carabid beetle; Belgian sheepdog, Belgian shepherd; dhole, Cuon alpinus; baboon; beach wagon, station wagon, wagon, estate car, beach waggon, station waggon, waggon; Bouvier des Flandres, Bouviers des Flandres; chainlink fence; basset, basset hound; Appenzeller; butcher shop, meat market; electric locomotive; maillot, tank suit; otterhound, otter hound; mailbag, postbag; ambulance; brown bear, bruin, Ursus arctos; Chesapeake Bay retriever; keeshond; true lobster; mastiff; basketball; bulldog, English bulldog; monitor, monitor lizard, varan; coffeepot; oboe, hautboy, hautbois; horizontal bar, high bar; black swan, Cygnus atratus; brassiere, bra, bandeau; schnauzer; great white shark, white shark, man-eater, man-eating shark, Carcharodon carcharias; wine bottle; sea turtle, marine turtle; nymphalid, nymphalid butterfly, brush-footed butterfly, four-footed butterfly; squirrel monkey, Saimiri sciureus; diaper, nappy, napkin; tennis ball; boa constrictor, Constrictor constrictor; ornithischian, ornithischian dinosaur; ice lolly, lolly, lollipop, popsicle; French loaf; monastery; Bedlington terrier; dingo, warrigal, warragal, Canis dingo; cargo ship, cargo vessel; backpack, back pack, knapsack, packsack, rucksack, haversack; venomous lizard; bassoon; teiid lizard, teiid; stone wall; bench; mud turtle; hotdog, hot dog, red hot; beer glass; turnstile; Shetland sheepdog, Shetland sheep dog, Shetland.}
\end{enumerate}

\newpage
\section{BREEDS dataset}
\label{section:breeds_datasets}

In \cref{table:unsupervised_image_classification_breeds} we report the results on four ImageNet subpopulation datasets of BREEDS \cite{santurkar2020breeds} (Entity13, Entity30, Living17, Nonliving26). These datasets are accompanied by class hierarchies re-calibrated by \cite{santurkar2020breeds} such that classes on same hierarchy level are of the same visual granularity. More specifically, 
\begin{enumerate}
    \item Entity13 contains 13 superclasses - \textit{garment; bird; reptile, reptilian; arthropod; mammal, mammalian; accessory, accoutrement, accouterment; craft; equipment; furniture, piece of furniture, article of furniture; instrument; man-made structure, construction; wheeled vehicle; produce, green goods, green groceries, garden truck.}
    \item Entity30 contains 30 superclasses - \textit{serpentes; passerine, passeriform bird; saurian; arachnid, arachnoid; aquatic bird; crustacean; carnivore; insect; ungulate, hoofed mammal; primate; bony fish; barrier; building, edifice; electronic equipment; footwear, legwear; garment; headdress, headgear; home appliance, household appliance; kitchen utensil; measuring instrument, measuring system, measuring device; motor vehicle, automotive vehicle; musical instrument, instrument; neckwear; sports equipment; tableware; tool; vessel, watercraft; dish; vegetable, veggie, veg; fruit.}
    \item Living17 contains 17 superclasses - \textit{salamander; turtle; lizard; snake, serpent, ophidian; spider; grouse; parrot; crab; dog, domestic dog, Canis familiaris; wolf; fox; domestic cat, house cat, Felis domesticus, Felis catus; bear; beetle; butterfly; ape; monkey.}
    \item Nonliving26 contains 26 superclasses - \textit{bag; ball; boat; body armor, body armour, suit of armor, suit of armour, coat of mail, cataphract; bottle; bus, autobus, coach, charabanc, double-decker, jitney, motorbus, motorcoach, omnibus, passenger vehicle; car, auto, automobile, machine, motorcar; chair; coat; digital computer; dwelling, home, domicile, abode, habitation, dwelling house; fence, fencing; hat, chapeau, lid; keyboard instrument; mercantile establishment, retail store, sales outlet, outlet; outbuilding; percussion instrument, percussive instrument; pot; roof; ship; skirt; stringed instrument; timepiece, timekeeper, horologe; truck, motortruck; wind instrument, wind; squash.}
\end{enumerate}

\end{document}